\documentclass[11pt]{article}
\usepackage[top=1in, bottom=1in, left=1in, right=1in]{geometry}

\usepackage{amsmath,amssymb,amsbsy,amsfonts,amscd,bm,url,color,latexsym}
\usepackage[hidelinks]{hyperref}
\usepackage{paralist}
\usepackage[dvipsnames]{xcolor}
\usepackage{color}
\usepackage{graphicx}
\graphicspath{{./figs/}}
\usepackage{algorithm}
\usepackage{algorithmic}
\usepackage{comment}
\usepackage{multirow}
\usepackage{enumitem}
\usepackage{fancyhdr}
\usepackage{cite}
\usepackage{cleveref}
\usepackage{subcaption}
\usepackage{authblk}
\newcommand{\R}{\mathbb{R}}

\newcommand{\e}{\begin{equation}}
\newcommand{\ee}{\end{equation}}
\newcommand{\en}{\begin{equation*}}
\newcommand{\een}{\end{equation*}}
\newcommand{\eqn}{\begin{eqnarray}}
\newcommand{\eeqn}{\end{eqnarray}}
\newcommand{\bmat}{\begin{bmatrix}}
\newcommand{\emat}{\end{bmatrix}}
\DeclareMathAlphabet\mathbfcal{OMS}{cmsy}{b}{n}
\newcommand{\mb}{\boldsymbol}
\newcommand{\mc}{\mathcal}

\newcommand{\bb}{\mathbb}

\newcommand{\vct}[1]{\boldsymbol{#1}}
\newcommand{\mtx}[1]{\boldsymbol{#1}}
\newcommand{\trace}{\operatorname{trace}}

\newcommand{\rank}{\operatorname{rank}}

\newcommand{\wh}{\widehat}

\newcommand{\wt}{\widetilde}
\newcommand{\ol}{\overline}
\newcommand{\NC}{$\mc {NC}$}

\newcommand{\norm}[2]{\left\| #1 \right\|_{#2}}

\newcommand{\abs}[1]{\left| #1 \right|}

\newcommand{\parans}[1]{\left(#1\right)}
\newcommand{\sets}[1]{\left\{#1\right\}}
\newcommand{\innerprod}[2]{\left\langle #1,  #2 \right\rangle}

\newcommand*\samethanks[1][\value{footnote}]{\footnotemark[#1]}

\newcommand{\calP}{\mathcal{P}}

\newcommand{\va}{\vct{a}}
\newcommand{\vb}{\vct{b}}

\newcommand{\vh}{\vct{h}}

\newcommand{\vu}{\vct{u}}
\newcommand{\vv}{\vct{v}}
\newcommand{\vw}{\vct{w}}
\newcommand{\vx}{\vct{x}}
\newcommand{\vy}{\vct{y}}

\newcommand{\valpha}{\vct{\alpha}}

\newcommand{\vtheta}{\vct{\theta}}

\newcommand{\vzero}{\vct{0}}
\newcommand{\vone}{\vct{1}}

\newcommand{\mA}{\mtx{A}}
\newcommand{\mB}{\mtx{B}}

\newcommand{\mG}{\mtx{G}}
\newcommand{\mH}{\mtx{H}}

\newcommand{\mP}{\mtx{P}}

\newcommand{\mR}{\mtx{R}}

\newcommand{\mU}{\mtx{U}}
\newcommand{\mV}{\mtx{V}}
\newcommand{\mW}{\mtx{W}}

\newcommand{\mY}{\mtx{Y}}
\newcommand{\mZ}{\mtx{Z}}

\newcommand{\mDelta}{\mtx{\Delta}}

\newcommand{\mSigma}{\mtx{\Sigma}}

\newcommand{\mId}{\mtx{I}}

\graphicspath{{./figs/}}

\newlength{\imgwidth}
\newcommand{\revise}[1]{\textcolor{black}{#1}}

\newboolean{twoColVersion}
\setboolean{twoColVersion}{false}
\newcommand{\twoCol}[2]{\ifthenelse{\boolean{twoColVersion}} {#1} {#2} }

\usepackage{tikz}
\usepackage{tikz-3dplot}
\usepackage{pgfplots}
\usepgfplotslibrary{patchplots}
\usetikzlibrary{patterns, positioning, arrows}
\pgfplotsset{compat=1.15}

\usepackage{hyperref}
\usepackage{wrapfig}

\hypersetup{
     colorlinks=true,%
     citecolor=blue,%
     filecolor=blue,%
     linkcolor=blue,%
     urlcolor=blue
}
 \usepackage[toc, page]{appendix}
 \newtheorem{theorem}{Theorem}[section]
\newtheorem{lemma}[theorem]{Lemma}
 
\newtheorem{proposition}[theorem]{Proposition}
\newtheorem{definition}[theorem]{Definition}

\newcommand{ \Brac }[1]{\left\lbrace #1 \right\rbrace}
\newcommand{ \brac }[1]{\left[ #1 \right]}
\newcommand{ \paren }[1]{ \left( #1 \right) }

\newcommand{\lambdaW}{\lambda_{\mW}}
\newcommand{\lambdaH}{\lambda_{\mH}}
\newcommand{\lambdab}{\lambda_{\vb}}

 \def \endprf{\hfill {\vrule height6pt width6pt depth0pt}\medskip}
\newenvironment{proof}{\noindent {\bf Proof} }{\endprf\par}

\title{On the Optimization Landscape of Neural Collapse under MSE Loss: Global Optimality with Unconstrained Features}

\begin{document}

\author[$\sharp$]{Jinxin Zhou\thanks{The first two authors contributed to this work equally.}$^,$}
\author[$\Box$]{Xiao Li\samethanks$^,$}
\author[$\Diamond$]{Tianyu Ding}
\author[$\clubsuit$]{Chong You}
\author[$\Box$]{Qing Qu\thanks{The last two authors share the corresponding authorship of this work.}$^,$}
\author[$\sharp$]{Zhihui Zhu\samethanks$^,$}

\affil[$\Diamond$]{Microsoft Research, Redmond}
%\affil[$\ddagger$]{Center for Data Science, New York University}
\affil[$\clubsuit$]{Google Research, New York City}
\affil[$\sharp$]{Department of Electrical \& Computer Engineering, University of Denver}
%\affil[$\spadesuit$]{Department of Biomedical Engineering \& MINDS, Johns Hopkins University}
\affil[$\Box$]{Department of Electrical Engineering \& Computer Science, University of Michigan}
%\date{\today}
\maketitle

%\twocolumn[
%\icmltitle{On the Optimization Landscape of Neural Collapse under MSE Loss:\\ Global Optimality with Unconstrained Features}








\begin{abstract}
When training deep neural networks for classification tasks, an intriguing empirical phenomenon has been widely observed in the last-layer classifiers and features, where (i) the class means and the last-layer classifiers all collapse to the vertices of a Simplex Equiangular Tight Frame (ETF) up to scaling, and (ii) cross-example within-class variability of last-layer activations collapses to zero. This phenomenon is 
called Neural Collapse (NC), which seems to take place regardless of the choice of loss functions. In this work, we justify NC under the mean squared error (MSE) loss, where recent empirical evidence shows that it performs comparably or even better than the \emph{de-facto} cross-entropy loss. Under a simplified unconstrained feature model, % which isolates the topmost layers from the classifier of the neural network, 
we provide the first global landscape analysis for vanilla nonconvex MSE loss and show that the (only!) global minimizers are neural collapse solutions, while all other critical points are strict saddles whose Hessian exhibit negative curvature directions. Furthermore, we justify the usage of rescaled MSE loss by probing the optimization landscape around the NC solutions, showing that the landscape can be improved by tuning the rescaling hyperparameters. Finally, our theoretical findings are experimentally verified on practical network architectures.%, and we observe better NC on MSE compared to cross-entropy loss. These findings could have profound implications for optimization, generalization, and robustness of broad interests. 
\end{abstract}

\section{Introduction}\label{sec:intro}
% \section{Introduction}\label{sec:intro}
%first paragraph: motivations for studying Neural Collapse in deep neural network
%In the past decade, the revival of deep neural networks has led to dramatic success in numerous applications ranging from computer vision, to natural language processing, to scientific discovery and beyond . 
Despite the dramatic success of modern deep neural networks (DNNs) across engineering and sciences \cite{krizhevsky2012imagenet,lecun2015deep,goodfellow2016deep,senior2020improved} that we have witnessed in the past decade, the practice of deep learning has yet been shrouded with mysteries, 
ranging from the design of appropriate network architectures \cite{qi2020deep,martens2021rapid} to the generalization and robustness properties \cite{nakkiran2019deep,yang2020rethinking,madry2018towards} of the learned networks. 
% ranging from the design of appropriate network architectures to the property (e.g. generalization and robustness \citep{nakkiran2019deep,yang2020rethinking,madry2018towards}) of the learned networks. 
For instance, even the right choice of training loss function has not been thoroughly justified. For classification problems, although the cross entropy (CE) loss is the standard choice for network training, recent work \cite{hui2021evaluation} demonstrated with extensive experiments  that DNNs trained with mean-squared error (MSE) loss achieve on par or even better 
performance compared to those of the CE loss.  

Towards demystifying DNN, a recent interesting line of work \cite{papyan2020prevalence,han2021neural,papyan2020traces,fang2021layer,mixon2020neural,graf2021dissecting,ergen2021revealing,zhu2021geometric} studied and characterized the learned deep representations during the terminal phase of training, where several intriguing phenomena have been discovered. In particular, recent seminal work of \cite{papyan2020prevalence,han2021neural} empirically demonstrated that last-layer features and classifiers of a trained DNN exhibit the following \emph{Neural collapse} (\NC) property:
\begin{itemize}
      \item[(\NC1)] \textbf{Variability collapse:} the individual features of each class concentrate to their class-means.
    \item[(\NC2)] \textbf{Convergence to simplex ETF:} the class-means have the same length and are 
    maximally distant; they form a Simplex Equiangular Tight Frame (ETF). 
    \item[(\NC3)] \textbf{Convergence to self-duality:} the last-layer linear classifiers perfectly match their class-means.
    \item[(\NC4)] \textbf{Simple decision rule:} the last-layer classifier is %behaviorally 
    equivalent to a Nearest Class-Center decision rule.
\end{itemize}
It has been empirically demonstrated that the \NC\; persists across the range of canonical classification problems with the CE loss. These results imply that deep networks are essentially learning maximally separable features between classes, and a max-margin classifier in the last layer upon these learned features, touching the ceiling in terms of the training performance. %during training \cy{Maybe ``training performance'' instead of ``performance during training''}. 
Later work theoretically investigated the \NC\;  based on a simplified assumption of the so-called {\it unconstrained feature model} \cite{mixon2020neural} or {\it layer-peeled model} \cite{fang2021layer}, where the features are viewed as free optimization variables. The underlying reasoning is that modern deep networks are often highly overparameterized with the capacity of learning any representations \cite{cybenko1989approximation, hornik1991approximation,lu2017expressive,shaham2018provable}, so that the last-layer features can approximate, or interpolate, any point in the feature space. Under the unconstrained feature model, the work \cite{lu2020neural,weinan2020emergence,mixon2020neural,graf2021dissecting,fang2021layer,ji2021unconstrained,tirer2022extended} showed that the \NC\;solutions are the only global optimal solution for nonconvex training losses under different settings. However, given the nonconvexity of the problem, even under the unconstrained feature model these global optimality results do guarantee that the \NC\;solutions can be efficiently achieved. This has been further resolved by the recent work \cite{zhu2021geometric}, showing that the CE loss function enjoys a benign global optimization landscape under the unconstrained feature model. It shows that every saddle point is a strict saddle with negative curvature, so that the CE loss can be efficiently optimized to the \NC\;solution regardless of the nonconvexity.

%, explaining why the global \NC\;solution can be %efficiently \cy{removing it as the benign landscape does not imply efficiency}
%achieved.

It should be noted that the \NC\;phenomenon is not \emph{solely} pertinent to the particular choice of the CE loss. It has been recently reported \cite{han2021neural}, that DNNs trained with the MSE loss also exhibit very similar \NC\;phenomena but with even \emph{faster} collapse in terms of training epochs and with better (adversarial) robustness. In the meanwhile, the MSE loss is not only appealing for its algebraic simplicity, but it also demonstrates on-par or even better generalization performances compared to the CE loss, as reported by recent line of work \cite{hui2021evaluation}. However, the theoretical study of MSE loss for \NC\;is still limited \cite{mixon2020neural,han2021neural,tirer2022extended}. Under the unconstrained feature model, their work  proved that the continuous gradient flow of the MSE loss converges to \NC\;solutions. In particular, the work \cite{mixon2020neural} relies on linearizations of the ordinary differential equation by assuming very small initializations, which is not well aligned with the practice of deep learning where the weights are usually initialized with non-negligible magnitudes such as by the Kaiming initialization \cite{he2015delving}. Because the choice of the loss function without balanced weight decay, the analysis in \cite{han2021neural} only focuses on the renormalized features and studies the continually renormalized gradient flow.\footnote{The model used in \cite{han2021neural} imposes a weight decay on the classifier, but not on the features. Thus, without renormalization, the weights of the classifier will converge to zero while the features will blow up.} Moreover, in practice deep networks are usually trained using iterative algorithms such as stochastic gradient descent (SGD) with nontrivial stepsizes, rather than using the continuous gradient flows. The work \cite{poggio2020explicit,poggio2020implicit,rangamani2021dynamics,rangamani2022deep} study deep homogeneous classification networks (without bias terms but beyond the unconstrained features model) trained with MSE loss, stochastic gradient descent, and weight decay. In particular, the solutions satisfying the so-called symmetric quasi-interpolation assumption are proved to obey~\NC\; properties, but the properties of other solutions are not investigated~\cite{rangamani2021dynamics,rangamani2022deep}

As far as we know, the work closest to ours is the concurrent work \cite{tirer2022extended}. Under similar unconstrained feature models, the work studies the global optimality condition of \NC\;for the MSE loss for both two-layer and three layer networks, but \emph{not} the global optimization landscape. Additionally, it studies special cases of the MSE loss with either no bias term, or no weight decay on the bias term. In comparison, our work not only study the MSE loss under more general setting with bias term included, but also shows the strict saddle property of the benign nonconvex landscape.

\paragraph{Contributions.}In this work, we provide a thorough analysis of nerual network by examining its last-layer features. % during the terminal phase of training. 
In particular, we work under the unconstrained feature model to characterize the \emph{global} optimization landscape of over-parameterized neural networks trained with the MSE loss.
% In this work, we provide a thorough analysis of the \NC\; phenomena a rescaled version\footnote{It has been widely observed that rescaling can lead to better generalization performance when the number of classes are large (e.g., $K>100$). } of the MSE loss. 
% In particular, we characterize the \emph{global} optimization landscape of the MSE loss for network training based on the unconstrained feature model. 
Our contributions can be highlighted as follows.
\begin{itemize}[leftmargin=*,topsep=0.25em]
\item \textbf{Characterization of global solutions. } We provide a mathematical characterization of {\it all} the global solutions for the last layer features and classifier, showing that they satisfy the \NC~properties with certain choices of regularization parameters. 
    This is in contrast to previous work \cite{mixon2020neural,han2021neural} which only characterize the solutions that are produced by a particular optimization algorithm (i.e., gradient flow).
    Moreover, these work only consider cases that the feature dimension is larger than the number of classes, while our analysis covers all choices of feature dimension. 
\item \textbf{Benign global landscape.} We prove that the loss function is a \emph{strict saddle function} \cite{ge2015escaping,sun2015nonconvex,zhang2020symmetry}, where every critical point is either a global solution or a \emph{strict saddle point}  with negative curvature. 
This implies that there is \emph{no} spurious local minimizer on the optimization landscape. 
Hence, our work is distinguished from %is in contrast to 
previous work \cite{lu2020neural,weinan2020emergence,mixon2020neural,graf2021dissecting,ergen2021revealing,fang2021layer,tirer2022extended} that only characterizes global minimizers.  
% \qq{should we still need to mention this in the last sentence?}
The benign global landscape implies that any method that can escape strict saddle points (e.g. stochastic gradient descent) converges to a global solution that exhibits \NC~(see \Cref{sec:experiment}). 
% This result supports our empirical observation, as shown in \Cref{sec:exp-NC}, that practical overparameterized networks always converge to ETF solutions with a diverse choice of optimization algorithms.
% \item \textbf{Efficient, Algorithmic Independent, Global Optimization to \NC\;Solutions.} \qq{this is the same with the previous work, maybe remove} \cy{Or may summarize in one sentence and combine with the previous point.} \zz{Agree. Given our previous work already provides a geometric analysis, we can replace the last sentence in the last point by one sentence from this point, and remove this point.}
%The benign global landscape implies that any method that can escape strict saddle points (e.g. stochastic gradient descent) converges to a global solution \cite{lee2016gradient} that exhibits \NC. This result supports our empirical observation, as shown in \Cref{sec:exp-NC}, that practical overparameterized networks always converge to ETF solutions with a diverse choice of optimization algorithms.
% \item \qq{something else experimentally? for example, sharper landscape leads to faster/better collapse}
% \item \qq{designed new metric for measuring neural collapse}
\item \textbf{Understanding the \emph{rescaled} MSE.} In practice, rescaling the MSE loss (see \Cref{subsec:rescaled-MSE}) is empirically demonstrated to be critical for obtaining competitive performance compared to the CE loss particularly when the number of classes is large \cite{demirkaya2020exploring,hui2021evaluation}. 
    % We show that our theoretical results for the MSE loss on the global solutions and landscape hold -- hence \NC~exhibits and persists -- for the rescaled MSE as well.
    We show empirically that the \NC~exhibits for rescaled MSE as well. 
    To understand the benefit of the rescaling, we provide a visualization of the optimization landscape w.r.t. unconstrained features, showing that rescaling aligns the gradient direction to be perpendicular to the decision boundary between classes hence may facilitate the convergence of gradient based algorithms to more discriminative features. 
\end{itemize}

Compared to the recent global landscape analysis for the CE loss \cite{zhu2021geometric}, our result implies that both losses learn similar \NC\;features and classifiers when $d\ge K$. Hence, from the \NC\;perspective, this work provides a theoretical explanation for the observations in \cite{hui2021evaluation} that the DNN trained by the MSE loss achieves on par performance compared to that trained with the CE loss. %\xiao{Are we overclaiming here? we don't have results indicating that the on-par performances is a result of NC.} 
Additionally, it should be noted that there are several major differences between our result and \cite{zhu2021geometric}. First, the work of \cite{zhu2021geometric} only studied the setting where the feature dimension $d$ is larger than the number of classes $K$, while we characterized the global optimality for both the cases of $d<K$ and $d\geq K$. We observe dramatically different performance for DNN learned by CE and MSE when $d<K$. Second, for the MSE loss, we showed that the bias term plays an important role\footnote{For the MSE loss, when there is no bias term, the features (and classifier) that minimize the loss function form orthonormal matrices instead of Simplex ETFs when $d\ge K$.} for the solution to be \NC, while for CE loss the \NC\;solution can be achieved without bias terms.

\section{The Problem Setup}\label{sec:problem}
%A deep neural network is essentially a nonlinear mapping $\psi(\cdot): \R^{D} \mapsto \R^K$, which can be modeled by a composition of simple maps: $\psi(\mb x) = \psi^{L} \circ \cdots \circ \psi^2\circ \psi^1(\mb x)$ for $\vx\in\R^D$, where $\psi^\ell(\cdot)$ ($1\leq \ell \leq L)$ are called ``layers''.More precisely, a vanilla $L$-layer neural network can be written as

The goal of deep learning is to learn a multi-layer nonlinear mapping $\psi(\cdot): \R^{m} \mapsto \R^K$, that is able to fit the training data and generalize. More precisely, a deep neural network classifier can be generally written as
\begin{align}\label{eq:func-NN}
    \psi_{\mb \Theta}(\mb x) \;=\;  \mb W \phi_{\vtheta}(\vx)  + \mb b,
\end{align}
where $\phi_{\vtheta}(\cdot):\R^m \mapsto \R^d$ is the feature mapping, on top of which is the linear classifier  $(\mW,\vb)$. $\phi_{\vtheta}(\vx)$ is usually referred to as the \emph{representation} or \emph{feature} of the input $\vx$ learned from the network.
For convenience, we use $\vtheta$ to denote the network parameters in the feature mapping, and $\mb \Theta = \Brac{\vtheta, \mb W,\mb b }$ to denote \emph{all} the network parameters. In this way, the function implemented by a neural network classifier can also be expressed as a linear classifier acting upon $\phi_{\vtheta}(\vx)$.

In this work, we focus on learning deep networks for multi-class classification tasks (say, with $K$ classes), where the class label of a sample $\mb x_{k,i}$ in the $k$-th class is given by a one-hot vector $\vy_k \in \mathbb{R}^K$ with only the $k$th entry equal to unity ($1\leq k \leq K$). Throughout the paper, we study the setting where the number of training samples in each class is balanced, i.e., each class has $n$ training samples. Let $N=Kn$. During the training phase, the task is then to learn the parameters $\mb \Theta$ so that the output of the model on an input sample $\mb x_{k,i}$ approximates the corresponding output $\vy$ (i.e. $\psi_{\mb \Theta}(\mb x_{k,i})\approx \mb y_k$). To quantify this approximation, it can be done by optimizing a simple MSE loss as follows
 \setlength{\belowdisplayskip}{3pt} \setlength{\belowdisplayshortskip}{3pt}
 \setlength{\abovedisplayskip}{3pt} \setlength{\abovedisplayshortskip}{3pt}
\begin{align}\label{eqn:dl-ce-loss}
    \min_{\mb \Theta} \; 
    \frac{1}{2N}\sum_{k=1}^K \sum_{i=1}^{n} \norm{ \psi_{\mb \Theta}(\mb x_{k,i}) - \vy_k }{2}^2 \;+\; \frac{\lambda}{2} \norm{\mb \Theta}{F}^2,
\end{align}
where $\lambda>0$ is the regularization parameter (a.k.a., the weight decay parameter).

%As introduced in \Cref{sec:intro}, recent work \cite{papyan2020prevalence} showed that the features learned by minimizing the above objective (i.e. $\phi_{\mb \theta}(\mb x)$) showcase the \NC\;phenomenon: their within-class variability vanishes, and the features converge to a Simplex ETF.

%The goal of deep learning is to fit the parameters $\mb \Theta$ so that the output of the model on an input samples $\mb x$ approximates the corresponding output $\vy$, i.e. so that $\psi_{\mb \Theta}(\mb x)\approx \mb y$, in expectation over a distribution of input-output pairs, $\mc D$. This can be achieved .  Naturally, the distribution $\mc D$ is unknown, but we have access to training samples that are drawn i.i.d. from $\mc D$. In this way, one can minimize the empirical risk over these samples by optimizing the following problem

\subsection{Basic Problem Formulation Based on Unconstrained Feature Models}\label{sec:layer-peeled-model}

%\qq{rewrite and condense, refer to our previous work on this}

%Despite of recent efforts \cite{xx}, analyzing deep networks is still a tremendously difficult task, mainly due to the nonlinear interactions between a large number of layers. 
Analyzing deep networks is a tremendously difficult task mainly due to the nonlinear interactions between a large number of layers. 
Nonetheless, as argued by a line of recent work \cite{cybenko1989approximation, hornik1991approximation,lu2017expressive,shaham2018provable} that modern deep networks are often highly overparameterized to approximate any continuous function, it motivates us to simplify the analysis by treating the last-layer features as \emph{free} optimization variables $\mb h_{k,i} = \phi_{\vtheta}(\vx_{k,i})\in\mathbb R^d$. Such a simplification is called {\it unconstrained feature model} 
\cite{mixon2020neural} (or {\it layer-peeled model} in \cite{fang2021layer}), which simplifies the study of the last-layer representations of the network.
% which simplifies the study of the last-layer representations of the network as 
% \begin{align*}
%     \psi_{\mb \Theta}(\mb x) \;=\; \mb W \mb h(\mb x) + \mb b,\quad \text{with}\quad \mb W \;:=\; \mb W_L,\quad \mb b\;:=\; \mb b_L. 
% \end{align*}
To simplify the notation, let us denote
\begin{align*}
    \quad \mb W &\;:=\; \begin{bmatrix}
    \mb w^{1 } & \ \mb w^{2 } & \cdots & \mb w^{K }  
    \end{bmatrix}^\top \in \bb R^{K \times d}, \\
    \quad \mb H &\;:=\; \begin{bmatrix}
    \mb H_1 & \mb H_2 & \cdots & \mb H_K
    \end{bmatrix}\in \bb R^{d \times N}  , ~~\text{and}\\
    \quad \mb Y &\::=\; \begin{bmatrix}
     \mb Y_1 & \mb Y_2 &\cdots & \mb Y_K
    \end{bmatrix} \in \bb R^{K \times N},
\end{align*}
%so that $\mb W \in \bb R^{K \times d}$, $\mb H \in \bb R^{d \times N} $, and $\mb Y \in \bb R^{K \times N}$.
where $\mb w^k$ is a row vector of $\mb W$, $\mb H_k:= \begin{bmatrix}\mb h_{k,1} & \cdots & \mb h_{k,n} \end{bmatrix} \in \bb R^{ d \times n}$ contains all the $k$-th class features, and $\mb Y_k := \begin{bmatrix}  \mb y_k  &\cdots & \mb y_k \end{bmatrix} \in \bb R^{K \times n}$ for all $k=1,2,\cdots,K$. Based on the unconstrained feature model, we consider a slight variant of \eqref{eqn:dl-ce-loss}, given by %\qq{discuss rescaling here?}
\begin{align}\label{eq:obj}
     \min_{\mb W , \mb H,\mb b  }  f(\mb W,\mb H,\mb b):= \Big\{\frac{1}{2N} \norm{ \mb W \mb H + \mb b \mb 1_{N}^\top  - \mb Y }{F}^2+ \frac{\lambda_{\mb W} }{2} \norm{\mb W}{F}^2 
     + \frac{\lambda_{\mb H} }{2} \norm{\mb H}{F}^2 + \frac{\lambda_{\mb b} }{2} \norm{\mb b}{2}^2\Big\},
\end{align}
where $\lambda_{\mb W}$, $\lambda_{\mb H},\;\lambda_{\mb b}>0$ are the penalties for $\mb W$, $\mb H$, and $\mb b$, respectively. 

Here, because we treat the last-layer feature $\mb H$ as a free optimization variable, we put the weight decay on $\mb W$ and $\mb H$, which is different from the practice that the weight decay is enforced on all the network parameters $\mb \Theta $ as shown in \eqref{eqn:dl-ce-loss}. Nonetheless, as discussed in \cite{zhu2021geometric}, this idealization is reasonable since the energy of the features (i.e., $\|\mb H\|_F$) can indeed be upper bounded by the energy of the weights at every layer if the inputs are bounded (which holds in practice), implying that the norm of $\mb H$ is \emph{implicitly} penalized by penalizing the norm of $\mb \Theta$. Additionally, for the CE loss, the experiments in \cite{zhu2021geometric} show on-par performance for the two types of weight decay. Thus, we expect similar performances for the MSE loss.

On the other hand, the experiments in \cite{zhu2021geometric,graf2021dissecting} conducted on random labels imply that the strong assumption of unconstrained feature model is reasonable for explaining \NC\;during the training phase: when the network \eqref{eq:func-NN} is highly overparameterized, the learned network in practice will fit to the random labels and neural collapse, regardless of the input. Moreover, as we shall see in the following sections, both theory and experiments demonstrate that such simplification preserves the core properties of last-layer classifiers and features---the \NC\;phenomenon.

\subsection{Rescaled MSE Loss under Unconstrained Features}
\label{subsec:rescaled-MSE}

On the other hand, it should be noted that, when training with the vanilla formulation of the MSE loss \eqref{eqn:dl-ce-loss}, empirically good performances are reported \emph{only} when the number of classes is small (e.g., CIFAR10 \cite{krizhevsky2009learning} with $K<100$). When training for a large number of classes such as ImageNet \cite{deng2009imagenet}, to achieve better performance \emph{rescaling} is often needed \cite{demirkaya2020exploring,hui2021evaluation}.  Intuitively, the basic idea is to rescale the MSE loss \eqref{eq:obj} by a pair of positive scalars $(\alpha,M)$, 
\begin{align}\label{eq:obj-rescaled}
         \min_{\mb W , \mb H,\mb b  } \;   \frac{1}{2N} \norm{ \mb \Omega_{\alpha}^{\odot 1/2} \odot \paren{ \mb W \mb H + \mb b \mb 1^\top  - M \mb Y} }{F}^2 \;+\; \frac{\lambda_{\mb W} }{2} \norm{\mb W}{F}^2 + \frac{\lambda_{\mb H} }{2} \norm{\mb H}{F}^2 + \frac{\lambda_{\mb b} }{2} \norm{\mb b}{2}^2,
\end{align}
so that we can put more emphasize on training the correct class. Here, $\odot$ denotes the entry-wise Hadamard product, $\mb \Omega^{\odot 1/2}$ means taking square root for each element, and 
\begin{align*}
    \mb \Omega_{\alpha} \;=\; \begin{bmatrix} \mb \omega_1 \mb 1_n^\top & \cdots & \mb \omega_K \mb 1_n^\top \end{bmatrix}, \quad \text{with}\;\;\mb \omega_k(\alpha) \in \bb R^K \;\;\text{and}\;\; \omega_{ki}(\alpha) \;=\; \begin{cases} 
    \alpha, & i = k, \\
    1, & \text{otherwise}.
    \end{cases}
\end{align*}
In comparison to \cite{mixon2020neural,han2021neural,tirer2022extended}, our work not only studies \NC\; under the vanilla setting \eqref{eq:obj} but also investigates the more practical rescaled version of the MSE loss \eqref{eq:obj-rescaled}. In particular, in \Cref{subsec:landscape-rescaling}, we provide geometric intuitions on why rescaling would be a better choice for loss design. We will corroborate our reasoning via experiments on practical network training in \Cref{sec:experiment}.

\section{Main Theoretical Results}\label{sec:main}

%In next section, we will analyze the optimization landscape for the objective function in \eqref{eq:obj}.

In this section, we present our study on global optimality conditions as well as geometric properties of the nonconvex (rescaled) MSE loss under the unconstrained feature model.

%\qq{mention about the rescaled version here}

\subsection{Global Optimality Conditions}
\label{subsec:global_optim}

First, we study the nonconvex MSE loss \eqref{eq:obj} by characterizing its global solutions under different settings of the feature and class dimensions. We show that the only global solutions of \eqref{eq:obj} are neural collapsing, satisfying the \NC\;properties introduced at the beginning of \Cref{sec:intro}.
%We begin by characterizing the global solutions of \eqref{eq:obj}, showing that the global minimizers exhibit \NC1, \NC3, and a variant of \NC2 under different settings. 

% in the following result.
%First of all, we can show that the Simplex ETF are the only global solutions. 
\begin{theorem}[Global Optimality Conditions]\label{thm:global-minima}
%Assume the balanced training samples and the unconstrained feature model of the network introduced in \Cref{sec:layer-peeled-model}.
Assume that the number of training samples in each class is balanced, $n = n_1 = \cdots = n_K$, and let $(\mW^\star, \mH^\star,\vb^\star)$ be a global minimizer of the vanilla MSE loss in \eqref{eq:obj}. Let $\ol\mH^\star = \begin{bmatrix}\ol\vh^\star_1 & \cdots \ol \vh^\star_K \end{bmatrix}$, with $\ol \vh^\star_k$ being the mean of the $k$-th class features.
	Then, $(\mW^\star,\mH^\star,\mb b^\star)$ satisfies the following properties:% for different $d$, $K$, $\lambda_{\mb W}$, $\lambda_{\mb H}$, and $\lambdab$:
\begin{itemize}[leftmargin=*]
    \item %[(\NC1,3)] 
    If $\lambdaW\lambdaH < \frac{1}{NK}$, then $(\mW^\star,\mH^\star)$
satisfies \NC1 and \NC3 as
\begin{gather*}
      \vh_{k,i}^\star \;=\;  \ol\vh_k^\star, \ \sqrt{ \frac{ \lambda_{\mb W}  }{ \lambda_{\mb H} n } } \vw^{\star k} \;=\;  \ol\vh_k^\star ,\quad \forall \; k\in[K],\; i\in[n].
\end{gather*}
Otherwise, if $\lambdaW\lambdaH \ge \frac{1}{NK}$, then $\mW^\star = \vzero, \mH^\star = \vzero$.
 \item %[(\NC2)] 
 If $\lambdaW\lambdaH < \frac{1}{NK}$, then $\ol\mH^\star$ further obeys the following properties (\NC2) for different $d$:
\begin{enumerate}[leftmargin=*]
\item If $d< K-1$:  we have $\ol\mH^{\star\top}\ol\mH^\star = C_1 \calP_d(\mId - \frac{1}{K} \mb 1_K \mb 1_K^\top)$, where $\calP_d(\mb M)$ denotes the best rank-$d$ approximating of $\mb M$;
\item  If $d = K-1$: we have $\ol\mH^{\star\top}\ol\mH^\star= C_2 (\mb I - \frac{1}{K} \mb 1_K \mb 1_K^\top)$;
 \item  If $d\ge K$: we have $\ol\mH^{\star\top}\ol\mH^\star =$ \e
  \begin{cases} C_3 \paren{ \mb I - \frac{1}{K} \mb 1_K \mb 1_K^\top},   ~~\textup{if} \ \lambdab \le \frac{\sqrt{KN \lambdaW\lambdaH}}{1 - \sqrt{KN \lambdaW\lambdaH}} \\ C_4 \big(\mId - \frac{\sqrt{n\lambdaW\lambdaH}}{\lambdab(1 - \sqrt{KN\lambdaW\lambdaH} )} \vone_K \vone_K^\top \big),  ~~\textup{otherwise}
        \end{cases}
    \label{eq:thm-eq-approx-ETF}    \ee
  where  $\frac{\sqrt{n\lambdaW\lambdaH}}{\lambdab(1 - \sqrt{KN\lambdaW\lambdaH})}\le \frac{1}{K}$ in the second case since $\lambdab \ge \frac{\sqrt{KN \lambdaW\lambdaH}}{1 - \sqrt{KN \lambdaW\lambdaH}}$. 
 %when $\lambdab \le \frac{\sqrt{KN \lambdaW\lambdaH}}{1 - \sqrt{KN \lambdaW\lambdaH}}$,
 %$b^\star \ge \frac{1}{K} - \sqrt{n\lambdaW\lambdaH}$ we have $\ol\mH^{\star\top}\ol\mH^\star = C_3 \paren{ \mb I - \frac{1}{K} \mb 1_K \mb 1_K^\top}$; otherwise, $\ol\mH^{\star\top}\ol\mH^\star = C_4 \paren{ \mId - \frac{\sqrt{n\lambdaW\lambdaH}}{\lambdab(1 - \sqrt{KN\lambdaW\lambdaH} )} \vone_K \vone_K^\top }$, where $\frac{\sqrt{n\lambdaW\lambdaH}}{\lambdab(1 - \sqrt{KN\lambdaW\lambdaH})}\le \frac{1}{K}$.
 %\qq{Here, I think the statement would be better changed to $\lambda_{\mb b}$ instead of $\mb b$, given the relationship of $\lambda_b$ and $\mb b$}
\end{enumerate}
Here, $C_1$, $C_2$, $C_3$, and $C_4$ are some positive numerical constants that depend on $\lambdaW,\lambdaH,\lambdab$. % \cy{They depend on parameters $\lambdaH, \lambdaW$, right? Can we say they are constant in this case?} \zz{Yes, they depend on the weight decay parameters. Added!}
  \item %[(Bias)] 
  The bias satisfies $\mb b^\star = b^\star \mb 1_K$ with $b^\star \le \frac{1}{K}$ given by:
    \begin{enumerate}[leftmargin=*]
        \item If $d<K$: we have $b^\star  = \frac{1}{K(\lambdab + 1)}$;
        \item Otherwise, $b^\star = \begin{cases} \frac{1}{K(\lambdab + 1)}, &   \lambdab \le \frac{\sqrt{KN \lambdaW\lambdaH}}{1 - \sqrt{KN \lambdaW\lambdaH}}, \\ \frac{\sqrt{n\lambdaW\lambdaH}}{\lambdab}, &  \text{otherwise}. 
        \end{cases}$
    \end{enumerate}
    In particular, when $\lambdab \rightarrow 0$, we have $b^\star \rightarrow \frac{1}{K}$; when $\lambdab \rightarrow \infty$, we have $b^\star \rightarrow 0$.
\end{itemize}
\end{theorem}
%At a high level, our proof technique finds lower bounds for the loss in \eqref{eq:obj} and studies the conditions for the lower bounds to be achieved, similar to \cite{lu2020neural,fang2021layer,zhu2021geometric,tirer2022extended,graf2021dissecting}. 
We postpone the detailed proof to Appendix \ref{app:thm-global}. In the following, we discuss the implications of \Cref{thm:global-minima} in detail.
\begin{itemize}[leftmargin=*,topsep=0.25em,itemsep=0.1em]
    \item \textbf{Implications on the choice of the feature dimension $d$.} As we observe from \Cref{thm:global-minima}, for the MSE loss \eqref{eq:obj}, any global solution always exhibits variability collapse (\NC1) and self-duality (\NC3). However, the convergence of  class means to simplex ETF (\NC2) critically depends on the feature dimension $d$. When $d\geq K-1$, for proper choices of $\lambdaW$, $\lambdaH$, and $\lambdab$, the global configuration of the class mean $\ol\mH^\star$ is always a simplex ETF. In particular, when $d=K-1$, the simplex ETF configuration even does not depend on $\lambdab$. On the other hand, if $d< K-1$, our theory implies that the global solution for $\ol\mH^\star$ is only the best rank-$d$ approximation of the simplex ETF, where the class-means of the each class are neither having equal length nor being maximally pairwise-distanced.  This result is consistent with the fact that $K$ vectors in $\bb R^d$ cannot form a $K$-Simplex ETF if $K>d-1$, and supports the practice of learning overparameterized network for choosing $d\ge K$.\footnote{For example, the dimension of the features of a ResNet \cite{he2016deep} is typically set to $d=512$ for CIFAR10 \cite{krizhevsky2009learning}, a dataset with $K = 10$ classes. This dimension grows to $d=2048$ for ImageNet \cite{deng2009imagenet}, a dataset with $K = 1000$ classes.}
    \item \textbf{Comparison to the CE loss.} For the CE loss under the unconstrained feature model, when $d \geq K$ recent work \cite{zhu2021geometric} showed that any global solution satisfies all three \NC\ properties regardless of choices of the weight decay parameters (i.e., $\lambda_{\mb W}$, $\lambda_{\mb H}$, and $\lambdab$). Moreover, the bias term there becomes zero.
    In contrast, \Cref{thm:global-minima} shows that the solution with the MSE loss is dependent upon choice of regularization parameters $\lambdaW, \lambdaH, \lambdab$ and that the class mean $\ol\mH^\star$ may not be a simplex ETF. Moreover, the bias term is essential to achieve simplex ETF solutions for MSE loss. Without the bias term (i.e., $\lambdab\rightarrow \infty$), \eqref{eq:thm-eq-approx-ETF} implies that the class mean $\ol\mH^\star$ becomes an orthonormal matrix even when $d\ge K$. Thus, the analysis of global optimality conditions for the MSE loss is more complicated than for the CE loss\footnote{\revise{The proof of \Cref{thm:global-minima} is also dramatically different to the one for CE loss in \cite{zhu2021geometric}:  the latter mainly shows that \NC~solutions have small objective value than others since \NC~solutions are the only global minimizers, while the proof of \Cref{thm:global-minima} directly analyzes the global minimizers for different scenarios.
    }
    }.
    %However, if $K$ is large enough then $\ol\mH^\star$ is close enough to a simplex ETF, as $\frac{1}{K}\approx \frac{\sqrt{n\lambdaW\lambdaH}}{\lambdab(1 - \sqrt{KN\lambdaW\lambdaH})} \approx0$ hence $\mb I - \frac{1}{K} \mb 1_K \mb 1_K^\top \approx  \mId - \frac{\sqrt{n\lambdaW\lambdaH}}{\lambdab(1 - \sqrt{KN\lambdaW\lambdaH} )} \vone_K \vone_K^\top$.  

    \item \textbf{Comparison to previous work \cite{mixon2020neural,han2021neural}.}  \revise{As discussed in \Cref{sec:intro}, the previous work \cite{mixon2020neural,han2021neural} only characterize the solutions to \eqref{eq:obj} that are produced by a particular optimization algorithm (i.e., gradient flow) and under specific cases such as $\lambdab\rightarrow 0$ and the feature dimension is larger than the number of classes. In contrast, we characterize the global optimality conditions for the MSE loss \eqref{eq:obj} and  our analysis covers all choices of feature dimension and weight decay parameters. }
    
    \item \textbf{Extension to the rescaled MSE.} Although our current analysis is only for the vanilla MSE loss \eqref{eq:obj}, we expect that similar global optimality results should also hold for the rescaled version \eqref{eq:obj-rescaled}. This has been corroborated by our experimental results in \Cref{sec:experiment}. Notice that if we fix $\alpha=1$ in \eqref{eq:obj-rescaled}, the analysis only with large $M$ is simple and remain the same as \Cref{thm:global-minima}. However, dealing with both $\alpha$ and $M$ requires extra technicalities, that we leave for future work.
   % In contrast, the global minimizers of the cross entropy loss under the same unconstrained features model \eqref{eq:obj} always exhibit the three \NC\ properties \cite{zhu2021geometric} as long as $d\ge K$, regardless of choices of the weight decay parameters. In particular, $\ol\mH^\star$ forms a simplex ETF and the cosine of the angle between each pair of different class features is equal to $-\frac{1}{K}$. However, we note that even $\ol\mH^\star$ of the MSE loss does not form a simplex ETF, in which the cosine of the angle becomes $-\frac{1}{K\lambdab(1 - \sqrt{KN\lambdaW\lambdaH})}$ as discussed above,   the difference is negligible when $K$ is large (e.g., $K = 100$) since in this case $\frac{1}{K}\approx \frac{1}{K\lambdab(1 - \sqrt{KN\lambdaW\lambdaH})} \approx0$.  
 %   \item \textbf{Relationship to the existing work on the MSE loss~\cite{mixon2020neural,han2021neural}.}
\end{itemize}

\subsection{Characterizations of The Benign Global Landscape}\label{subsec:main-geometry}

\Cref{thm:global-minima} implies that the (only!) global minimizers to \eqref{eq:obj} are those satisfying \NC\;properties. However, the MSE loss function is \emph{nonconvex}, hence it is not obvious whether the benign global solutions can be \emph{efficiently} achieved even under the unconstrained feature model. To deal with this challenge, in the following we further investigate the global optimization landscape of \eqref{eq:obj}. By leveraging recent advances on nonconvex optimization \cite{sun2015nonconvex,ge2015escaping,sun2016complete,sun2018geometric,zhang2020symmetry,qu2020geometric,qu2020finding}, we first show that our nonconvex MSE loss \eqref{eq:obj} without bias term is a \emph{strict saddle} function that every non-global critical point is a saddle point with negative curvature (i.e., its Hessian has at least one negative eigenvalue). 
\begin{theorem}\label{thm:global-geometry-no-bias} \emph{\bf (Benign landscape for MSE without bias term)} 
The following MSE loss without bias term
\[\frac{1}{2N} \norm{ \mb W \mb H - \mb Y }{F}^2\nonumber + \frac{\lambda_{\mb W} }{2} \norm{\mb W}{F}^2 
     + \frac{\lambda_{\mb H} }{2} \norm{\mb H}{F}^2 \]
is a strict saddle function with no spurious local minimum. That is, any of its critical point is either a global minimizer, or it is a strict saddle point whose Hessian has a strictly negative eigenvalue.
\end{theorem}

We postpone the proof to Appendix \ref{app:thm-global} (see \Cref{lem:global-fact-nuclear}). By viewing $\mW$ and $\mH$ as two factors of a matrix $\mb Z = \mb W \mb H $, the formulation in \eqref{eq:obj} is closely related to nonconvex low-rank matrix problems \cite{haeffele2015global,ge2016matrix,bhojanapalli2016global,ge2017no,li2019non,li2019symmetry,chi2019nonconvex} with the Burer-Moneirto factorization approach \cite{burer2003nonlinear}. In particular, the work \cite{ciliberto2017reexamining,li2019non} studied a similar problem with $\lambdaW = \lambdaH$, but only for particular choices of $d$: $d$ is either required to be exactly the rank of the solution of the corresponding convex problem \cite{li2019non}, or relatively large in \cite{ciliberto2017reexamining}. In contrast, our \Cref{thm:global-geometry-no-bias} characterizes the benign landscape for all choices of feature dimension.

The following result establishes global optimization landscape of the MSE loss \eqref{eq:obj}. 

\begin{theorem}\label{thm:global-geometry} \emph{\bf (Benign landscape for MSE loss \eqref{eq:obj})}
    Assume that the feature dimension $d$ is larger than the number of classes $K$. %(i.e., $d> K$). 
    The nonconvex MSE loss function $f(\mb W,\mb H,\mb b)$ in \eqref{eq:obj} is a strict saddle function.
\end{theorem}

This result is similar to that of \cite[Theorem 3.2]{zhu2021geometric}, % which showed that the nonconvex CE loss is a strict saddle function under the unconstrained feature model. 
which showed that the unconstrained feature model with CE loss is a strict saddle function. 
The high level proof idea for \cite{zhu2021geometric} is to construct the negative curvature direction for saddle points in the null space of $\mW\in\R^{K\times d}$. Because the proof in \cite{zhu2021geometric} actually holds more generally for any smooth convex loss function with weight decay, 
the same technique also offers a proof for \Cref{thm:global-geometry} (and potentially can extend \Cref{thm:global-geometry} for the rescaled MSE in \eqref{eq:obj-rescaled}). 
Here, it should be noted that we make the assumption $d>K$ so that the null space of $\mW\in\R^{K\times d}$ always exists. However, we believe the strict saddle property holds for any $d$ and leave it as future work. %Next, we discuss the implication of \Cref{thm:global-geometry} and its relationship to prior arts.

As a consequence, if $\mb H$ is a free optimization variable, this implies that the global solutions of the training problem in \eqref{eq:obj} can be efficiently found by many first-order and second-order optimization methods \cite{bottou2018optimization}. In particular, (stochastic) gradient descent with random initialization is guaranteed \cite{ge2015escaping,lee2016gradient} to almost surely find a global minimizer for strict saddle functions with no spurious local minima, which is the case for our problem \eqref{eq:obj}. In comparison, existing results on MSE loss \cite{mixon2020neural,han2021neural} only studied the trajectory of gradient flows \eqref{eq:obj} on either the linear terms \cite{mixon2020neural} or the central path component \cite{han2021neural}, which is insufficient to explain/guarantee efficient, global convergence of iterative optimization algorithms. %\qq{figure 2 needs to be modified}

%We now investigate the global nonconvex optimization landscape of \eqref{eq:obj}. Following \cite{sun2016complete1,ge2015escaping}, we call $f$ a strict saddle function if any critical point that is not a local minimizer is a strict saddle with negative curvature, i.e., the Hessian at this critical point has at least one negative eigenvalue. Our next result implies that the training loss \eqref{eq:obj} is a strict saddle function, and every local minimizer is global. 

% \subsection{Benefit of rescaled MSE loss}

\subsection{Delving Deeper into Optimization Landscapes: Why Rescaling Helps?}
\label{subsec:landscape-rescaling}

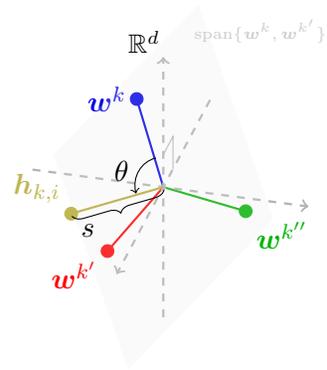
\begin{wrapfigure}{r}{0.33\textwidth}
\vspace{-0.3in}
    \begin{center}
    \tdplotsetmaincoords{60}{110}
    \begin{tikzpicture}[scale=2.0]
     
      \tdplotsetrotatedcoords{-25}{0}{30}
      \begin{scope}[tdplot_rotated_coords]
        \coordinate (origin) at (0,0,0);
      
        % one-hots
        \draw [thick,black!10!blue] (0,0,0) -- (-1/4,-1/4,2/4) node (wkvec) [anchor=east]{$\bm w^k$};
        \draw plot [mark=*, mark size=1.2, mark options={draw=black!10!blue, fill=black!10!blue}] coordinates{(-1/4,-1/4,2/4)};
        \draw [thick,red] (0,0,0) -- (2/4,-1/4,-1/4) node[anchor=north east]{$\bm w^{k'}$};
        \draw plot [mark=*, mark size=1.2, mark options={draw=red, fill=red}] coordinates{(2/4,-1/4,-1/4)};
        \draw [thick,black!30!green] (0,0,0) -- (-1/4,2/4,-1/4) node[anchor=north west]{$\bm w^{k''}$};
        \draw plot [mark=*, mark size=1.2, mark options={draw=black!30!green, fill=black!30!green}] coordinates{(-1/4,2/4,-1/4)};
        
        % coordinate axis
        \draw [black!30!white,line width=0.3mm,dashed,->] (-1.2,0,0) -- (1.2,0,0);
        \draw [black!30!white,line width=0.3mm,dashed,->] (0,-0.9,0) -- (0,1,0);
        \draw [black!30!white,line width=0.3mm,dashed,->] (0,0,-1) -- (0,0,1);
        \draw [black!30!white] (0,0,0.25) -- (-0.25,0,0.25) -- (-0.25,0,0);
        
        % plane
        \filldraw[draw=none,fill=gray!20, opacity=0.2] 
			(1/3,-2/3,1/3) -- (-0.9,0,0.9) -- (-1/3,2/3,-1/3) -- (0.9,0,-0.9) -- cycle;
		\node[gray!40] at (-1,0.4,0.7) {\tiny span$\{\bm w^k, \bm w^{k'}\}$};
		
		% h
		\draw [thick,black!30!yellow] (0,0,0) -- (2/4,-2/4,0) node (hvec) [anchor=south east]{$\bm h_{k, i}$};
        \draw plot [mark=*, mark size=1.2, mark options={draw=black!30!yellow, fill=black!30!yellow}] coordinates{(2/4,-2/4,0)};
        
        % theta
        % \draw pic [draw, ->, angle eccentricity=1, transform shape] {angle = wkvec--origin--hvec};
        \tdplotdefinepoints(0,0,0)(-1/4,-1/4,2/4)(2/4,-2/4,0)
        \tdplotdrawpolytopearc[->]{0.2}{anchor=east}{$\theta$}
        
        % s
        \draw[decoration={brace,raise=1pt,amplitude=4pt},decorate]
			(0,0,0) -- (2/4,-2/4,0);
		\draw (2/4,-2/4,0) node [anchor=north west]{$s$};
		
		% R^d
		\node at (-1,-0.4,0.5) {\small $\bb R^d$};
      \end{scope}
    \end{tikzpicture}
    \end{center}
	\caption{\textbf{An illustration of the visualization method.}}
	\label{fig:visualization-setup}
\end{wrapfigure}
While our global landscape analysis for the vanilla MSE loss \eqref{eq:obj} in \Cref{subsec:main-geometry} implies that a gradient based algorithm converges to global \NC\;solutions \emph{asymptotically} \cite{lee2016gradient}, it did not characterize the rate of convergence -- in other words, how fast an optimization method converges. Often around the global solutions (i.e., the simplex ETF), we expect that the landscape has certain regularity condition which measures how well-aligned between the negative gradient direction and the direction towards the global solution. Thus, the regularity conditions in turn will characterize how fast a gradient based method converges. For better understanding the regularity properties and algorithmic convergences, we use visualization techniques to visualize the optimization landscape of MSE losses around the global ETFs solutions. In particular, our visualization sheds light on (\emph{i}) why training with vanilla MSE loss performs worse than that of the CE loss, and (\emph{ii}) how the rescaling techniques in \Cref{subsec:rescaled-MSE} improves the performance of the MSE loss.

\begin{figure*}[t]
    \centering
    \subfloat[\scriptsize Vanilla MSE ($\alpha=1, M = 1$)\label{fig:visualization-mse}]{
        \includegraphics[width=0.23\textwidth,trim={0cm 0cm 2cm 1cm},clip]{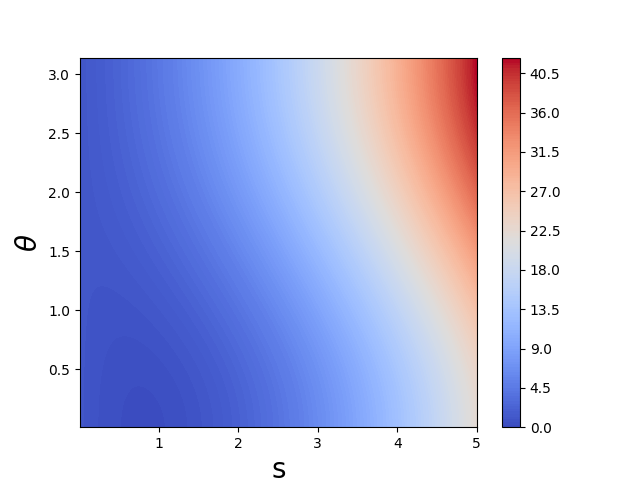}
        \includegraphics[width=0.23\textwidth,trim={3cm 1cm 2cm 2cm},clip]{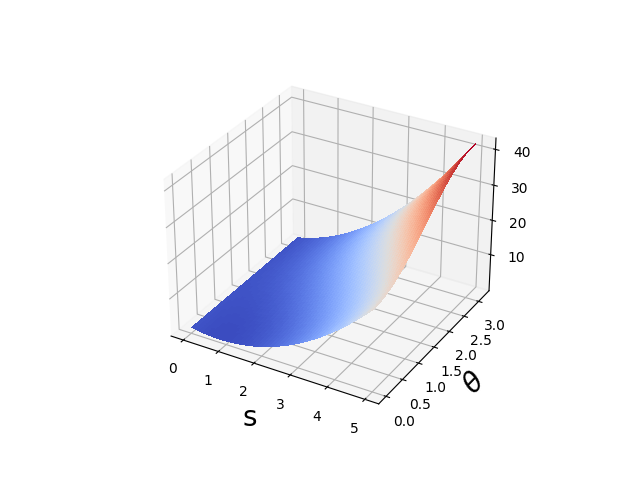}
    } 
    ~
    \subfloat[\scriptsize Cross Entropy \label{fig:visualization-ce}]{
        \includegraphics[width=0.23\textwidth,trim={0cm 0cm 2cm 1cm},clip]{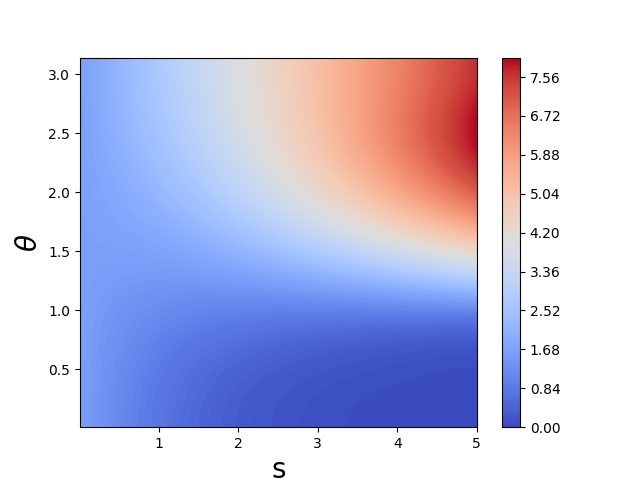}
        
        \includegraphics[width=0.23\textwidth,trim={3cm 1cm 2cm 2cm},clip]{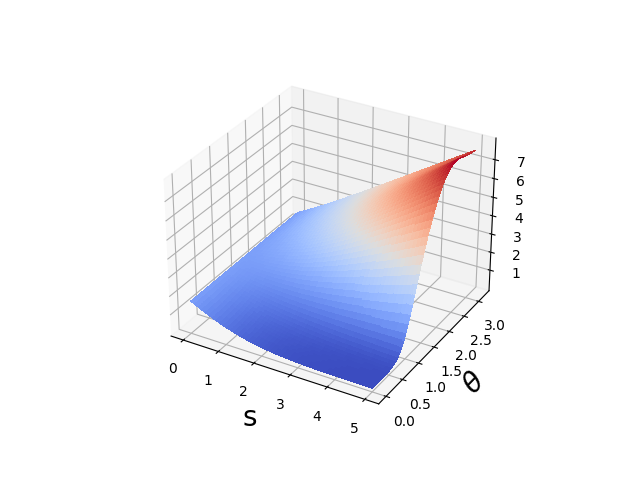}
    }  \\
    ~
    \subfloat[\scriptsize Rescaled MSE ($\alpha=5, M = 1$) \label{fig:visualization-rmse-alpha}]{
        \includegraphics[width=0.23\textwidth,trim={0cm 0cm 2cm 1cm},clip]{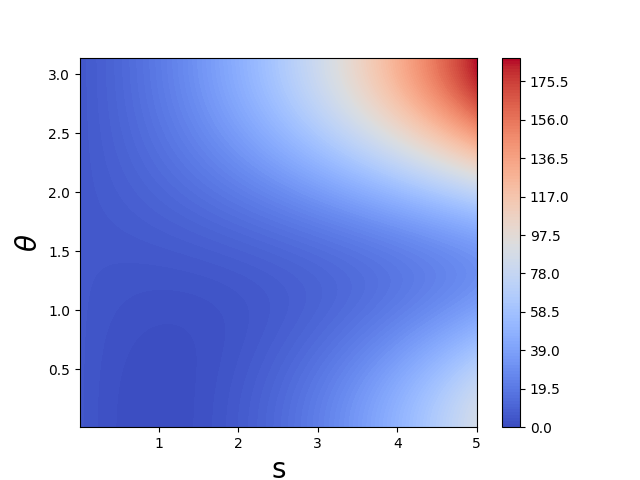}
        
        \includegraphics[width=0.23\textwidth,trim={3cm 1cm 2cm 2cm},clip]{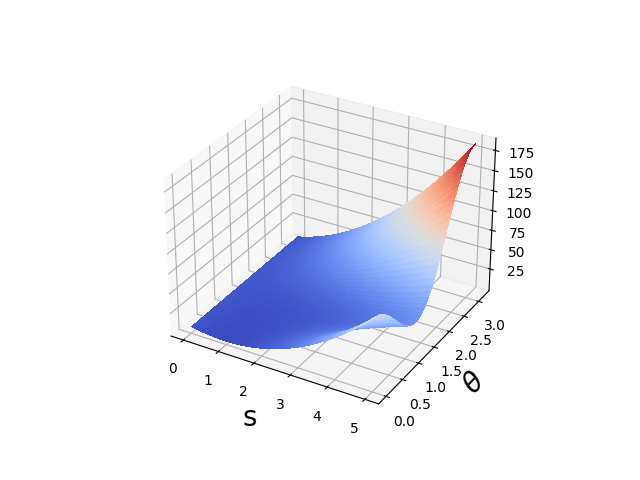}
    } 
    ~
    \subfloat[\scriptsize Rescaled MSE ($\alpha=1, M = 5$)  \label{fig:visualization-rmse-M}]{
        \includegraphics[width=0.23\textwidth,trim={0cm 0cm 2cm 1cm},clip]{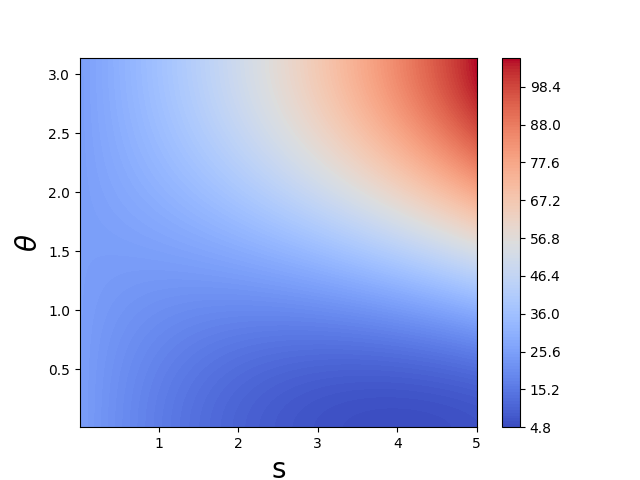}
        
        \includegraphics[width=0.23\textwidth,trim={3cm 1cm 2cm 2cm},clip]{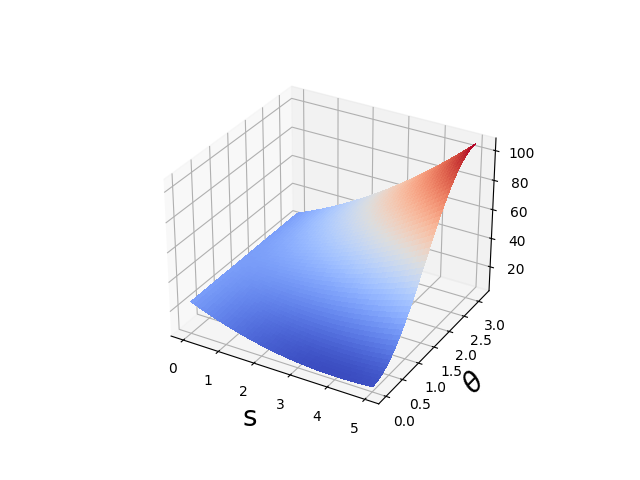}
    } 
    
    \caption{\textbf{Visualization of optimization landscape with different losses.} We fix $\mb W$ as a simplex ETF and illustrate the landscape only w.r.t. a feature $\mb h_{k, i}$. For each plot, the $s$-axis denotes $\norm{\mb h_{k, i}}{2}$, and the $\theta$-axis denotes the angle $\arccos\paren{ \innerprod{\mb h_{k, i}}{\mb w^k} }$.
    }
    \label{fig:visualization}
\end{figure*}

Even under the unconstrained feature model, visualization of the MSE loss landscape could still be difficult, which is due to the fact that the variables $\mb H, \mb W$, and $\mb b$ are all high-dimensional. Here, we further simplify the problem by assuming $\mb b = \mb 0$ and that $\mb W$ is at the global optimum and forms a simplex ETF. Thus, we can examine the landscape only with respect to (w.r.t.) the feature vectors $\mb h_{k,i}$ for the $k$th class. Although $\mb h_{k,i}\in \bb R^d$ is still high-dimensional for large $d$, we plot the optimization landscape by restricting $\mb h_{k,i}$ to a 2D plane spanned by $\{\mb w^k, \mb w^{k'}\}$, where $\mb w^k$ is the classifier for the $k$th class and $k'\not = k$ can be chosen arbitrarily because the simplex ETF is invariant to rotations. Finally, we visualize the landscape using the polar coordinates, where the $s$-axis denotes the $\ell_2$ norm of $\mb h_{k, i}$ and the $\theta$-axis denotes the angle between $\mb h_{k, i}$ and $\mb w^k$ (see \Cref{fig:visualization-setup} for an illustration). The predicted membership for $\mb h_{k, i}$ is determined by $\theta$ and is invariant to $s$. Hence, larger gradient along the $\theta$ direction may help with learning more discriminative features. See Appendix \ref{sec:appendix-visualization} for a formal explanation. This design choice allows us to examine the gradient in directions co-linear to (i.e., with varying $s$) and perpendicular to (i.e., with varying $\theta$) the decision boundary separately. %\qq{figures needs to be enlarged}

In \Cref{fig:visualization}, the visualizations of landscapes of different loss functions are provided. As we observe from  \Cref{fig:visualization-mse}, the landscape of vanilla MSE loss is steep w.r.t. $s$ while it is flat w.r.t. $\theta$. Because the size of $\theta$ determines the closeness to the right class, this implies that optimizing the vanilla MSE loss will take a longer time to converge to a desired solution with $\theta \approx 0$. In contrast, the landscape of CE loss in \Cref{fig:visualization-ce} is steeper w.r.t. $\theta$ than w.r.t. $s$ in a large region where $s > 1$ and $\theta < 1.5$. This difference of the landscapes around the global solutions potentially explains why CE is a preferred choice than the vanilla MSE, given that the features $\mb h_{k,i}$ would converge faster to the simplex ETF solutions via optimizing the CE loss. Nonetheless, the issue with the vanilla MSE can be mitigated via the rescaling approach that we discussed in \Cref{subsec:rescaled-MSE}. As shown in \Cref{fig:visualization-rmse-alpha,fig:visualization-rmse-M}, the rescaled MSE loss \eqref{eq:obj-rescaled} (with large $M$, in particular), leads to a ``better'' optimization landscape similar to that of the CE loss. Therefore, through studying the \NC\;and corresponding optimization landscapes, our work provides intuitive explanations on (\emph{i}) the incompetence of the vanilla MSE loss \eqref{eq:obj}, and (\emph{ii}) the effectiveness of rescaling \eqref{eq:obj-rescaled} for classification tasks.

\section{Experiments}\label{sec:experiment}

\begin{figure*}[t]
    \centering
    \subfloat[$\mc {NC}_1$]{\includegraphics[width=0.2\textwidth]{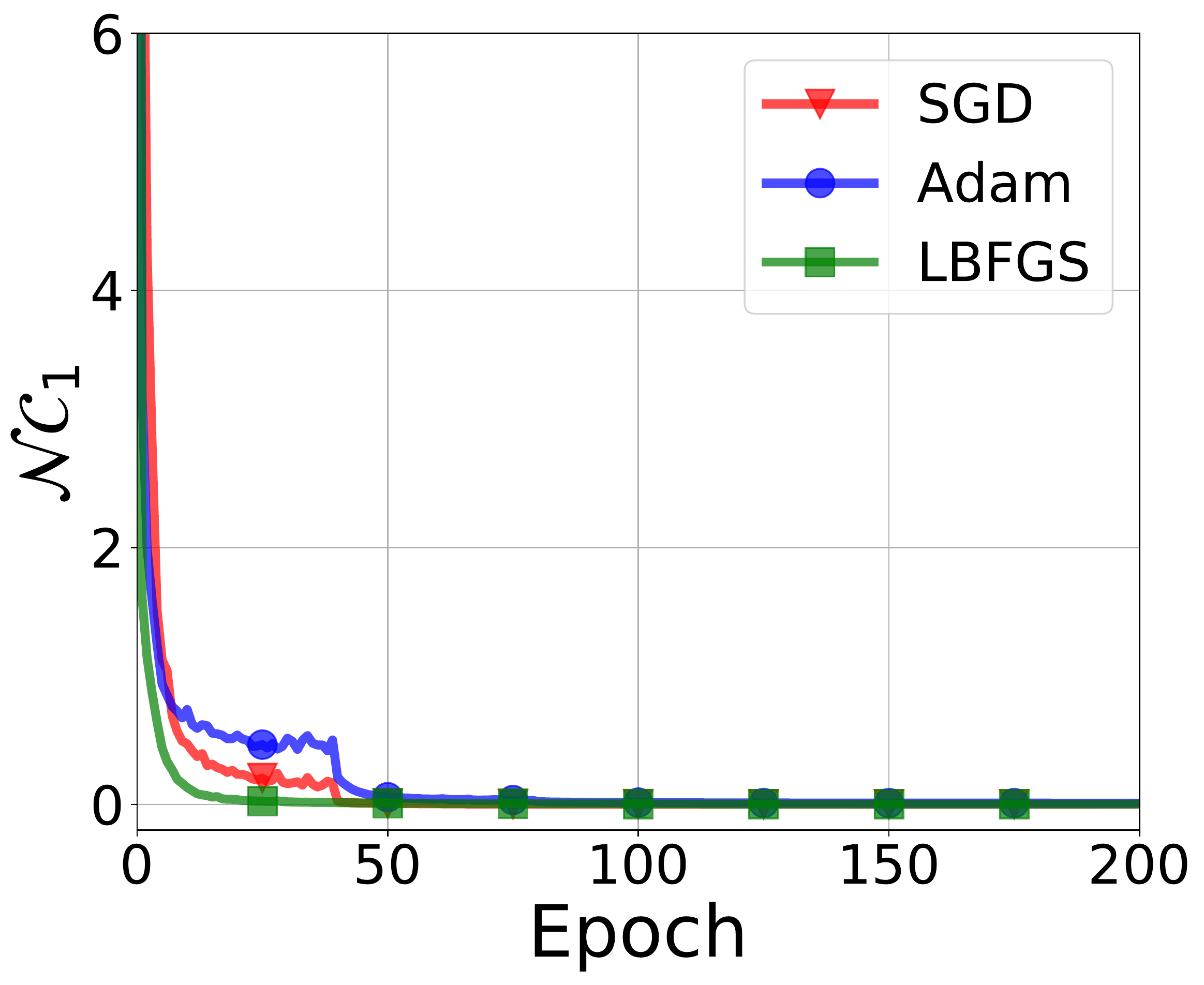}} \
    \hspace{0.08in}
    \subfloat[$\mc {NC}_2$]{\includegraphics[width=0.2\textwidth]{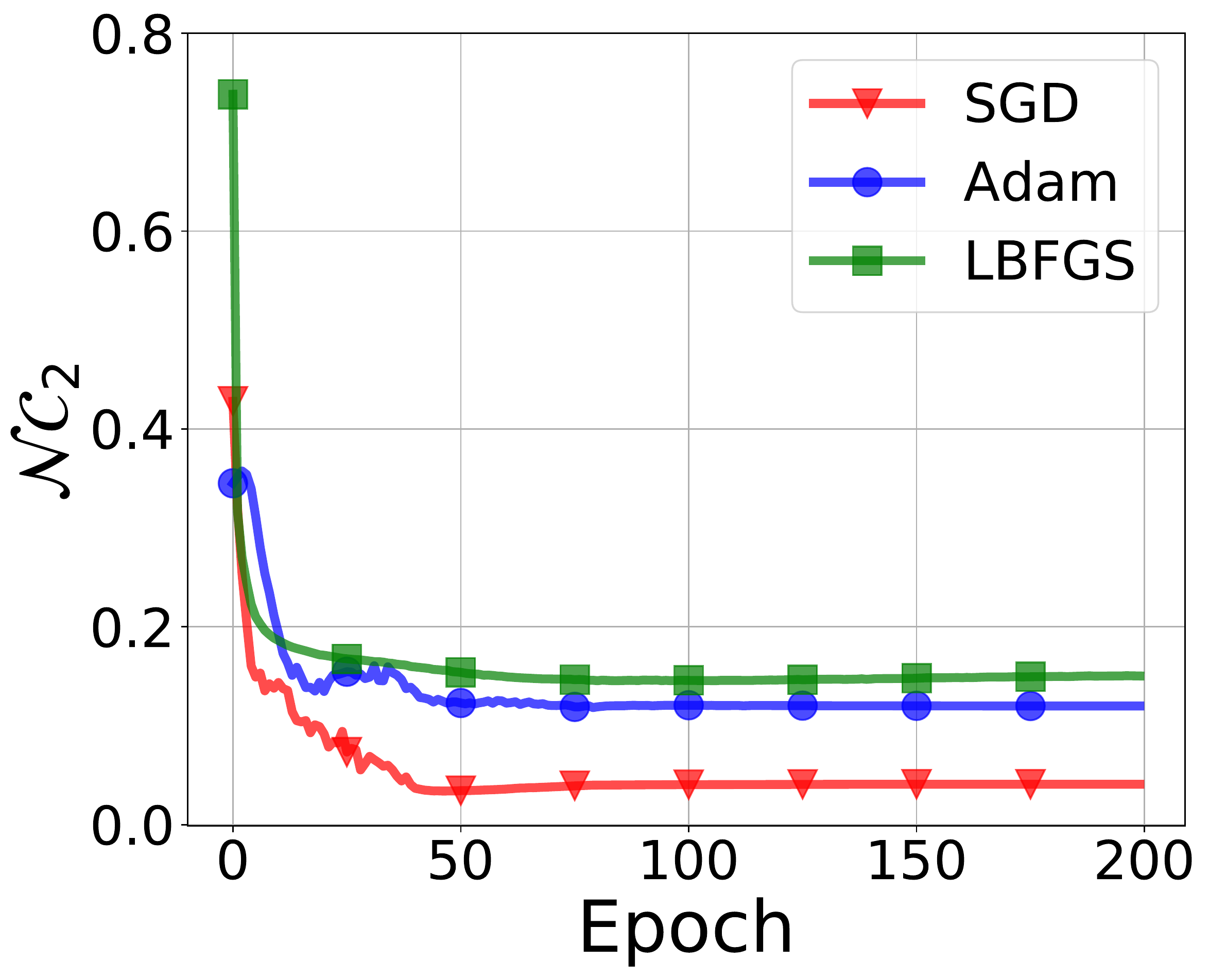}} \
    \hspace{0.08in}
    \subfloat[$\mc {NC}_3$]{\includegraphics[width=0.21\textwidth]{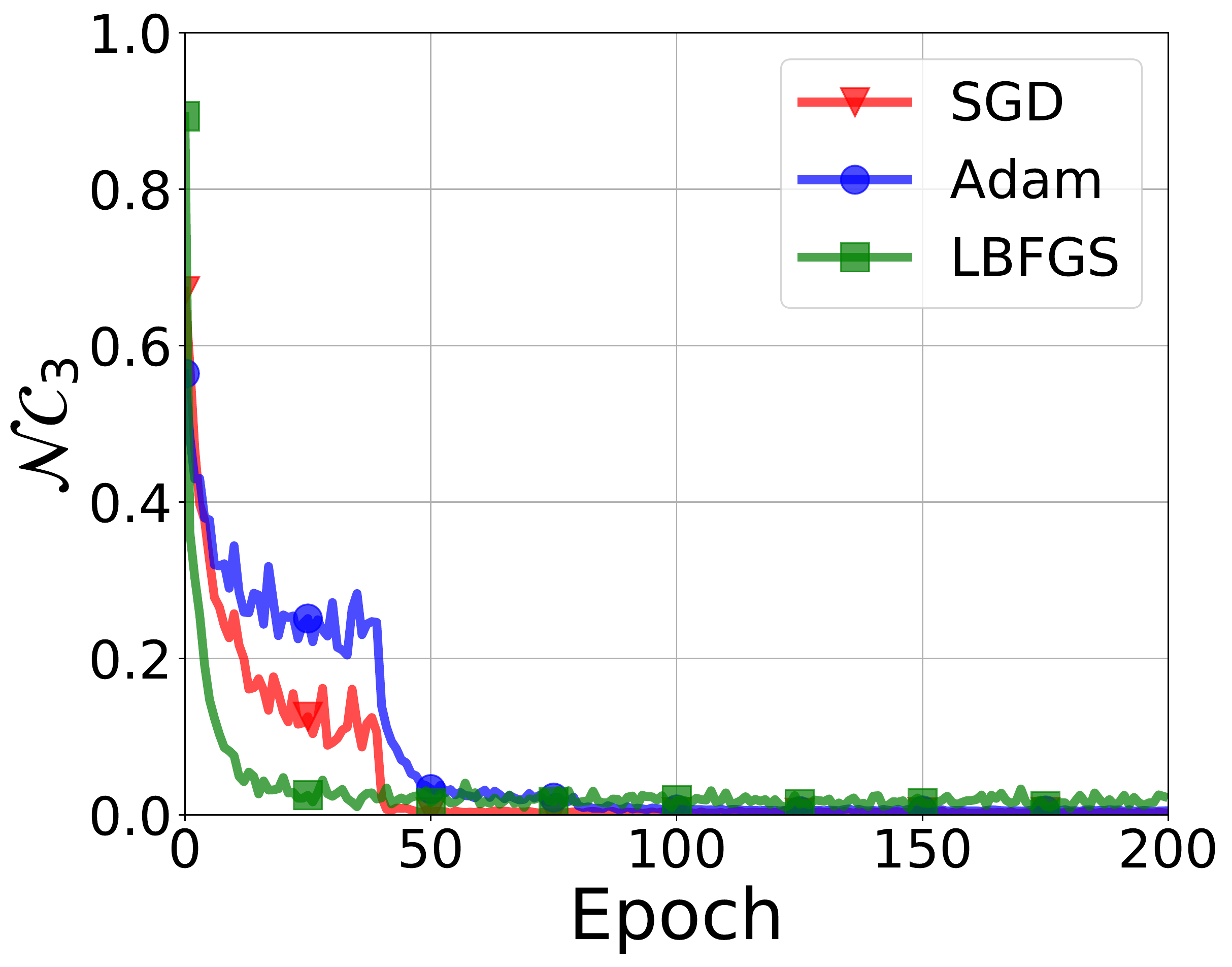}} \\
    \subfloat[Training Acc.]{\includegraphics[width=0.2\textwidth]{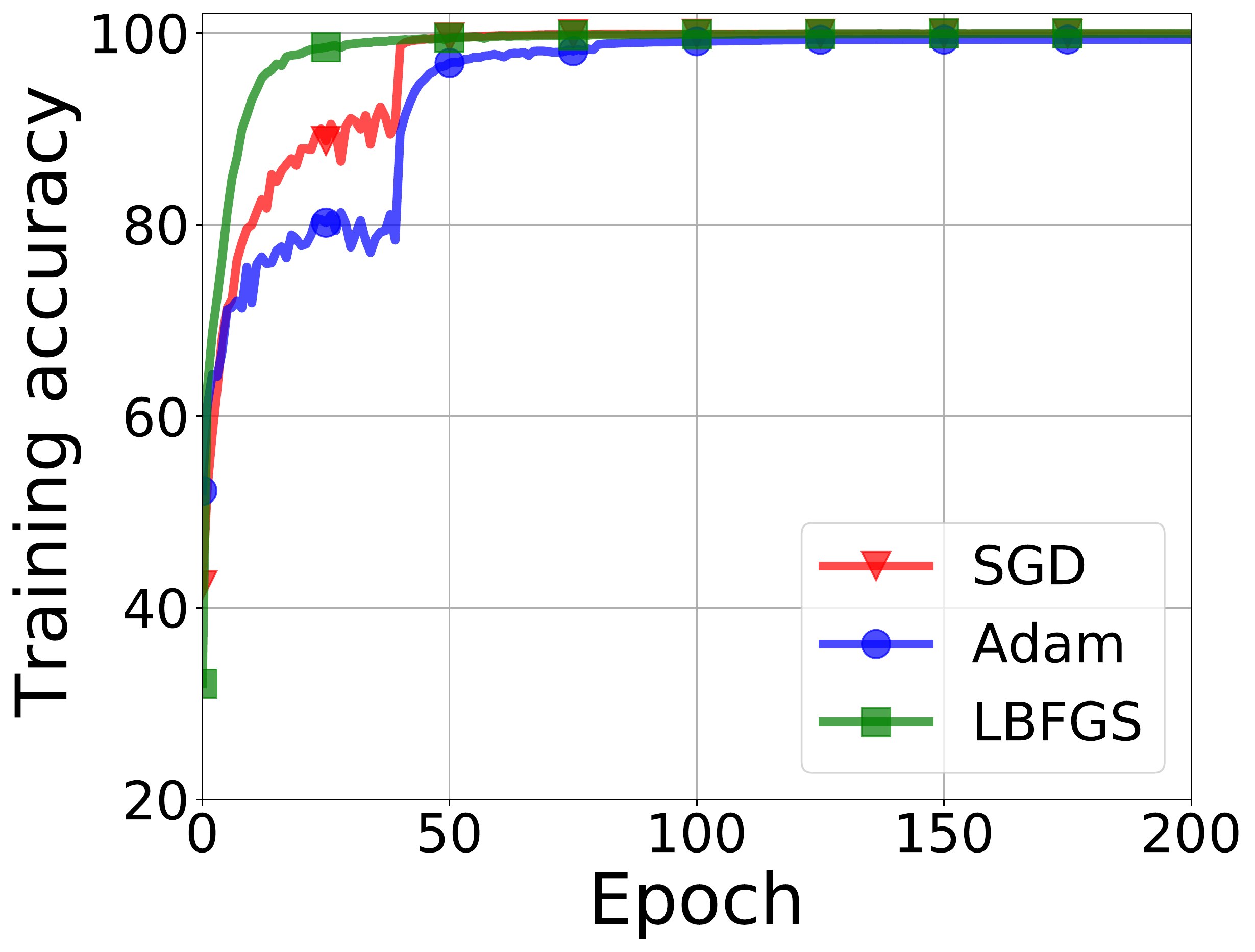}}\
    \hspace{0.08in}
    \subfloat[Test Acc.]{\includegraphics[width=0.2\textwidth]{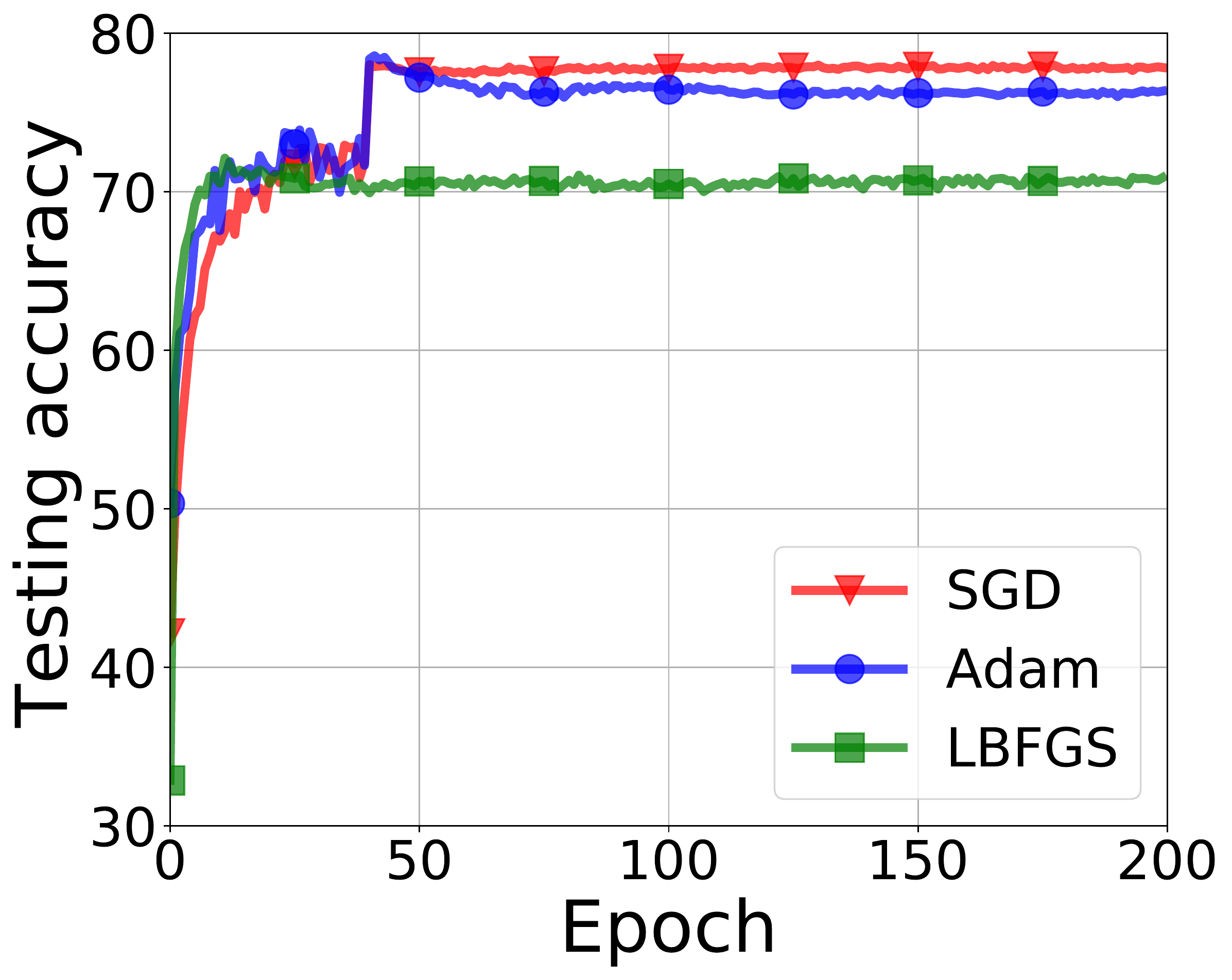}} \
    \hspace{0.08in}
    \subfloat[$\mc {P}_{CM}$]{\includegraphics[width=0.21\textwidth]{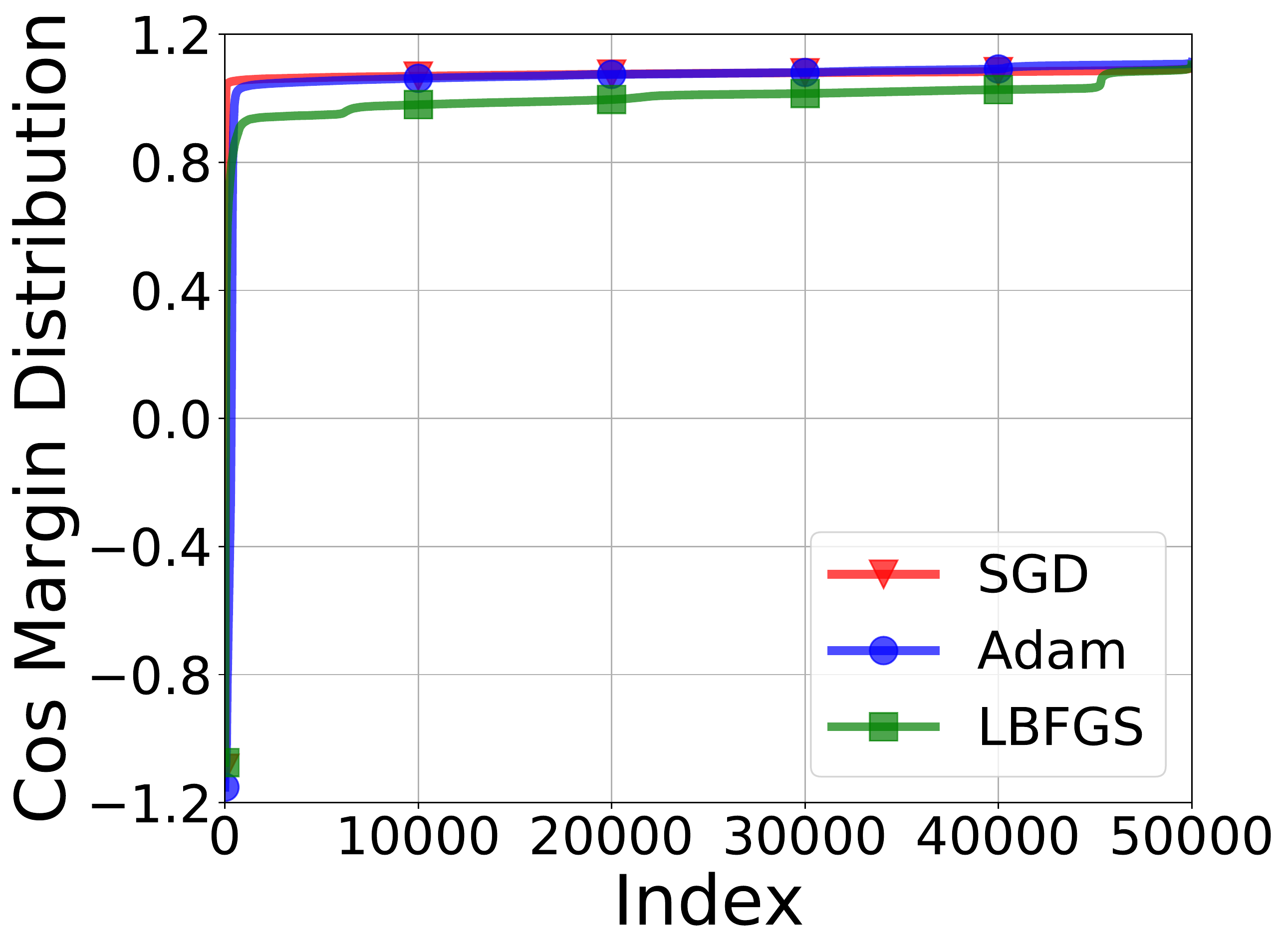}} 
    \caption{\textbf{Illustration of \NC\;, training and test accuracy and cosine margin across different training algorithms} with ResNet18 on CIFAR10. The networks are trained without data augmentation.
    % On the top row, the plots show the three metrics, $\mc {NC}_1, \mc {NC}_2$, and $\mc {NC}_3$, respectively. On the bottom row, the plots show training accuracy, test accuracy, and $\mathcal{P}_{CM}$, respectively.
    }
    \label{fig:cifar10_optim}
\end{figure*}

In this section, we conduct experiments to validate our findings from \Cref{sec:main} on practical networks and standard datasets. % More specifically, %in \cref{subsec:new_metrics}, 
We first introduce new metrics to better evaluate how well the \NC\ properties are satisfied in practical neural networks, in addition to the ones used in \cite{papyan2020prevalence,zhu2021geometric}.  %In \Cref{subsec:exp-NC}, 
Second, we verify our theoretical results in \Cref{subsec:global_optim} by showing that the \NC\ phenomena are algorithmic independent. 
%In \Cref{subsec:rescaled_exp}, 
Third, by a similar experiment as in \cite{zhu2021geometric}, we show that we could fix the last layer weights as a Simplex ETF while achieving comparable generalization performances as explicitly training the classifier. Finally, we examine our findings in \Cref{subsec:landscape-rescaling} that the rescaling factors in the rescaled MSE loss is beneficial for forming benign optimization landscapes. For the details of the experimental setup, we refer readers to the Appendix \ref{app:exp_setup}.

\paragraph{New metrics for evaluating \NC.} %\label{subsec:new_metrics}
To evaluate the \NC\;properties of well-trained neural networks, we adopt the same \NC$_1$, \NC$_2$ and \NC$_3$ metrics as \cite{papyan2020prevalence,zhu2021geometric}, which measure the within-class variability of $\mb H$, the convergence of $\mb W$ to a simplex ETF, and the self-duality between $\mb H$ and $\mb W$;
 see \Cref{app:exp_setup} for the details.\footnote{We also refer the reader to \cite{zhu2021geometric} for the exact definitions of these quantities. Note that for the case $d<K$, the definition of \NC2\ and \NC3\ will be slightly different from those in \cite{zhu2021geometric} based on our theoretical results in \Cref{subsec:global_optim}.} 
To better measure \NC, this paper also introduces the following two metrics that measure the diversities and margins of the learned features: 
\begin{itemize}[leftmargin=*,topsep=0.25em,itemsep=0.1em]
    \item \textbf{Numerical rank.} The \NC$_1$ metric measures the variability collapse through the between-class and within-class covariance matrices, which does not directly reveal the dimensionality of the features spanned for each class. Ideally, when \NC\;happens, for each class the feature dimension should collapse to one. To measure the dimensionality, we introduce a new metric that we call it \emph{numerical rank}, denoted by $\wt\rank(\mb H):= \frac{1}{K} \sum_{k=1}^K \frac{\|\mb H_k\|_{*}^2}{\|\mb H_k\|_{F}^2}$. Here, $\norm{\cdot}{*}$ represents the nuclear norm \cite{recht2010guaranteed} (i.e., the sum of singular values), while the Frobenius norm $\norm{\cdot}{F}$ in the denominator serves as a normalization factor. The metric is evaluated by averaging over all the classes. Our metric is inspired by the numeral sparsity (defined as $\|\va\|_1^2/\|\va\|_2^2$ for $\va\in\R^n$) that serves as a stable measure for sparsity of vectors \cite{lopes2013estimating}. For our numerical rank, we expect that the smaller $\wt\rank(\mb H)$ is, the more collapsed the features are to their class means.
    
  %  Besides \NC1, we use the rank of within-class features as an additional measure of the variability collapse. In particular, inspired by the numeral sparsity (defined as $\|\va\|_1^2/\|\va\|_2^2$ for $\va\in\R^n$) that serves as a stable measure for sparsity  \cite{lopes2013estimating}, we use $\tfrac{\norm{\mb H_k}{*}^2}{\norm{\mb H_k}{F}^2}$ to measure the numerical rank of the $k$-th class features $\mH_k \in \mathbb{R}^{n \times d}$, where $\norm{\mb H_k}{*}$ represents the nuclear norm, i.e., the sum of singular values. Then, we denote by $\wt\rank = \frac{1}{K} \sum_{k=1}^K \frac{\|\mb H_k\|_{*}^2}{\|\mb H_k\|_{F}^2}$ the average numerical rank among all classes. 
    \item \textbf{Cosine margin.} All current metrics measure \NC\; from a panoramic view, and do not quantify the behavior of individual features. 
    We introduce a metric based on the consine margin of individual features.
    % To evaluate the behavior of each feature individually, we introduce another metric based on the consine margin.
  %  Aside from the entire behavior of the features, we will also measure the behavior of each feature. 
    From the explanation in \Cref{subsec:landscape-rescaling}, neural network determines the class member by the direction of features rather than its length. Thus, we define the cosine margin for each sample as $CM_{k,i}=\cos{\theta_{k,i;k}}-\max\limits_{j\neq k}\cos{\theta_{k,i;j}}$, where $\cos{\theta_{k,i;j}}=\frac{\innerprod{\mb w^j-\mb w_G}{\mb h_{k,i}-\mb h_G}}{\|\mb w^j- \mb w_G\|_2 \,\|\mb h_{k,i}-\mb h_G\|_2}$ represents the cosine of the angle between the feature $\vh_{k,i}$ and the $j$-th classifier $\vw^j$, $\vh_G$ denotes the global mean of all the features, and $\vw_G$ denotes the mean of all the rows in $\mW$. Recall that 
    $\vh_{k,i} \in \mathbb{R}^{d}$ denotes the feature of $i$-th sample in the $k$-th class and $\vw^j \in \mathbb{R}^{d}$ denotes the $j$-th row of the linear classifier weight $\mW \in \mathbb{R}^{K \times d}$. We sort the cosine margins over the training dataset in the ascending order and denote the resulted distribution as $\mathcal{P}_{CM}$. We note that a similar metric has been explored by the work \cite{banburski2021distribution} as an alternative for the probability margin.\footnote{The probability margin cannot be adopted here because probability is not well-defined given that softmax is not used in the MSE loss.}
    %We are not using the probability margin here as the probability in the MSE loss case is not well-defined. To be more specific, the output can have negative values and may not sum to $1$.}.
    
    % When $d\ge K-1$ and $W$ converges to Simplex ETF, $CM_{k, i}$ achieves the maxima $\frac{K}{K-1}$ if $h_{k, i}$ aligns with $w_k$, and $CM_{k, i}$ achieves the minima $-\frac{K}{K-1}$ if $h_{k, i}$ aligns with $w_j$ and $j\neq k$.
    % =\frac{1}{nK}\sum\limits_{j=1}^K\sum\limits_{i=1}^n h_{k,i}
\end{itemize}
%\qq{refer to our paper for NC 1-3 (discuss in footnote for $d<K$), introduce the new metrics, and discuss the motivations and intuitions behind those metrics.}
%For evaluating the \NC\ performances of the experiments, we also adopt the metrics \NC1, \NC2 and \NC3\ from \cite{zhu2021geometric} where these metrics are corresponding to the last layer structures we discussed in \Cref{sec:intro}.\footnote{Note that for the case $d<k$, the definition of \NC2\ and \NC3\ will be slightly different as in \cite{zhu2021geometric}.} Besides, we consider other metrics that focus on the diversities and margins of the learned features: 
%\qq{study $d>K$ and $d<K$}

%\qq{study rescaled vs unrescaled ls loss}

%\qq{put experimental results on imagenet mini, more collapse}

%\label{subsec:exp-NC}
\paragraph{The prevalence of \NC\;across different optimization algorithms.}
% \qq{needs to make clear that this is on vanilla MSE or rescaled?}
The benign landscape for optimization of neural networks with vanilla MSE loss suggests the existence of \NC\ regardless of specific choice of the optimizer. 
We validate this result by training ResNet18 on CIFAR10 with vanilla MSE loss, using three different optimization algorithms: SGD, Adam and L-BFGS. 
As shown in \Cref{fig:cifar10_optim}, $\mc {NC}_1, \mc {NC}_2$, and $\mc {NC}_3$ converge to zero as training progresses, regardless of algorithm used. 
% This implies that \NC\; happens regardless of the choice of training methods. 
Similar to the observation for the CE loss in \cite{zhu2021geometric}, although all algorithms lead to \NC~solutions, networks trained with different algorithms have notably different generalization performances.\footnote{L-BFGS with strong Wolfe line-search strategy may result in quite small stepsize at the terminal phase of training. We think that L-BFGS with proper diminishing stepsize can improve the generalization ability.}
% Similar to \cite{zhu2021geometric}, the curves of training and test accuracy for all three different optimization show that neural networks trained with different optimizations have notably different generalization performance, even though all of them exhibit \NC.\footnote{LBFGS with strong Wolfe line-search strategy may result in quite small stepsize at the terminal phase of training. We think that LBFGS with proper diminishing stepsize can improve the generalization performance.}
We find the cosine distribution $\mathcal{P}_{CM}$ consistently aligns with the test accuracy, the more and higher. This may due to the fact that different training methods have different converge rate during the terminal phase of training, and it further lead to different distribution of features.  

\begin{figure*}[t]
    \centering
    \subfloat[$\mc {NC}_1$ (MSE)]{\includegraphics[width=0.18\textwidth]{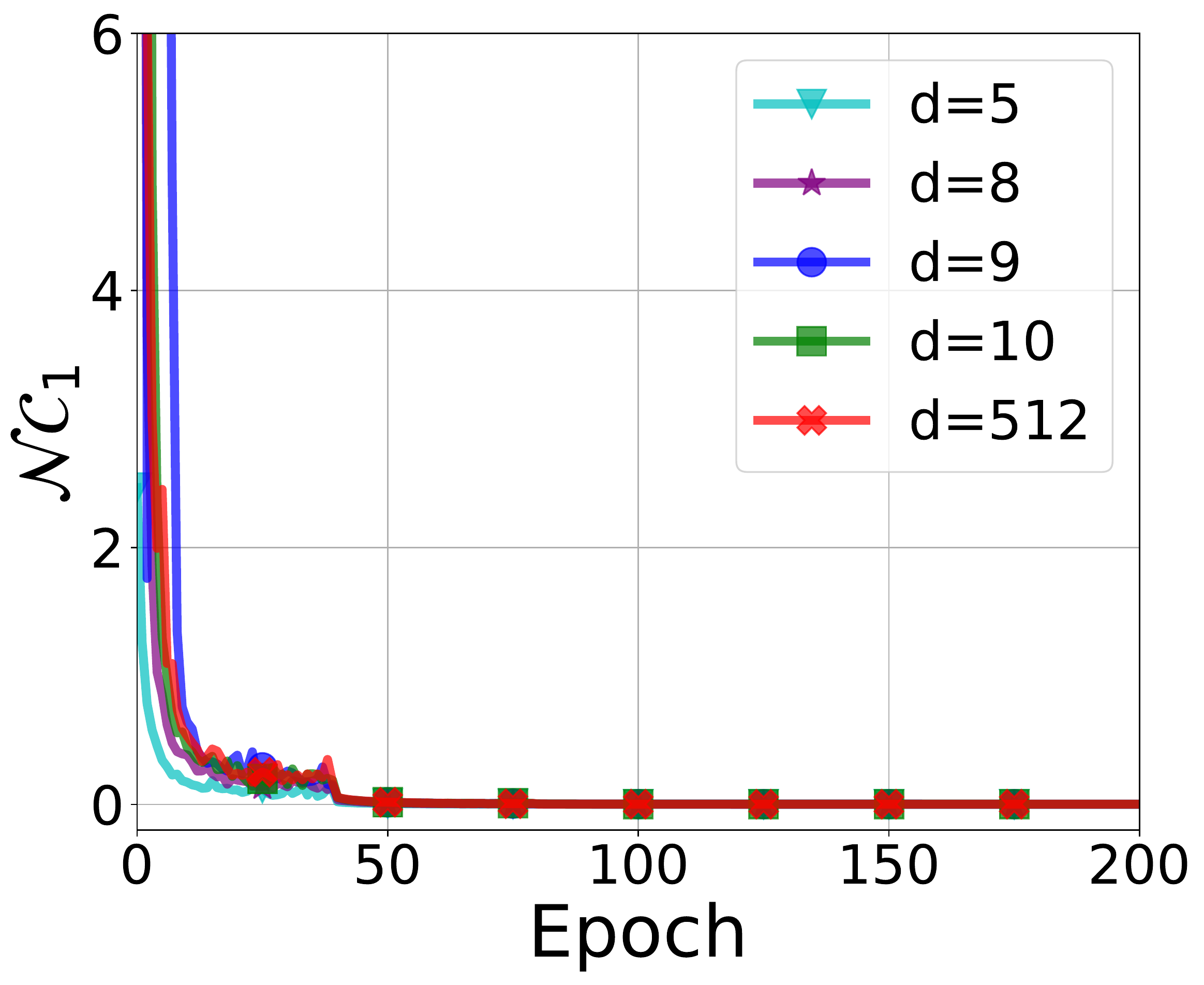}} \
    \subfloat[$\mc {P}_{CM}$ (MSE)]{\includegraphics[width=0.20\textwidth]{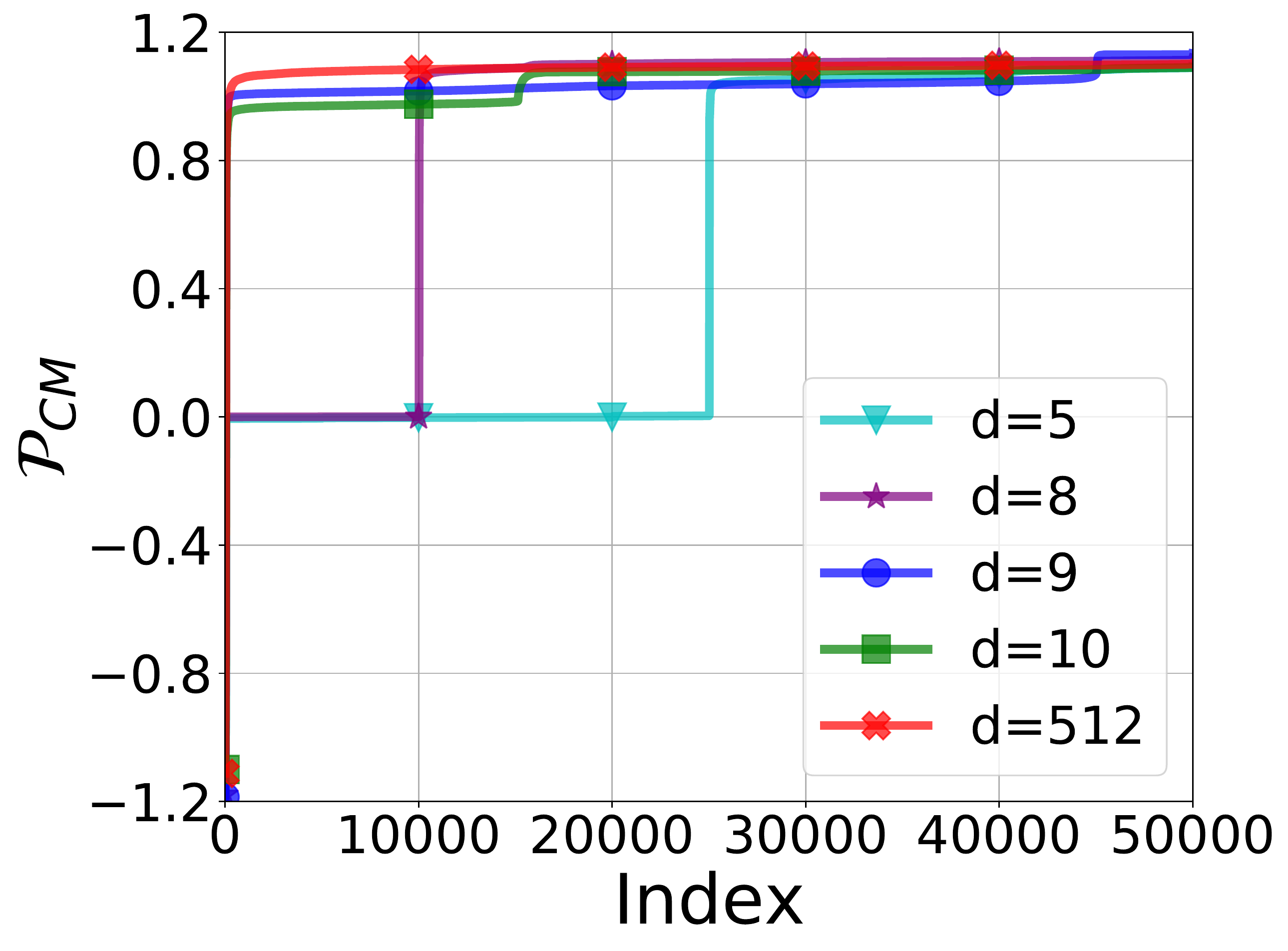}}
    \
    \subfloat[Train Acc.(MSE)]{\includegraphics[width=0.19\textwidth]{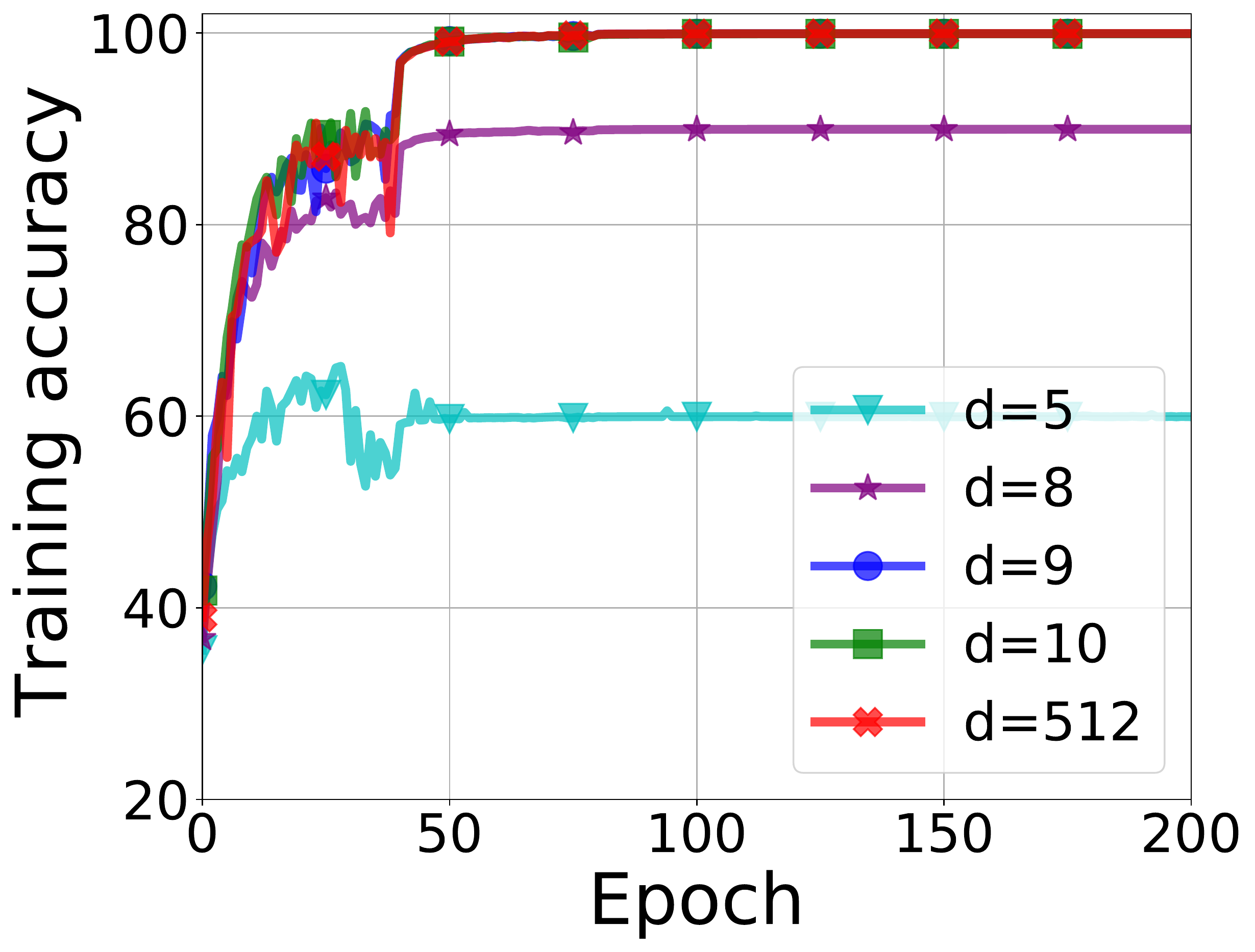}}\
    \subfloat[Test Acc. (MSE)]{\includegraphics[width=0.19\textwidth]{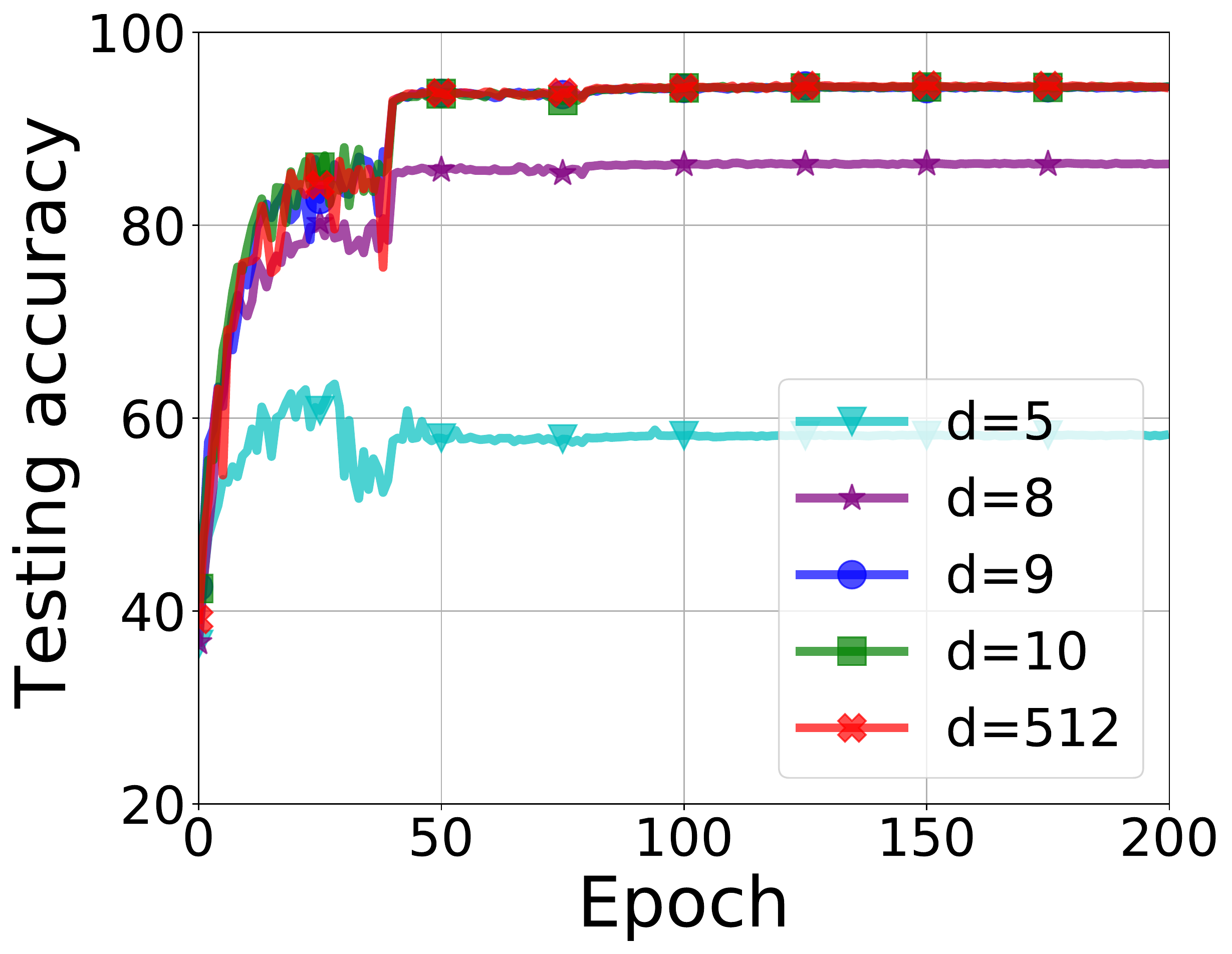}}
    \
    \subfloat[Test Acc. (CE)]{\includegraphics[width=0.19\textwidth]{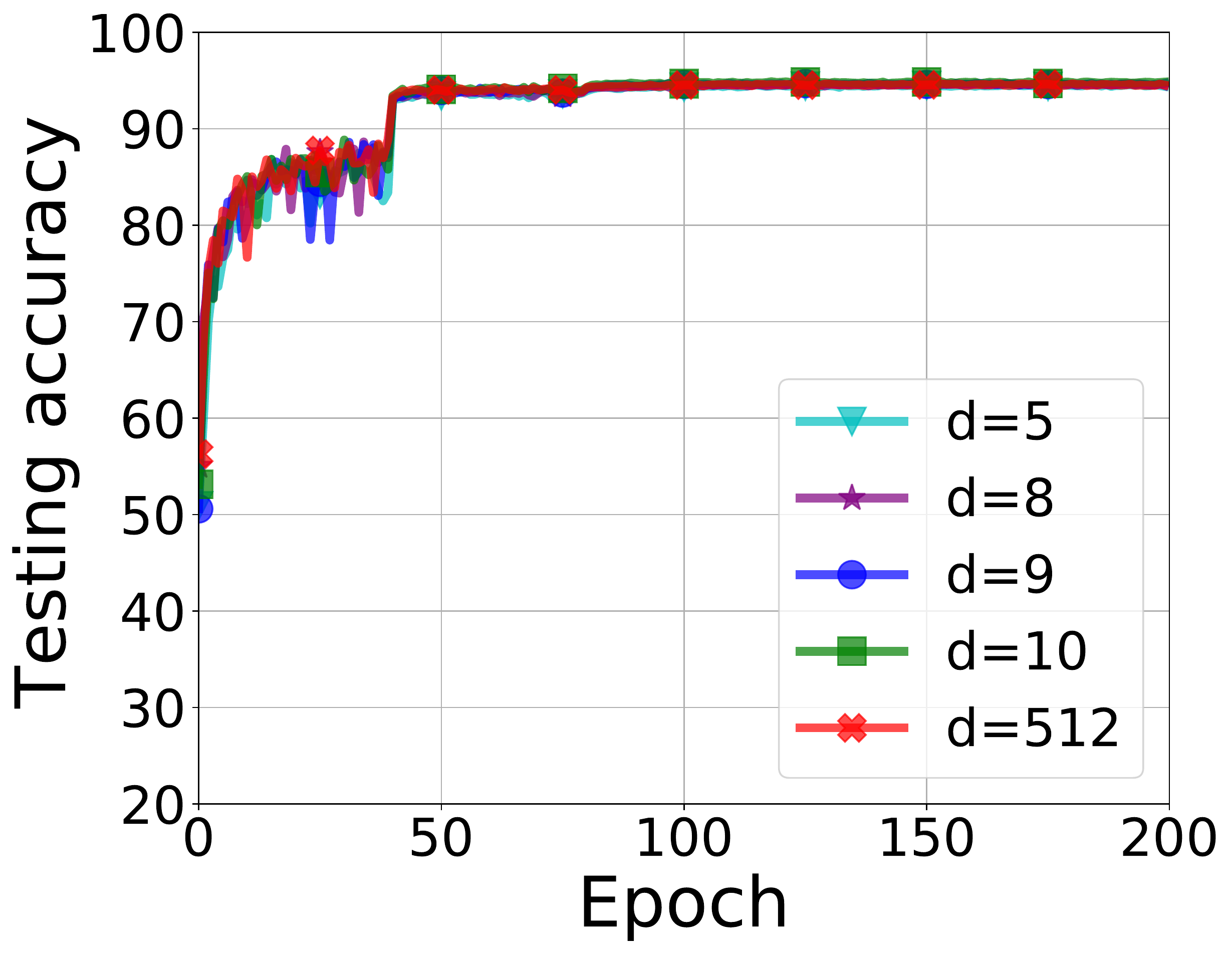}}
    \caption{\textbf{Comparison of the performances on networks with different feature dimensions $d$ for MSE and CE losses.} We compare within-class variation collapse $\mc {NC}_1$, cosine margin distribution $\mathcal{P}_{CM}$, training accuracy, and test accuracy  on learned classifier with different feature dimension $d$ on CIFAR10 using ResNet18 with data augmentation. The network is trained by the SGD optimizer.
    }
    \label{fig:cifar10_dim}
\end{figure*}

\paragraph{Improving network efficiency via fixing classifiers as simplex ETFs.}%\label{subsec:exp-fix-classifier} 
In \Cref{thm:global-minima}, when $d \geq K-1 $  and the weight decay terms are properly chosen, we showed that the optimal classifier for the vanilla MSE loss is a simplex ETF. 
This implies that we can \emph{(i)} fix the last-layer classifier as a simplex ETF, and \emph{(ii)} reduce the feature dimension $d=K$. 
By doing so, we substantially reduce the number of trainable parameters without sacrificing the generalization performance as shown in \Cref{fig:cifar10_etf}. 
% As shown \Cref{subsec:global_optim}, by doing these for the vanilla MSE loss, we can also substantially reduce the computation and memory cost without the sacrifice of performance (see \Cref{fig:cifar10_etf}), similar to that of \cite{zhu2021geometric}. 
% \qq{again, which loss are we using? vanill MSE or rescaled?}
%We demonstrate that the weight of the final classifier under standard MSE loss has the same geometry as that under CE loss \cite{zhu2021geometric} when $d \geq K$ and the weight decay terms are properly chosen. Thus, we can also substantially reduce the computation and memory cost without he sacrifice of performance, by \emph{(i)} fixing the last-layer classifier as a Simplex ETF, and \emph{(ii)} reducing the feature dimension $d=K$. In \Cref{fig:cifar10_etf}, we run experiments on the CIFAR10 dataset with ResNet18. Similar to the those of CE loss, the results of standard MSE loss verify our conclusion in \Cref{subsec:global_optim} that we can fix the weights in the last layer as a Simplex ETF and choose $d=K$ with comparable performance. \xiao{ Is this part essentially the same thing with the previous work? Should we just cite the previous work and say that similar experiments also works in the MSE case?}

\paragraph{Choice of the feature dimension $d$.}%\label{sec:exp-dim} 
On the other hand, \Cref{thm:global-minima} shows that the optimal class means $\ol\mH^\star$ form a simplex ETF only when $d\geq K-1$.
% In \Cref{subsec:global_optim} (well below \Cref{thm:global-minima}), we discussed the implications of \Cref{thm:global-minima} for the choice of the feature dimension $d$. Our theory implies that the global solutions are simplex ETFs only when $d\geq K-1$. 
If $d<K-1$,  then the global solution $\ol\mH^\star$ is only the best rank-$d$ approximation of the simplex ETF, where the class-means of the each class neither have equal length nor are maximally distant.
% which will lead to poor generalization performance. 
To demonstrate its effect, we run experiments on the CIFAR10 dataset using vanilla MSE loss and ResNet18, with both $d<K-1$ and $d \geq K-1$. 
As shown in \Cref{fig:cifar10_dim}, even though all cases exhibit \NC, choosing $d\geq K-1$ is crucial for fitting the training data and generalization to test data. 
This is also corroborated by observing $\mathcal{P}_{CM}$, which shows that more training samples lie on the decision boundary (i.e., $CM_{k,i}=0$) as $d$ decreases in the range of $d<K-1$. \revise{As shown in \Cref{fig:cifar10_dim}(e), this is in sharp contrast to CE loss which produces similar performance for different $d$. Note that all the existing work on CE loss \cite{papyan2020prevalence,papyan2020traces,fang2021layer,mixon2020neural,graf2021dissecting,ergen2021revealing,zhu2021geometric} only study the case when $d\ge K$. In the Appendix, we visually compare the features learned by CE and MSE, but we leave the thorough analysis for CE loss as future work.}

%In \Cref{subsec:global_optim}, although we also show the geometry of the final classifier when $d<K-1$, the relationship between $d$ and performance is unclear for us. In \Cref{fig:cifar10_dim}, we run experiments on the CIFAR10 dataset with ResNet18. The gap between $d<K-1$ and $d \geq K-1$ in training accuracy and testing accuracy plots indicates that $d\geq K-1$ is crucial for learning capacity and generalization, even though all cases are collapsed in the $\mc {NC}_1$ plot. In $\mathcal{P}_{CM}$ plot, we find that more training samples locate at the separate boundary($CM_{k,i}=0$) as $d$ decreases for $d<K-1$.

\begin{figure*}[t]
    \centering
    \subfloat[$\mc {NC}_1$]{\includegraphics[width=0.183\textwidth]{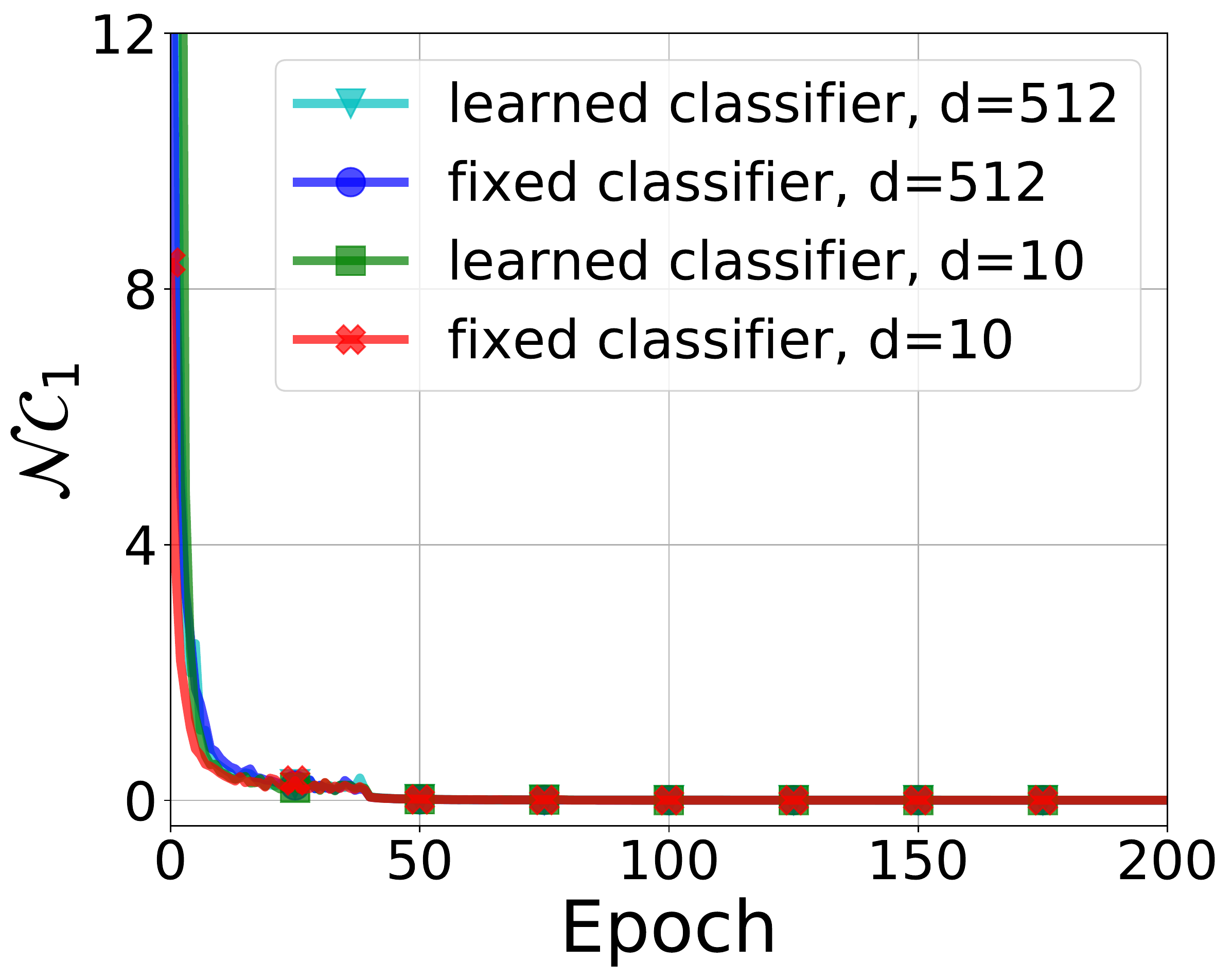}} \
    \subfloat[$\mc {NC}_3$]{\includegraphics[width=0.187\textwidth]{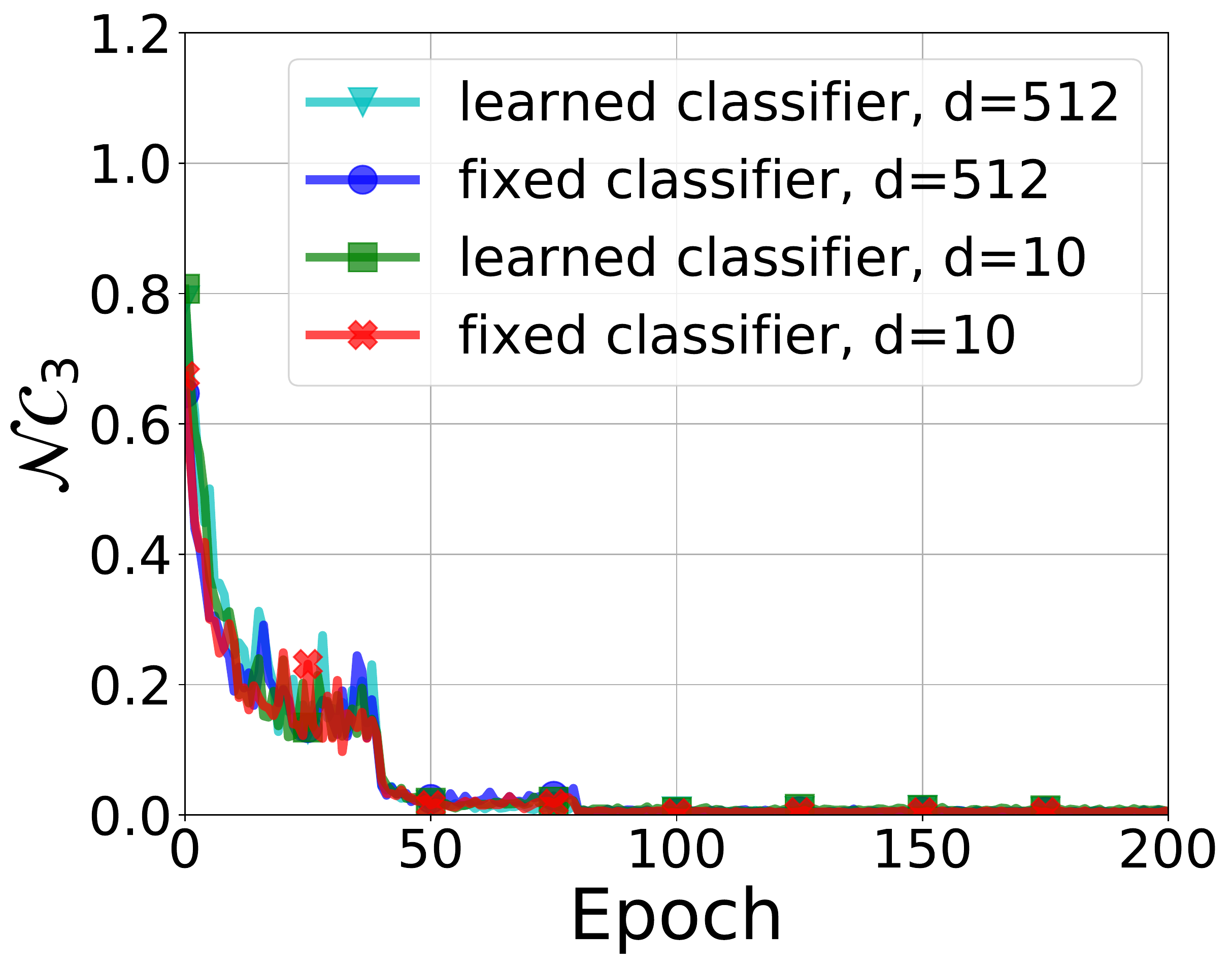}} \
    \subfloat[Training Acc.]{\includegraphics[width=0.187\textwidth]{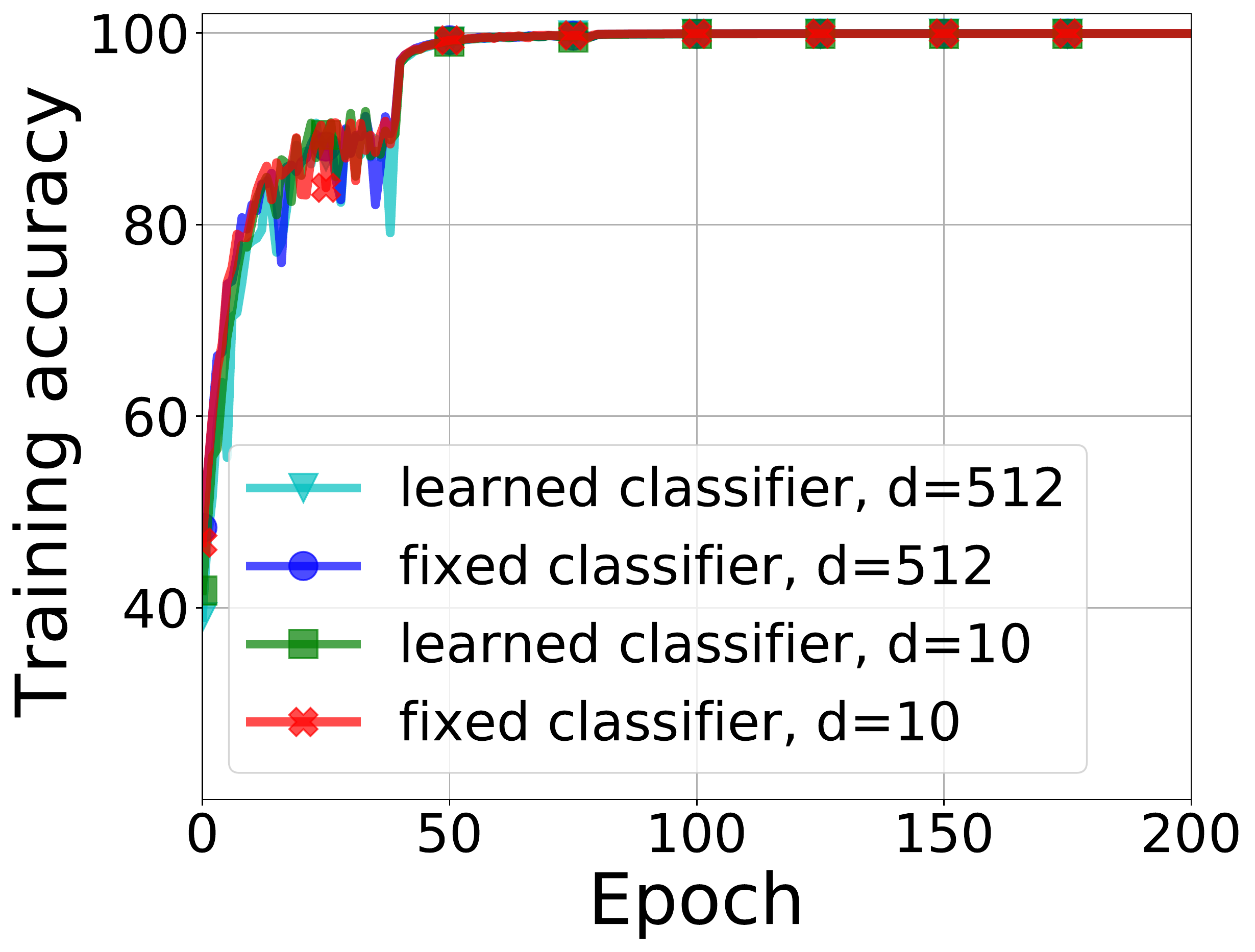}}\
    \subfloat[Test Acc.]{\includegraphics[width=0.187\textwidth]{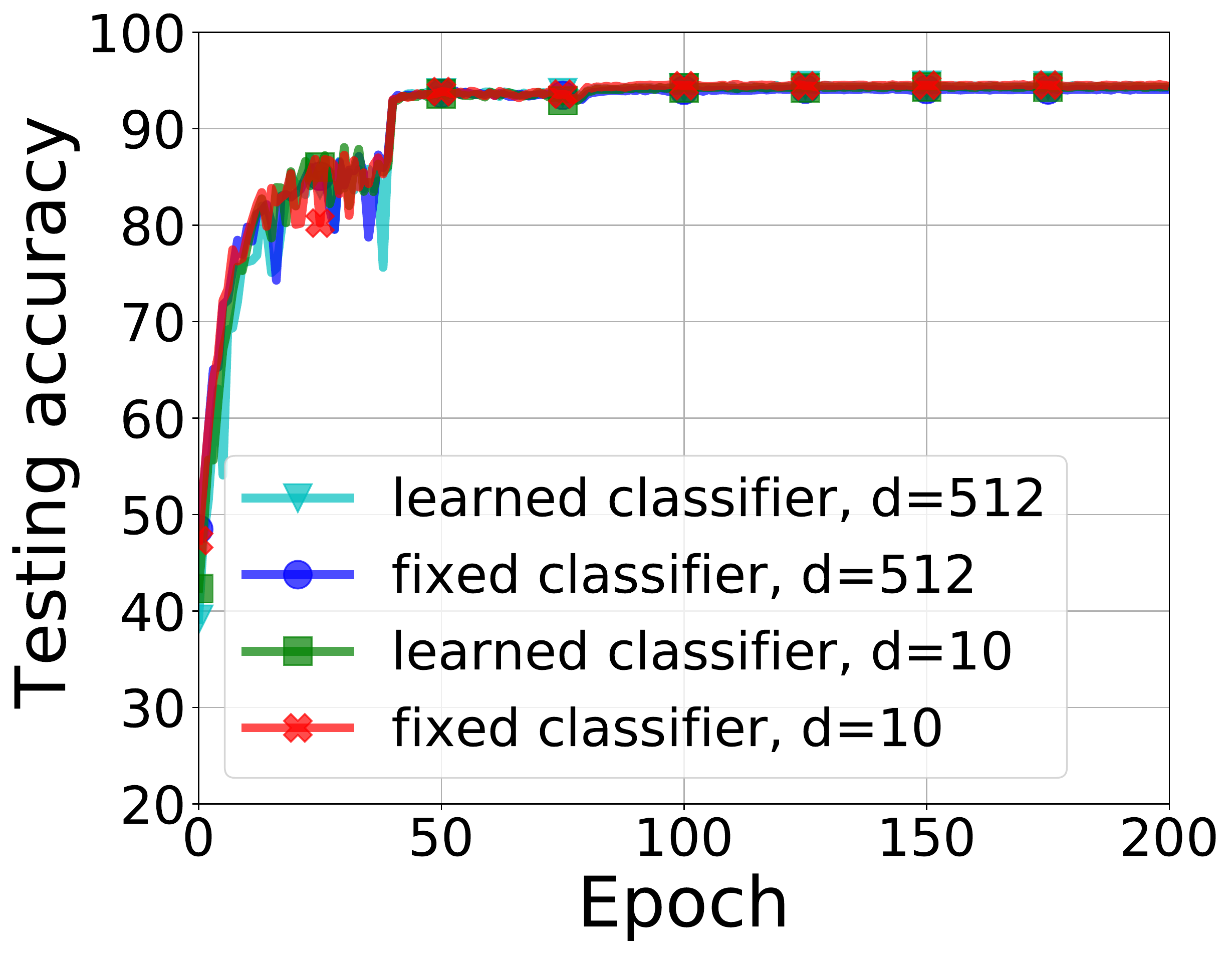}}\
    \subfloat[$\mc {P}_{CM}$]{\includegraphics[width=0.196\textwidth]{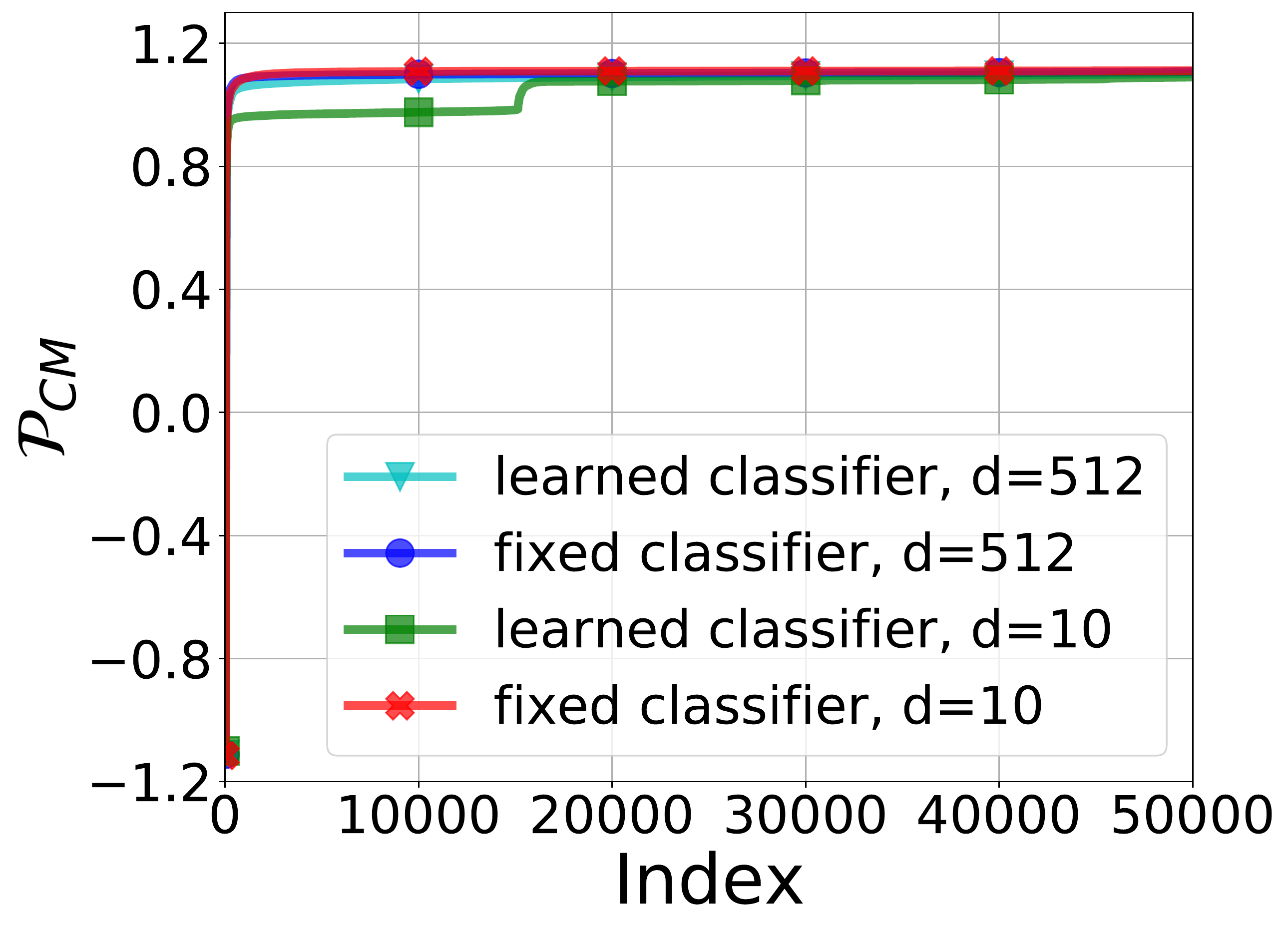}} 
    \caption{\textbf{Comparison of the performances on learned vs. fixed last-layer classifiers.} We compare within-class variation collapse $\mc {NC}_1$, self-duality $\mc {NC}_3$, training accuracy, test accuracy and cosine margin distribution $\mathcal{P}_{CM}$ on fixed and learned classifier on CIFAR10-ResNet18 with data augmentation. The network is trained by SGD optimizer.
    }
    \label{fig:cifar10_etf}
\end{figure*}

\paragraph{Experiments of the rescaled MSE loss.}
%\label{subsec:rescaled_exp}
In \Cref{subsec:landscape-rescaling}, we argued through landscape visualization that rescaling improves the optimization landscape for the MSE loss around the global solutions. Here, we corroborate our findings via experiments, showing that rescaling of MSE indeed leads to better \NC\;and hence better optimization landscapes. In particular, we empirically examine the effect of the two rescaling parameters $(\alpha,M)$ on the \NC\;phenomenon and the generalization performance. In \Cref{fig:mini_image_rescaling}, we run experiments on the miniImageNet \cite{vinyals2016matching} dataset with ResNet18 \cite{he2016deep}. We notice that when one scaling factor is fixed, the other scaling parameter has a positive correlation with the degree of \NC\;as well as the training and test performances. This observation is well-aligned with our analysis in \Cref{subsec:landscape-rescaling}.

\begin{figure*}[t]
    \centering
    \subfloat[$\mc {NC}_1$ ]{\includegraphics[width=0.16\textwidth]{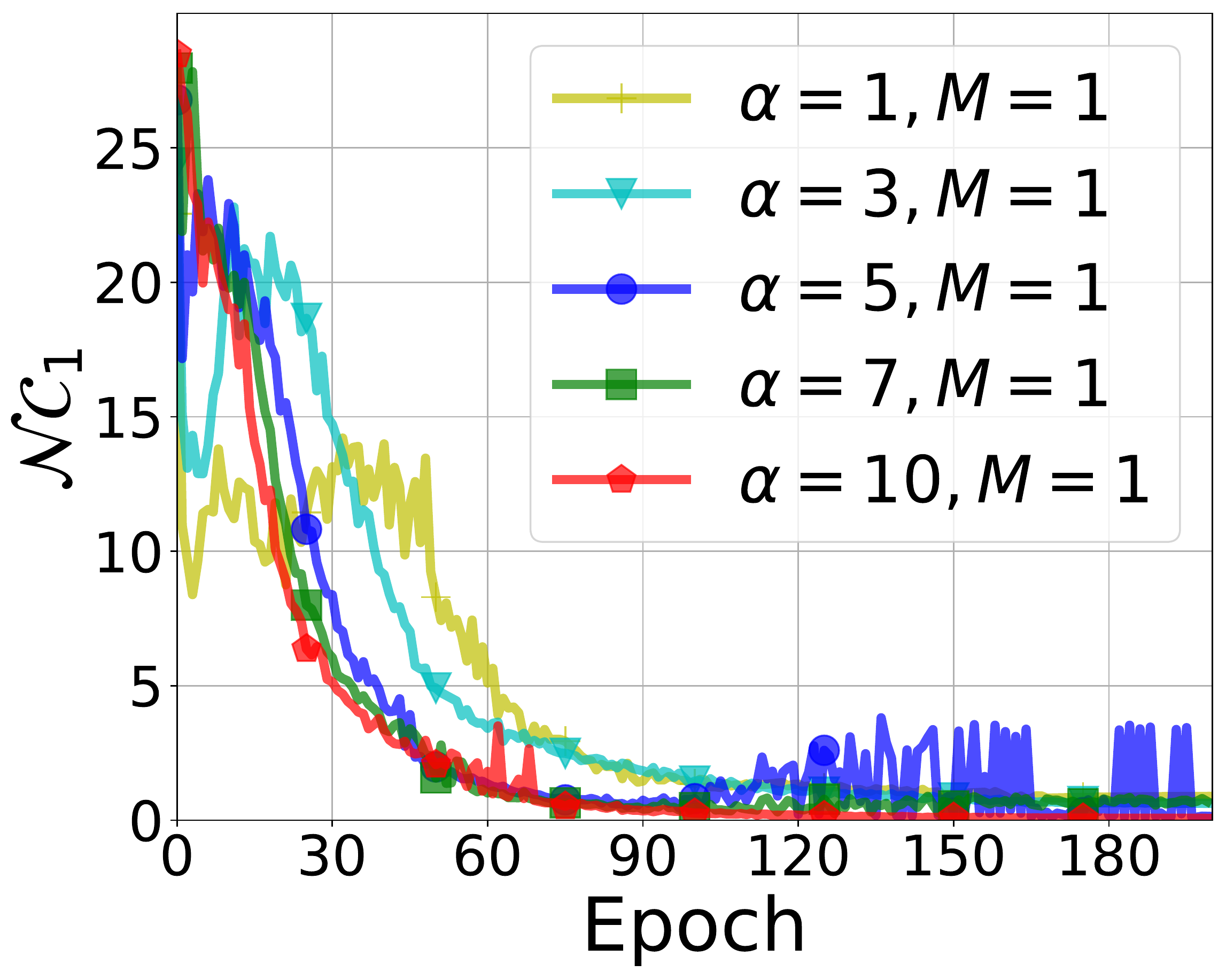}} \
    \subfloat[$\mc {NC}_3$]{\includegraphics[width=0.16\textwidth]{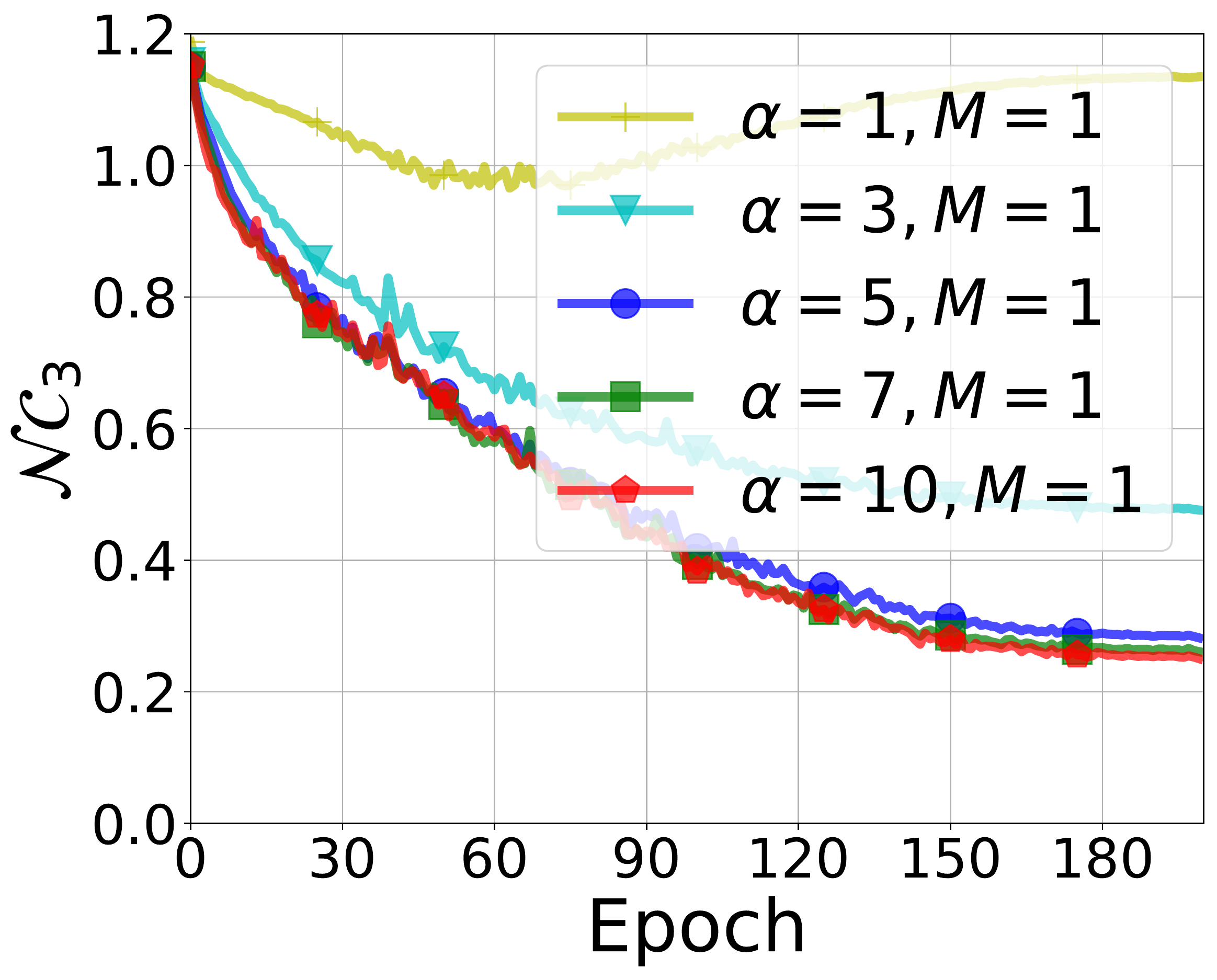}} \
    \subfloat[$\mc{P}_{CM}
    $]{\includegraphics[width=0.166\textwidth]{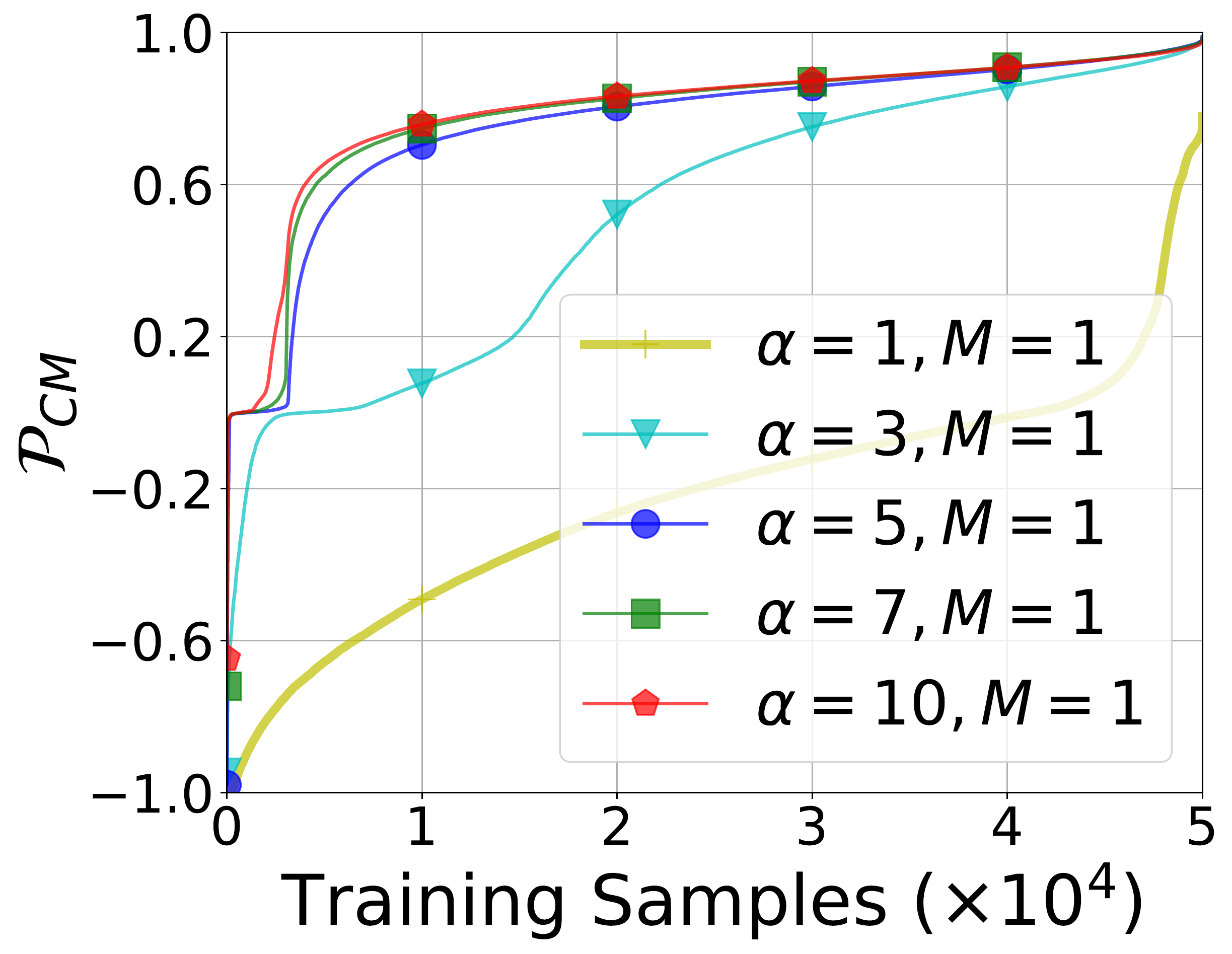}} \
    \subfloat[$\wt\rank$]{\includegraphics[width=0.155\textwidth]{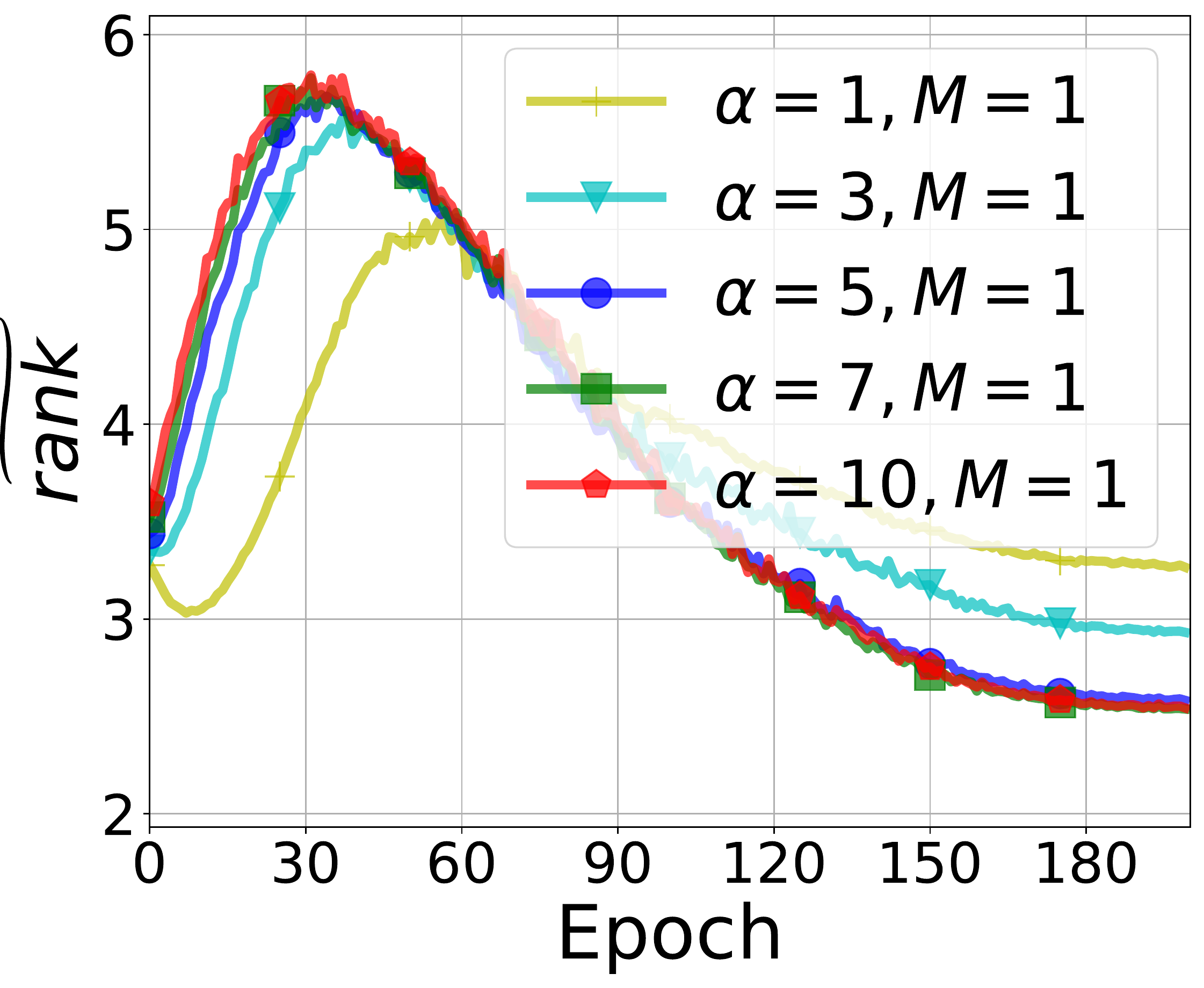}} \
    \subfloat[Train Acc.]{\includegraphics[width=0.162\textwidth]{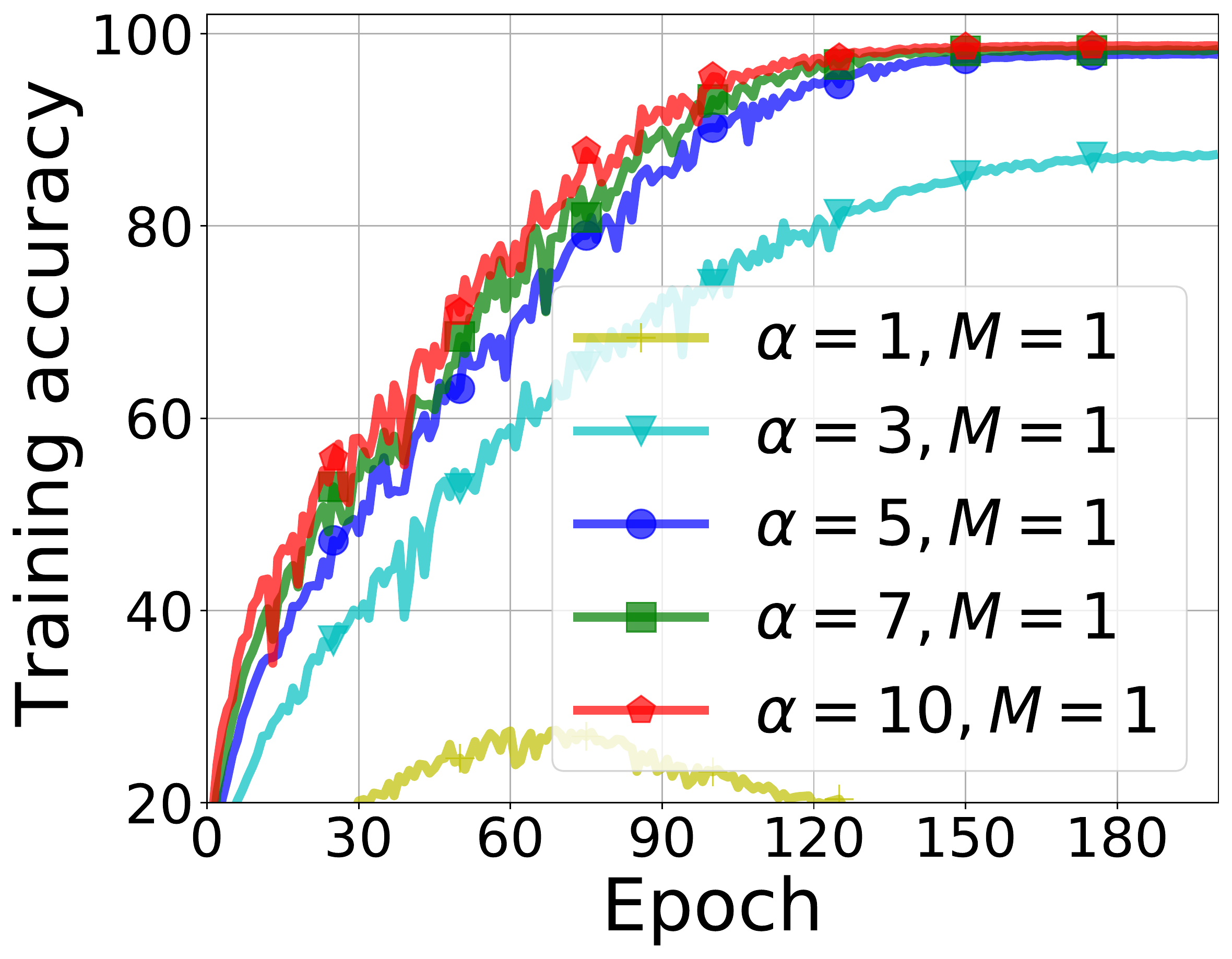}} \
    \subfloat[ Test Acc.]{\includegraphics[width=0.157\textwidth]{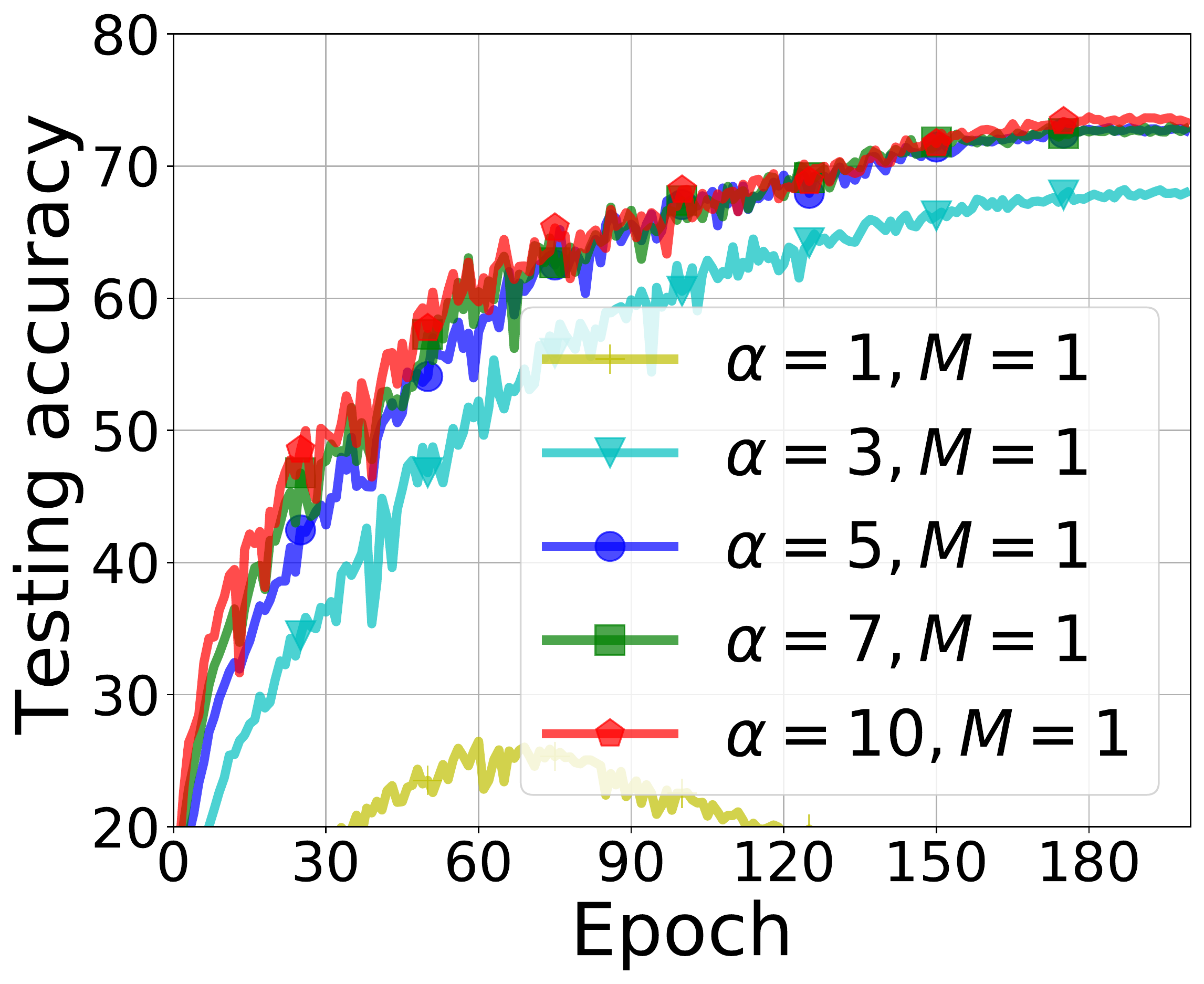}} \\
    
    \subfloat[$\mc {NC}_1$ ]{\includegraphics[width=0.16\textwidth]{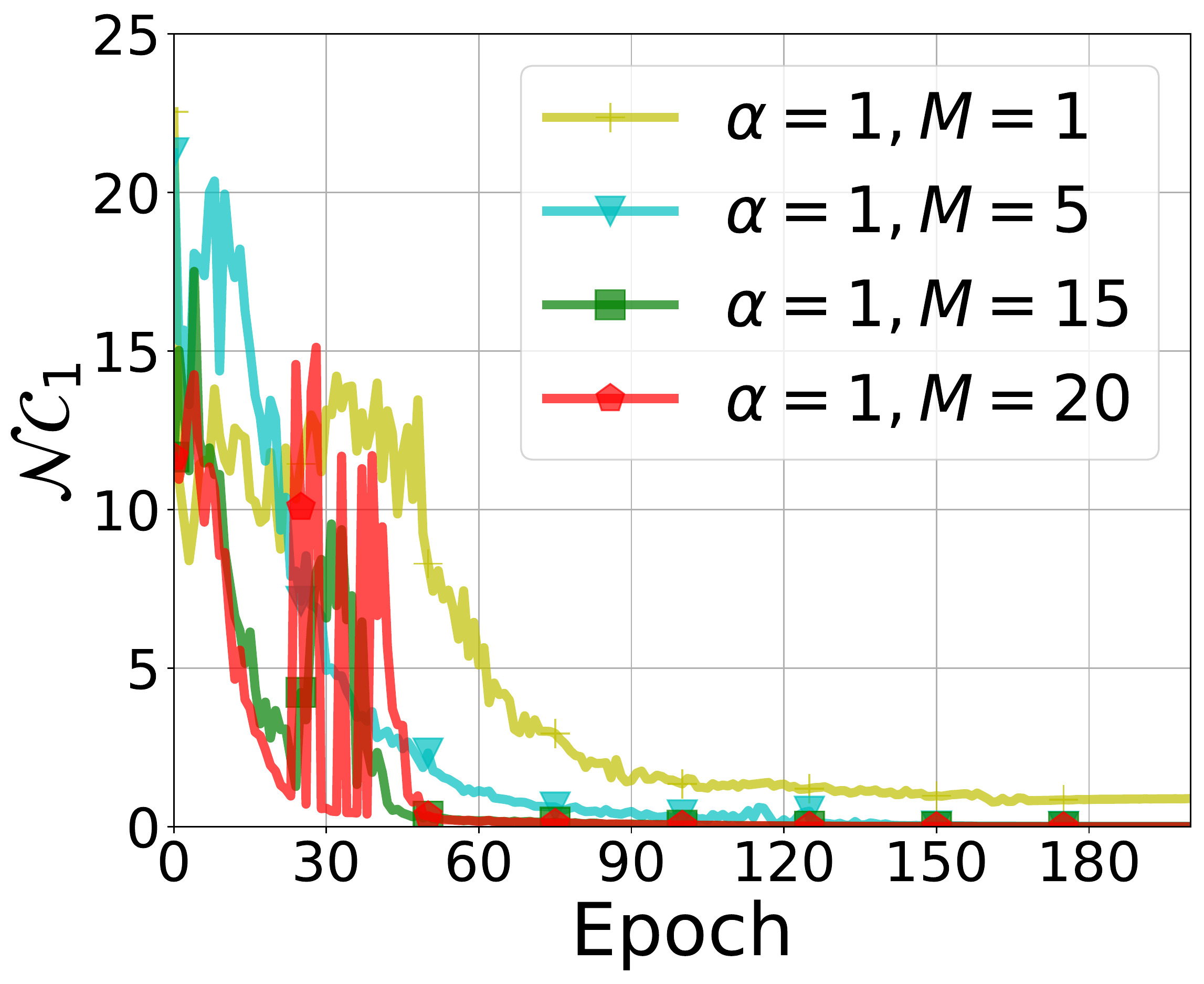}} \
    \subfloat[$\mc {NC}_3$]{\includegraphics[width=0.16\textwidth]{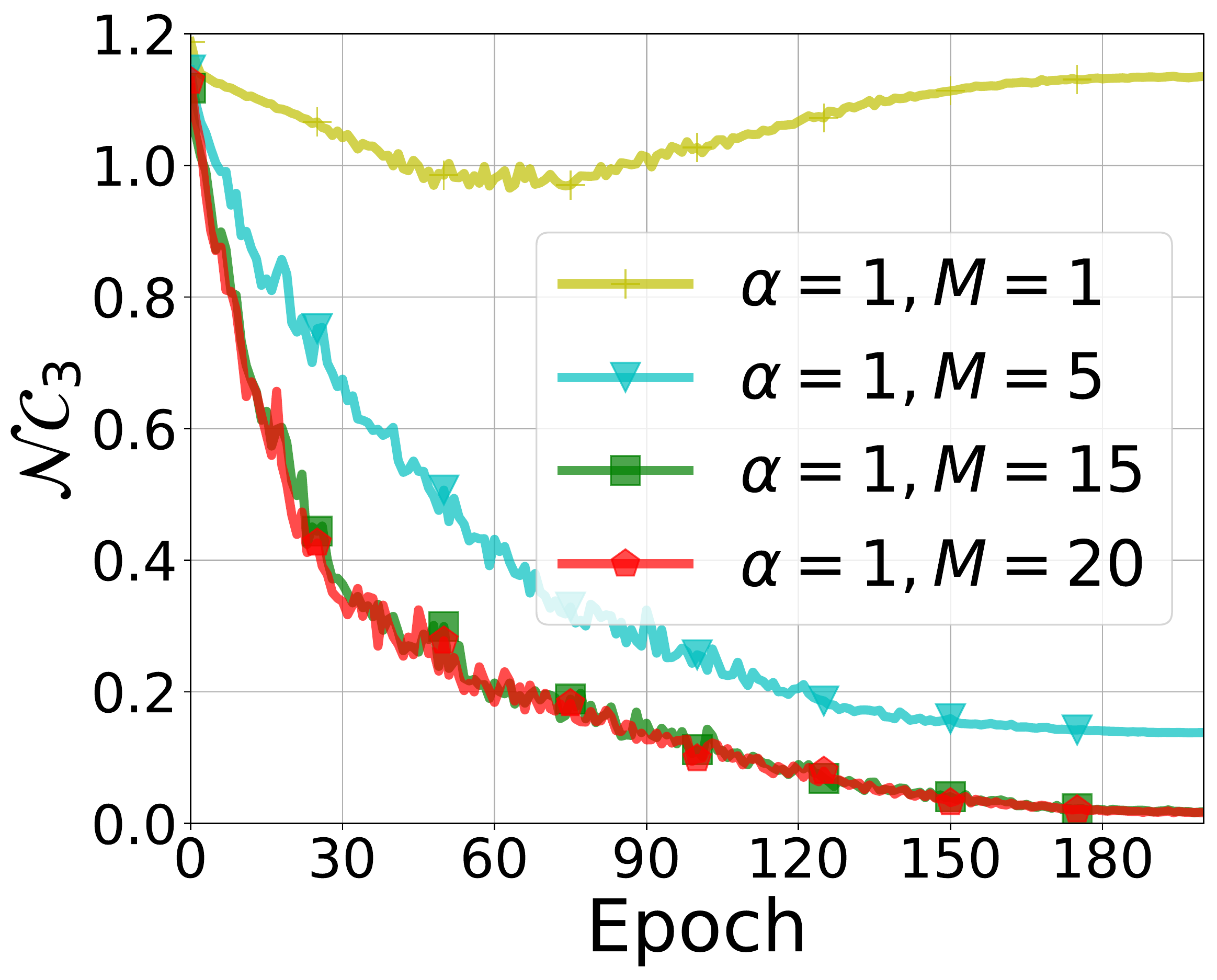}} \
    \subfloat[$\mc{P}_{CM}
    $]{\includegraphics[width=0.166\textwidth]{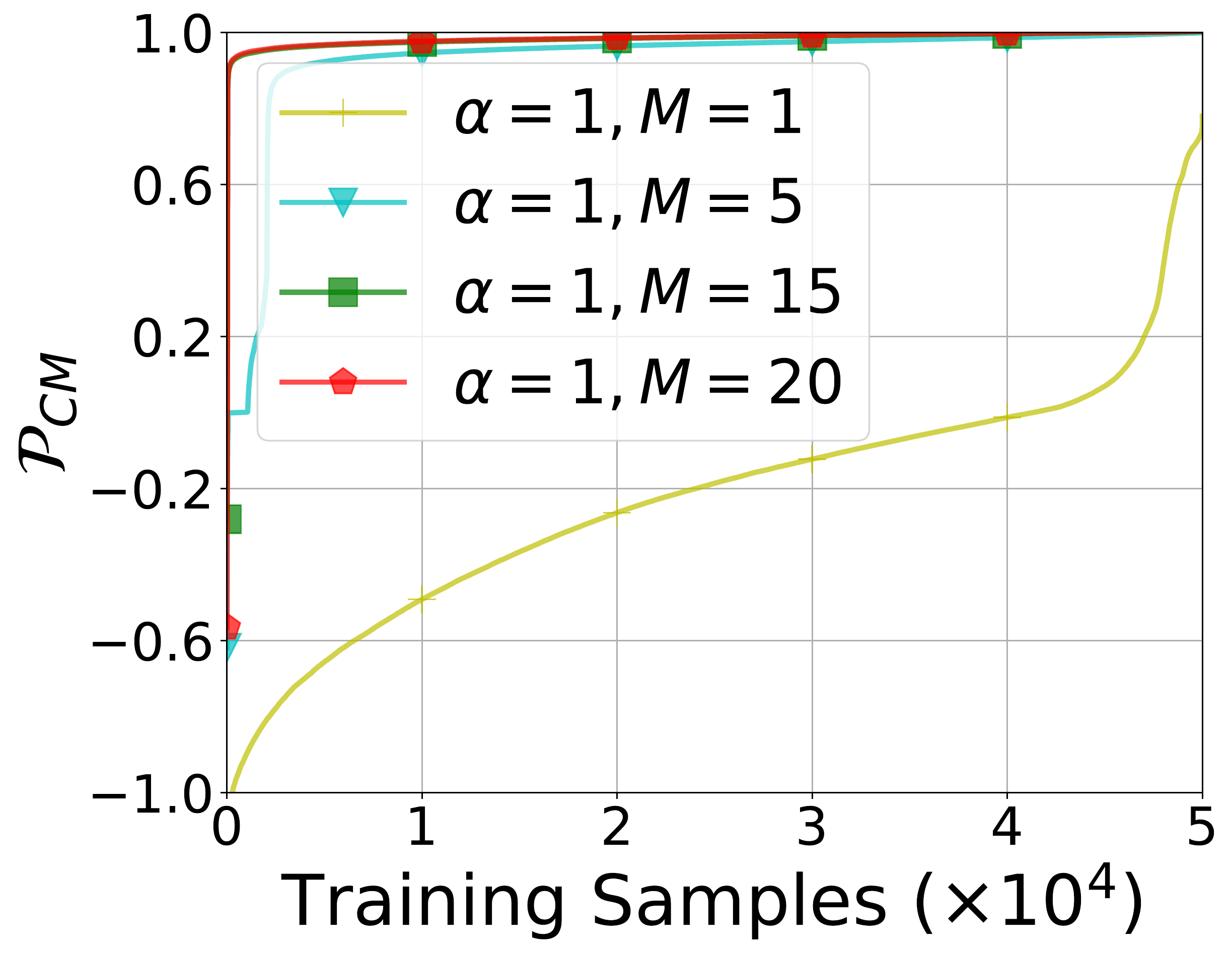}} \
    \subfloat[$\wt\rank$]{\includegraphics[width=0.155\textwidth]{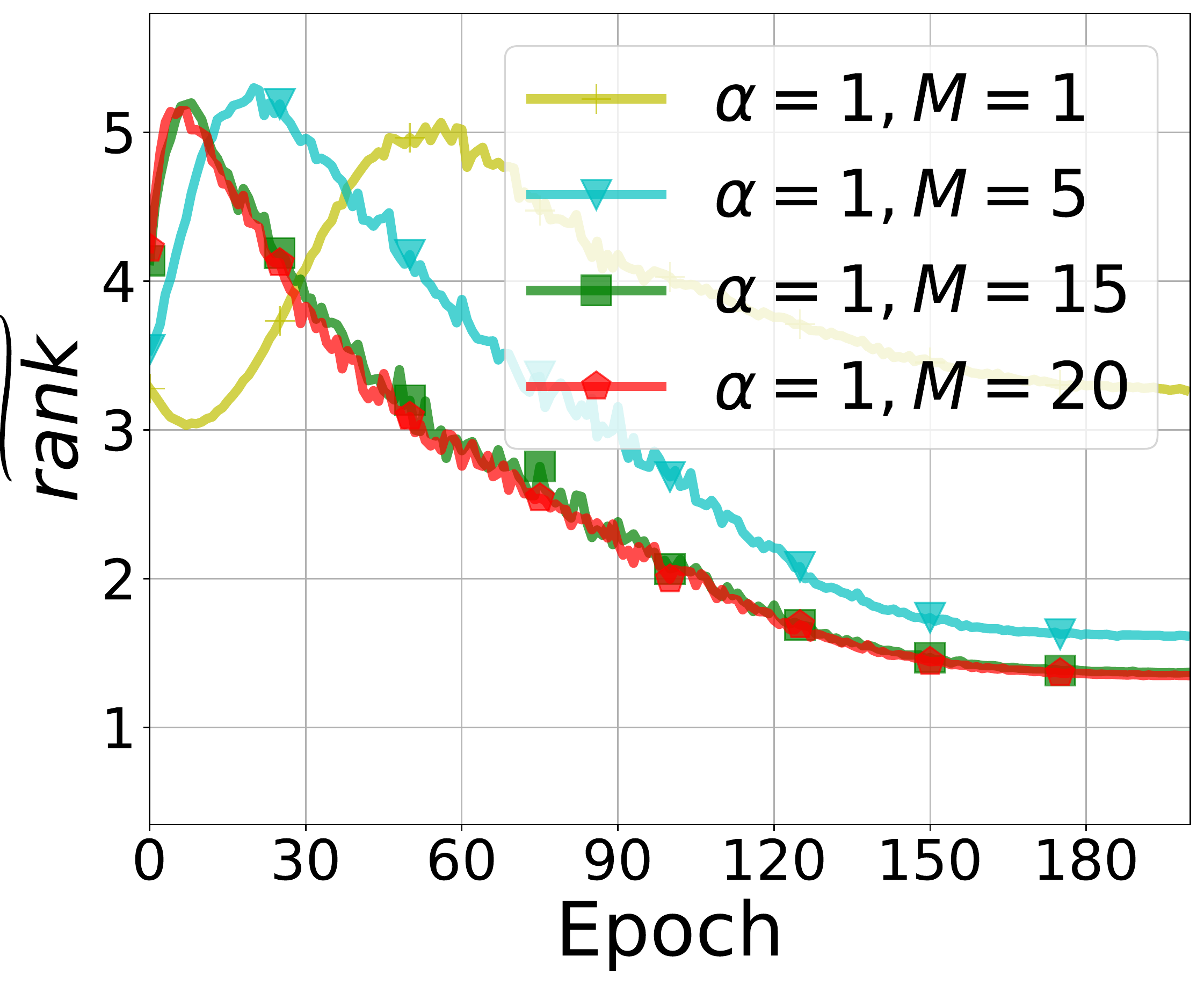}} \
    \subfloat[Train Acc.]{\includegraphics[width=0.162\textwidth]{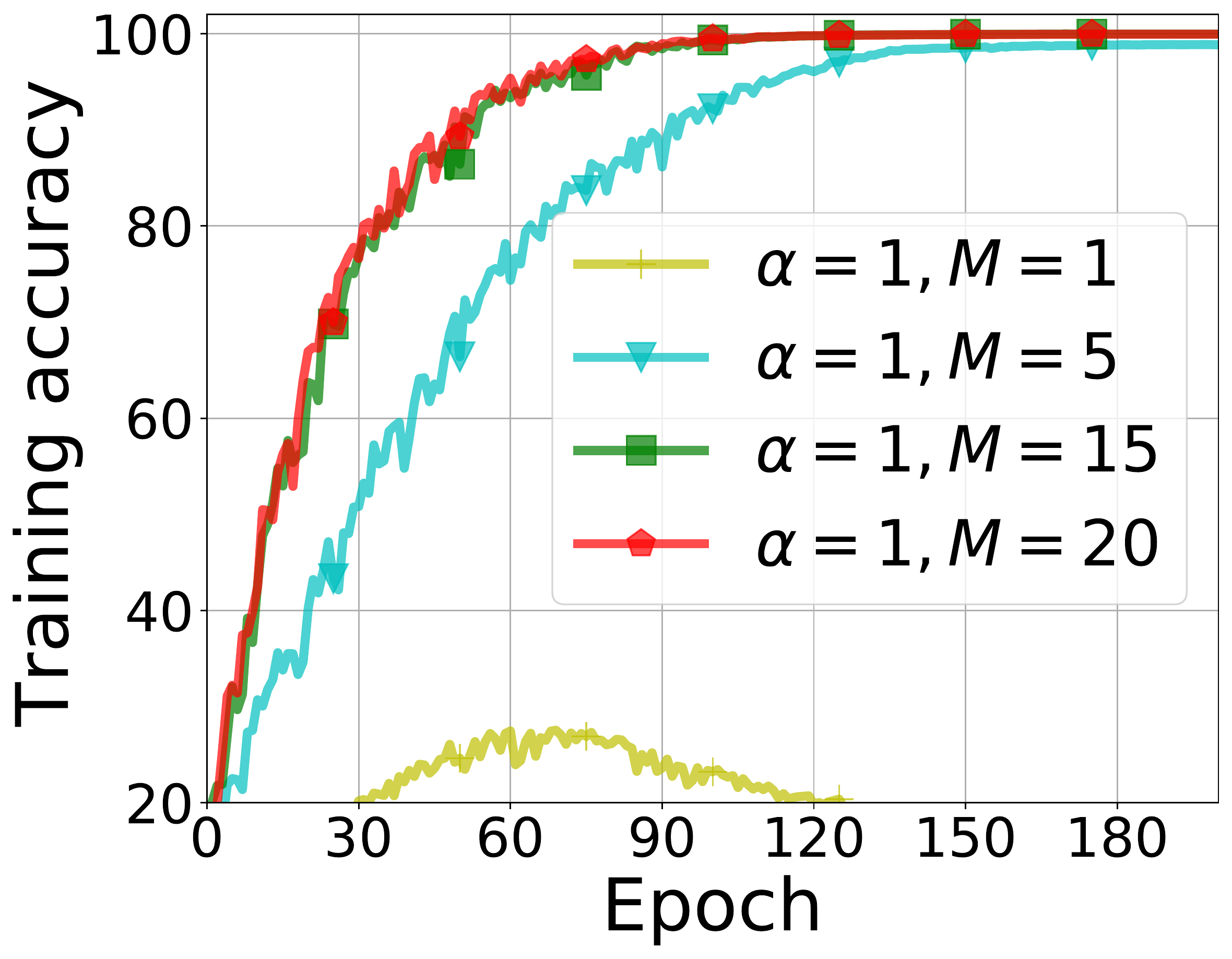}} \
    \subfloat[ Test Acc.]{\includegraphics[width=0.157\textwidth]{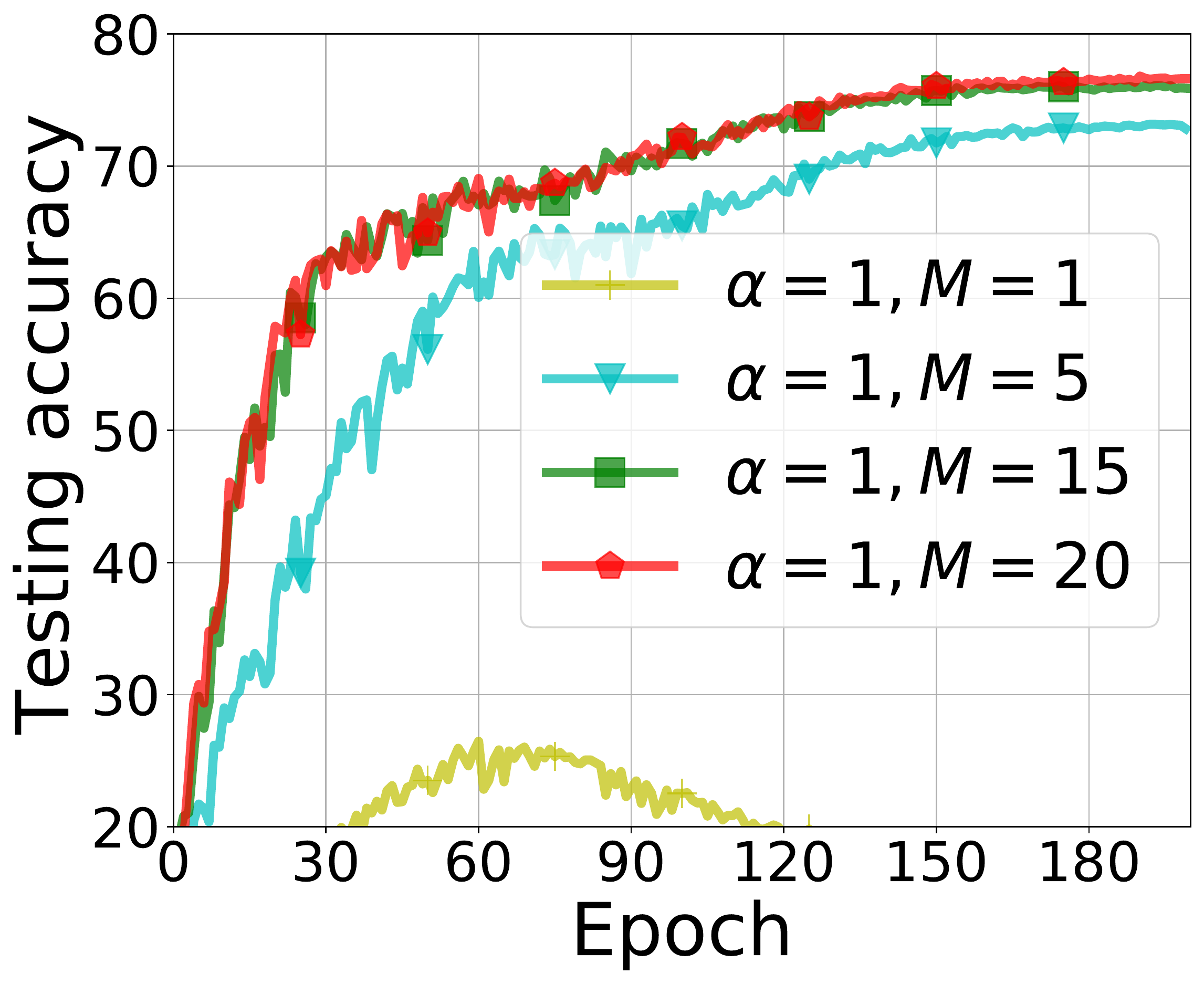}}

    \caption{\textbf{Effects of rescaling parameters $\alpha$ and $M$}. Experiments are conducted on the miniImageNet dataset (MIN) with a ResNet18 backbone. Top row shows the result of varying $\alpha$ with fixed $M$. Bottom row shows the result of varying $M$ with fixed $\alpha$.
    }
    \label{fig:mini_image_rescaling}
\end{figure*}

\section{Conclusion}\label{sec:conclusion}

In this work, we provide a global landscape analysis for deep neural networks trained via the MSE loss, under the unconstrained feature model.
Our theoretical results reveal that all global solutions exhibit the \NC~phenomenon, and that the global landscape is benign in the sense that it does not have spurious local minimizers. 
Such results extend the scope where \NC~provably occurs with the MSE loss, which was restricted to neural networks trained via particular and unrealistic algorithms in prior work \cite{mixon2020neural,han2021neural}.
More broadly, our results extend the scope of the ``prevalence of neural collapse'' in the seminal work \cite{papyan2020prevalence}, which was restricted to neural networks trained via the CE loss. 
Combined with the results in \cite{zhu2021geometric}, the prevalence of neural collapse now subsumes (at least) that deep neural networks trained for classification tasks with both CE and MSE losses exhibit neural collapse, regardless of the training algorithm (as long as it can escape strict saddle points) and network architecture (as long as it is sufficiently expressive).

 \paragraph{Towards designing better loss functions. } 
 As a future work, the improved understanding of \NC~with different choices of loss functions may help us to study and demystify the role of loss design for learning more generalizable and transferable deep features \cite{kornblith2021better,
 dikkala2021manifold,hui2022limitations,galanti2022on,ben2022nearest}.
 The fact that both CE and MSE exhibit the \NC~ does not mean that they are equally good at inducing neural collapse solutions in practical neural network training. As shown in our experiments, rescaling of the MSE loss is indispensable for improving \NC~hence producing better test performance over the vanilla MSE loss.  
 There is, however, no reason to be satisfied with the rescaled MSE loss since it is heuristically designed and does not have any justification on its ``optimality''.
Even though we are able to offer insights into the benefits of rescaling for MSE loss via landscape visualization, our explanation is approximate, based on extravagant simplifications of the optimization problem (by using two parameters $\theta$ and $s$ to summarize a very high-dimensional landscape!).  In practice, all the optimization variables $\bm H$, $\bm W$ and $\bm b$ are intricately correlated, and the insights gained from the visualization via simplification may hardly be useful for the design of new loss functions. The derivation of an ``optimal'' loss functions for inducing \NC~may require the development of new analysis techniques which we leave as future work. 

%============

%Could one envision similar tendencies in deep neural
%networks handling regression or synthesis tasks? Indeed, what is
%the parallel of the ideal classification to this breed of networks?
%These are important open questions that should be addressed, in
%our quest to demystify neural network solutions of inverse problems, generative adversarial networks, and more

%\zz{Not sure for language processing, but I don't see similar patterns in regression problems.}

%\paragraph{Optimization}

%\paragraph{Generalization}

%\paragraph{Robustness}

% \cy{Contrastive learning operates in the regime of $d \ll K$, where the features cannot form a Simplex ETF. Nonetheless, the features possess some ``uniformity'' properties that seem to generalize the notion of ETF \cite{wang2020understanding,chen2020intriguing}. May be of interest to generalize the study to those cases as future work. } 
% \qq{very good point}

\section*{Acknowledgements}

ZZ acknowledges support from NSF grants CCF 2008460 and CCF 2106881. XL and QQ acknowledge support from NSF grant DMS 2009752 and NSF Career Award 2143904. We also acknowledge Sheng Liu (NYU CDS) and Kangning Liu (NYU CDS) for fruitful discussion during various stages of the work.

% Authors are kindly asked to make their submissions as accessible as possible for everyone including people with disabilities and sensory or neurological differences.
% Tips of how to achieve this and what to pay attention to will be provided on the conference website \url{http://icml.cc/}.

% \section*{Software and Data}

% If a paper is accepted, we strongly encourage the publication of software and data with the
% camera-ready version of the paper whenever appropriate. This can be
% done by including a URL in the camera-ready copy. However, \textbf{do not}
% include URLs that reveal your institution or identity in your
% submission for review. Instead, provide an anonymous URL or upload
% the material as ``Supplementary Material'' into the CMT reviewing
% system. Note that reviewers are not required to look at this material
% when writing their review.

% % Acknowledgements should only appear in the accepted version.

% \textbf{Do not} include acknowledgements in the initial version of
% the paper submitted for blind review.

% If a paper is accepted, the final camera-ready version can (and
% probably should) include acknowledgements. In this case, please
% place such acknowledgements in an unnumbered section at the
% end of the paper. Typically, this will include thanks to reviewers
% who gave useful comments, to colleagues who contributed to the ideas,
% and to funding agencies and corporate sponsors that provided financial


\newpage 

{\small 
\bibliographystyle{unsrt}
\bibliography{biblio/optimization,biblio/learning}
}

%\bibliography{biblio/learning, biblio/optimization}
%\bibliographystyle{ICML/icml2022}

%%%%%%%%%%%%%%%%%%%%%%%%%%%%%%%%%%%%%%%%%%%%%%%%%%%%%%%%%%%%%%%%%%%%%%%%%%%%%%%
%%%%%%%%%%%%%%%%%%%%%%%%%%%%%%%%%%%%%%%%%%%%%%%%%%%%%%%%%%%%%%%%%%%%%%%%%%%%%%%
% APPENDIX
%%%%%%%%%%%%%%%%%%%%%%%%%%%%%%%%%%%%%%%%%%%%%%%%%%%%%%%%%%%%%%%%%%%%%%%%%%%%%%%
%%%%%%%%%%%%%%%%%%%%%%%%%%%%%%%%%%%%%%%%%%%%%%%%%%%%%%%%%%%%%%%%%%%%%%%%%%%%%%%
\newpage
\appendix
\onecolumn

% You can have as much text here as you want. The main body must be at most $8$ pages long.
% For the final version, one more page can be added.
% If you want, you can use an appendix like this one, even using the one-column format.

\newpage

\appendices

% --- PDF will be split by an editor (e.g. macOS preview), so need to restart from page 1
%\setcounter{page}{1}

% --- repeat the title (AT: haven't found a more elegant way to do this...)
%\onecolumn

%\begin{center}
%\Large
%Supplementary Material
%\end{center}

%\vspace{1.0em}
\paragraph{Notations and Organizations.} For a scalar function $f(\mZ)$ with a variable $\mZ\in\R^{K\times N}$, its Hessian can be represented by a bilinear form defined via $[\nabla^2 f(\mZ)](\mA,\mB) = \sum_{i,j,k,\ell} \frac{\partial^2 f(\mZ)}{\partial z_{ij}z_{k\ell}} a_{ij}b_{k\ell}$ for any $\mA,\mB\in\R^{K\times N}$, which avoids representing the Hessian as a tensor, or  vectorizing the variable $\mZ$. We will use the bilinear form for the Hessian throughout the Appendix. Now we give the formal definition of Simplex ETF.

\begin{definition}
[$K$-Simplex ETF \cite{strohmer2003grassmannian,papyan2020prevalence}]\label{def:simplex-ETF} 
A standard Simplex ETF is a collection of points in $\bb R^K$ specified by the columns of
\begin{align*}
    \mb M \;=\;  \sqrt{\frac{K}{K-1}}  \paren{ \mb I_K - \frac{1}{K} \mb 1_K \mb 1_K^\top },
\end{align*}
where $\mb I_K \in \bb R^{K \times K}$ is the identity matrix, and $\mb 1_K \in \bb R^K$ is the all ones vector.

As in \cite{papyan2020prevalence,fang2021layer}, in this paper we consider general Simplex ETF as a collection of points in $\R^d$ specified by the columns of $\sqrt{\frac{K}{K-1}} \mP \paren{ \mb I_K - \frac{1}{K} \mb 1_K \mb 1_K^\top }$, where $(i)$ when $d\ge K$, $\mP\in\R^{d\times K}$ is an orthonormal matrix, i.e., $\mP^\top \mP = \mb I_K$, and $(ii)$ when $d = K-1$, $\mP$ is chosen such that $\begin{bmatrix}\mP^\top & \frac{1}{\sqrt{K}}\vone_K \end{bmatrix}$ is an orthonormal matrix.%\cy{Should be $\begin{bmatrix}\mP^\top & \frac{1}{\sqrt{K}} \vone_K \end{bmatrix}$?}

\end{definition}

The appendix is organized as follows. In Appendix \ref{app:exp_setup}, we describe the datasets, network architectures and training settings. In Appendix \ref{app:thm-global}, we provide a detailed proof for \Cref{thm:global-minima}, analyzing the  global minimizers to our regularized MSE loss. %Finally, in Appendix \ref{sec:appendix-prf-global-geometry}, we present the whole proof for \Cref{thm:global-geometry} that the function is a strict saddle function and no spurious local minimizers exist, which is one of the major contributions of the work. 
Finally, in Appendix \ref{sec:appendix-visualization} we provide additional details for obtaining the visualization of rescaled MSE and CE losses presented in \Cref{subsec:landscape-rescaling}.

\section{Technical Details of the Experimental Setup in \Cref{sec:experiment}}\label{app:exp_setup}
%\section{Datasets, Network Architectures and Training, and Three Neural Collapse Measures }\label{app:exp_setup}
In \Cref{sec:experiment}, we conduct experiments on CIFAR10 \cite{krizhevsky2009learning} and miniImageNet \cite{vinyals2016matching} datasets. We note that for miniImageNet dataset, since we are not doing few-shot learning where the work \cite{vinyals2016matching} primarily considers, we split the total $60000$ images into training set ($50000$ images) and validation set ($10000$ images) such that both training and validation set include the full $100$ classes. All images from the datasets are normalized by their mean and variance channel-wise. We use the ResNet18 \cite{he2016deep} architecture throughout all the experiments. For CIFAR10, we use the same experiment setting in \cite{zhu2021geometric} except the replacement of CE loss by standard MSE loss for fair comparison. Specifically, we train ResNet18 for 200 epochs with three different optimizers: SGD, Adam and LBFGS. For SGD, the initial learning rate and momentum are set to $0.05$ and $0.9$, respectively. For Adam, the initial learning rate, $\beta_1$ and $beta_2$ are set to $0.001$, $0.9$ and $0.999$, respectively. We decay the learning rate by 0.1 every 40 epochs for SGD and Adam. We use LBFGS with an initial learning rate of 0.01 and strong Wolfe line search strategy for subsequent iterations. Without explicitly mentioned, we use the weight decay of $5\times 10^{-4}$ and the same data augmentation in \cite{zhu2021geometric} for all experiments on CIFAR10. For miniImageNet, we use the rescaled MSE loss as described in \Cref{subsec:rescaled-MSE} with the SGD optimizer with an initial learning rate $0.01$, momentum $0.9$ and weight decay $0.001$. We use a Cosine Annealing Warm Restarts \cite{loshchilov2017sgdr} learning rate scheduler where the number of epochs before the first restart is set as $200$ and the minimum learning rate is $0.0001$.

\paragraph{\textbf{Three \NC\ measures \NC$_1$-\NC$_3$} \cite{papyan2020prevalence,zhu2021geometric}} For the sake of completeness, we describe the three  \NC\ measures \NC$_1$-\NC$_3$ \cite{papyan2020prevalence,zhu2021geometric} used in \Cref{sec:experiment}. Towards that end, first define the global mean of the last-layer features $\Brac{ \mb h_{k,i} }$ as $\mb h_G \;=\; \frac{1}{nK} \sum_{k=1}^K  \sum_{i=1}^n \vh_{k,i}$ and the class mean as $\ol{\mb h}_k \;=\; \frac{1}{n} \sum_{i=1}^n \vh_{k,i} \;(1\leq k\leq K).$
\begin{itemize}[leftmargin=*]
    \item \textbf{\NC$_1$.} We measure the within-class variability collapse by 
    \begin{align}\label{eq:NC1}
    \mc {NC}_1\;:=\;\frac{1}{K}\trace\parans{\mSigma_ W\mSigma_B^\dagger},
\end{align}
where $\mb \Sigma_W \;:=\; \frac{1}{nK} \sum_{k=1}^K \sum_{i=1}^n \paren{  \mb h_{k,i}  -  \ol{\mb h}_k } \paren{ \mb h_{k,i} - \ol{\mb h}_k }^\top \in \bb R^{d \times d}$ denotes the within-class covariance of the features, $\mb \Sigma_B \;:=\; \frac{1}{K} \sum_{k=1}^K \paren{ \ol{\mb h}_k - \mb h_G } \paren{ \ol{\mb h}_k - \mb h_G }^\top \in \bb R^{d \times d}$ represents  the between-class covariance, and $\mSigma_B^\dagger$ denotes the pseudo inverse of $\mSigma_B$.
\item \textbf{\NC$_2$.} We measure the onvergence of the learned classifier $\mW \in\R^{K\times d}$ (for $d\ge K-1$) to a Simplex ETF by 
\begin{align}\label{eq:NC2}
\mc {NC}_2\;:=\;  \norm{ \frac{\mW \mW^\top}{\norm{\mW \mW^\top}{F}} \;-\; \frac{1}{\sqrt{K-1}}  \paren{ \mb I_K - \frac{1}{K} \mb 1_K \mb 1_K^\top } }{F},
\end{align}
where the Simplex ETF and $\mW\mW^\top$ are rescaled to have unit energy (in Frobenius norm).
\item \textbf{\NC$_2$.} For $d\ge K-1$, we measure the convergence to self-duality between the learned features $\mb H$ and the learned classifier $\mb W$ via
\begin{align}\label{eq:NC3}
 \mc {NC}_3\;:=\; \norm{\frac{\mW \ol{\mH}}{\norm{\mW\ol{\mb H}}{F}}   \;-\; \frac{1}{\sqrt{K-1}}  \paren{ \mb I_K - \frac{1}{K} \mb 1_K \mb 1_K^\top} }{F},
\end{align}
where $
    \ol{\mb H} \;:=\; \begin{bmatrix}
\ol{\mb h}_1 - \mb h_G & \cdots & \ol{\mb h}_K - \mb h_G
\end{bmatrix} \in \bb R^{d \times K}$ are the centered class-means.
\end{itemize}

\revise{\paragraph{Visual comparison of features learned by MSE and CE losses with feature dimension $d = 3$.}  To visualize the learned features, we set the feature dimension $d = 3$ for ResNet18 and train it with CIFAR10.  \Cref{fig:cifar10_distribution} display the learned features with MSE loss and CE loss on randomly selected 100 training samples for each class.
We observe that the features learned by CE loss is more diverse and discriminative than MSE loss.}

% \begin{figure*}[t]
%     \centering
%     \subfloat[$\mc {NC}_1$]{\includegraphics[width=0.21\textwidth]{figs/cifar10_res18_ce_dk/resnet18-NC1.pdf}} \
%     \subfloat[Training Acc.]{\includegraphics[width=0.222\textwidth]{figs/cifar10_res18_ce_dk/resnet18-train-acc.pdf}}\
%     \subfloat[Test Acc.]{\includegraphics[width=0.222\textwidth]{figs/cifar10_res18_ce_dk/resnet18-test-acc.pdf}}\
%     \subfloat[$\mc {P}_{CM}$]{\includegraphics[width=0.232\textwidth]{figs/cifar10_res18_ce_dk/c_margin_dist.pdf}} 
%     \vspace{-0.03in}
%     \caption{\textbf{Comparison of the performances on networks with different feature dimensions $d$ for CE loss.} We compare within-class variation collapse $\mc {NC}_1$, training accuracy, test accuracy and cosine margin distribution $\mathcal{P}_{CM}$ on learned classifier with different feature dimension $d$ on CIFAR10 using ResNet18 with data augmentation. The network is trained by the SGD optimizer.
%     }
%     \label{fig:cifar10_dim_ce}
% \end{figure*}

\begin{figure*}[t]
    \centering
    \subfloat[MSE]{\includegraphics[width=0.25\textwidth]{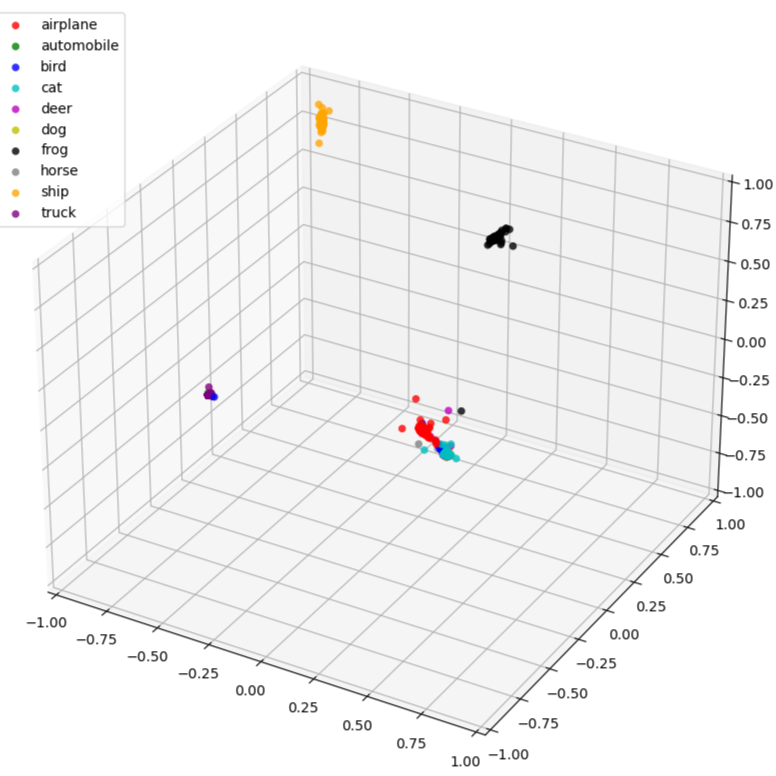}} \
    \subfloat[CE]{\includegraphics[width=0.25\textwidth]{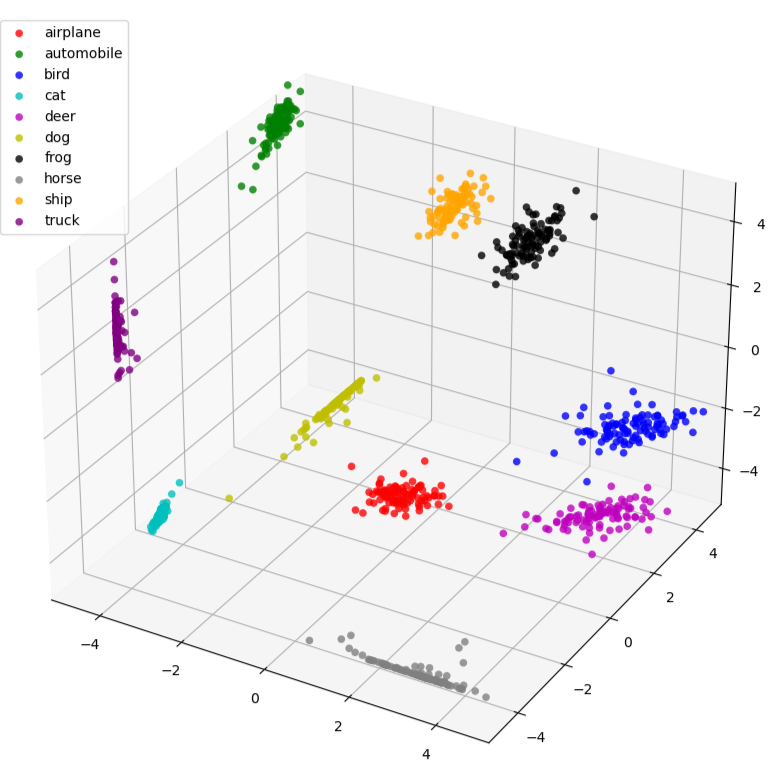}} 
    \vspace{-0.03in}
    \caption{\textbf{Visual comparison of features learned by MSE and CE losses with feature dimension $d = 3$.} 
    We compare the training feature distribution by setting the feature dimension $d = 3$ for ResNet18 and training it with CIFAR10. The network is trained by the SGD optimizer.
    }
    \label{fig:cifar10_distribution}
\end{figure*}
%\qq{the figure is too large}

%\section{Basics}\label{app:basics}
%\input{ICLR/Appendices/basics}

\section{Proof of \Cref{thm:global-minima} in \Cref{subsec:global_optim}}\label{app:thm-global}

In this part of appendices, we prove \Cref{thm:global-minima} in \Cref{sec:main} that we restate as follows.

\begin{theorem}[Global Optimality Condition]\label{thm:global-minima-app}
Let $(\mW^\star, \mH^\star,\vb^\star)$ be a global minimizer of 
	\begin{align}\label{eq:obj-app}
     \min_{\mb W , \mb H,\mb b  } \; f(\mb W,\mb H,\mb b) \;:=\; \frac{1}{2N}\norm{\mW\mH + \vb\vone^\top - \mY}{F}^2 \;+\; \frac{\lambda_{\mb W} }{2} \norm{\mb W}{F}^2 + \frac{\lambda_{\mb H} }{2} \norm{\mb H}{F}^2 + \frac{\lambda_{\mb b} }{2} \norm{\mb b}{2}^2.
\end{align}
Then $(\mW^\star, \mH^\star,\vb^\star)$ satisfies:
\begin{itemize}[leftmargin= 3em]
    \item [(\NC1,3)] If $\lambdaW\lambdaH < \frac{1}{NK}$, then $(\mW^\star,\mH^\star)$
satisfies \NC1 and \NC3 as
\begin{gather*}
      \vh_{k,i}^\star \;=\;  \ol\vh_k^\star, \ \sqrt{ \frac{ \lambda_{\mb W}  }{ \lambda_{\mb H} n } } \vw^{\star k} \;=\;  \ol\vh_k^\star ,\quad \forall \; k\in[K],\; i\in[n].
\end{gather*}
Otherwise, if $\lambdaW\lambdaH \ge \frac{1}{NK}$, then $\mW^\star = \vzero$ and $\mH^\star = \vzero$.
 \item [(\NC2)] If $\lambdaW\lambdaH < \frac{1}{NK}$, then $\ol\mH^\star$ further obeys the following properties for different $d$:
\begin{enumerate}
\item If $d< K-1$:  we have $\ol\mH^{\star\top}\ol\mH^\star = C_1 \calP_d(\mId - \frac{1}{K} \mb 1_K \mb 1_K^\top)$, where $\calP_d(\mb M)$ denotes the best rank-$d$ approximating of $\mb M$;
\item  If $d = K-1$: we have $\ol\mH^{\star\top}\ol\mH^\star= C_2 (\mb I - \frac{1}{K} \mb 1_K \mb 1_K^\top)$;
 \item  If $d\ge K$: we have
 \e
 \ol\mH^{\star\top}\ol\mH^\star = \begin{cases} C_3 \paren{ \mb I - \frac{1}{K} \mb 1_K \mb 1_K^\top}, &   \lambdab \le \frac{\sqrt{KN \lambdaW\lambdaH}}{1 - \sqrt{KN \lambdaW\lambdaH}}, \\ C_4 \paren{ \mId - \frac{\sqrt{n\lambdaW\lambdaH}}{\lambdab(1 - \sqrt{KN\lambdaW\lambdaH} )} \vone_K \vone_K^\top }, &  \text{otherwise}, 
        \end{cases}
    \nonumber \ee
  where  $\frac{\sqrt{n\lambdaW\lambdaH}}{\lambdab(1 - \sqrt{KN\lambdaW\lambdaH})}\le \frac{1}{K}$ in the second case since $\lambdab \ge \frac{\sqrt{KN \lambdaW\lambdaH}}{1 - \sqrt{KN \lambdaW\lambdaH}}$. 
\end{enumerate}
Here, $C_1$, $C_2$, $C_3$, and $C_4$ are some positive numerical constants that depend on $\lambdaW,\lambdaH,\lambdab$. 
  \item [(Bias)] The bias satisfies $\mb b^\star = b^\star \mb 1_K$ with $b^\star \le \frac{1}{K}$ given by:
    \begin{enumerate}
        \item If $d<K$: we have $b^\star  = \frac{1}{K(\lambdab + 1)}$;
        \item If $d \ge K$: we have $b^\star = \begin{cases} \frac{1}{K(\lambdab + 1)}, &   \lambdab \le \frac{\sqrt{KN \lambdaW\lambdaH}}{1 - \sqrt{KN \lambdaW\lambdaH}}, \\ \frac{\sqrt{n\lambdaW\lambdaH}}{\lambdab}, &  \text{otherwise}. 
        \end{cases}$
    \end{enumerate}
    In particular, when $\lambdab \rightarrow 0$, we have $b^\star \rightarrow \frac{1}{K}$; when $\lambdab \rightarrow \infty$, we have $b^\star \rightarrow 0$.
\end{itemize}

\end{theorem}

\subsection{Main Proof}

\begin{proof}[Proof of \Cref{thm:global-minima-app}] We first characterize the solutions $(\mW,\mH)$ in terms of $\vb$. Denote by $\wt\mY = \mY - \vb\vone^\top$ and let $\wt \mY = \mU\mSigma\mV^\top = \sum_{i=1}^K \sigma_i \vu_i\vv_i^\top$ be its SVD, where  $\sigma_1 \ge \sigma_2 \ge \cdots \ge \sigma_K \ge 0$ are the singular values. For convenience, we denote by $\wt\lambda = N\sqrt{\lambdaW\lambdaH}$.
By \Cref{lem:global-fact-nuclear}, we know
\begin{align}
	 f(\mb W,\mb H,\mb b) \ge \frac{\lambda_{\mb b} }{2} \norm{\mb b}{2}^2 + \frac{1}{N}\cdot \begin{cases}\sum\limits_{i=1}^K \frac{1}{2}\paren{\sigma_i - \brac{ \sigma_i - \wt\lambda}_+}^2 + \wt\lambda \brac{ \sigma_i - \wt\lambda}_+, & d\ge K\\
	   \sum\limits_{i=1}^d \frac{1}{2}\paren{\sigma_i - \brac{ \sigma_i - \wt\lambda}_+}^2 + \wt\lambda \brac{ \sigma_i - \wt\lambda}_+ + \sum\limits_{i = d+1}^K \frac{1}{2}\sigma_i^2 , & d<K
	    \end{cases}
\label{eq:proof-lower-bound}\end{align}
where the inequality becomes an equality when $\mW\mH = \sum_{i=1}^{\min(d,K)} \brac{\sigma_i - \sqrt{\lambdaW\lambdaH}}_+\vu_i\vv_i^\top$.

Noting that the singular values $\sigma_i$ also depend on $\vb$, to minimize the right hand side (RHS) of \eqref{eq:proof-lower-bound} in terms of $\vb$, we first rewrite each term involving the singular value as
\begin{align}
\frac{1}{2}\paren{\sigma_i - \brac{ \sigma_i - \wt\lambda}_+}^2 + \wt\lambda \brac{ \sigma_i - \wt\lambda}_+ = \begin{cases} 
\frac{1}{2}\sigma_i^2, & \sigma_i \le \wt\lambda,\\
\wt\lambda \sigma_i - \frac{1}{2}\wt\lambda^2, & \sigma_i \ge \wt\lambda,
\end{cases}
\label{eq:singular-value-reform}\end{align}
where for both cases it increases as $\sigma_i$ increases. 
Thus, for any $\vb$ with the same energy, say $c$, minimizing the RHS of \eqref{eq:proof-lower-bound} is equivalent to minimizing the singular values $\sigma_i$. With this in mind, we now show that if $\vb^\star$ is a minimizer to RHS of \eqref{eq:proof-lower-bound}, then $\|\vb^\star\| \le \frac{1}{\sqrt{K}}$. %We prove it by contradiction, i.e., we suppose $\|\vb^\star\| > \frac{1}{\sqrt{K}}$. 
By \Cref{lem:spectral-Ytilde}, we know for any $\vb$ we have $\sigma_2 = \sigma_3 = \cdots = \sigma_{K-1} = \sqrt{n}$ and $\sigma_1 \ge \sqrt{n}$ (see \eqref{eq:sigma1-lower-bound}). On the other hand, when $\vb = \frac{1}{K}\vone$, we have $\sigma_1 = \sigma_2 = \cdots = \sigma_{K-1} = \sqrt{n}$ and $\sigma_K = 0$, which are the smallest possible singular values that can be achieved. Thus, considering the weight decay term on \eqref{eq:proof-lower-bound}, the minimizer $\vb^\star$ must satisfy $\|\vb^\star\| \le \norm{\frac{\vone}{K}}{} = \frac{1}{\sqrt{K}}$.

Therefore, we only need to optimize over $\vb$ with $\norm{\vb}{} = c \le \frac{1}{\sqrt{K}}$. It this case, it follows from \Cref{lem:spectral-Ytilde} that $\sigma_2 = \cdots = \sigma_{K-1} = \sqrt{n}$, $\sigma_1 \ge  \sqrt{n}, \sigma_K \ge \sqrt{n}\paren{1 - \sqrt{K}c }$, and both inequalities become equalities \emph{if and only if} $\vb = \frac{c}{\sqrt{K}}\vone$. The remaining is to optimize the RHS of \eqref{eq:proof-lower-bound} in terms of $\sigma_K$ which depends on $c$. By \eqref{eq:proof-lower-bound} and \eqref{eq:singular-value-reform}, this problem reduces to 
\begin{align}
\min_{0\le c\le \frac{1}{\sqrt{K}}} \frac{\lambda_{\mb b} }{2} c^2  + 
\frac{n}{2N} \paren{1 - \sqrt{K}c }^2 
\label{eq:problem-c-1}\end{align}
if $d<K$, and otherwise reduces to
\begin{align}
\begin{cases}
\min_{0\le c\le \frac{1}{\sqrt{K}}} \frac{\lambda_{\mb b} }{2} c^2 + \frac{\wt\lambda}{N} \paren{
  \sqrt{n}\paren{1 - \sqrt{K}c } - \frac{1}{2}\wt\lambda}, &  \sqrt{n}\paren{1 - \sqrt{K}c } \ge \wt \lambda\\ 
\min_{0\le c\le \frac{1}{\sqrt{K}}} \frac{\lambda_{\mb b} }{2} c^2 + \frac{n}{2N} \paren{1 - \sqrt{K}c }^2 , & \sqrt{n}\paren{1 - \sqrt{K}c } \le \wt \lambda
\end{cases}
\label{eq:problem-c-2}\end{align}
% \begin{align}
% \min_{0\le c\le \frac{1}{\sqrt{K}}} \frac{\lambda_{\mb b} }{2} c^2 + \frac{1}{N} \cdot \begin{cases}
%  \wt\lambda \sqrt{n}\paren{1 - \sqrt{K}c } - \frac{1}{2}\wt\lambda^2, & d\ge K \ \& \  \sqrt{n}\paren{1 - \sqrt{K}c } \ge \wt \lambda,\\
% \frac{1}{2}n \paren{1 - \sqrt{K}c }^2, & \text{otherwise}.
% \end{cases}
% \label{eq:problem-c}\end{align}
We now consider the two cases as follows:
\begin{enumerate}
    \item Case I: $d<K$. In this case, the problem \eqref{eq:problem-c-1} achieves its minimum at $c^\star = \frac{1}{\sqrt{K}(\lambdab + 1)}$.
    \item Case II: $d\ge K$. In this case, when $c\ge \frac{1}{\sqrt{K}}\paren{1 - \frac{\wt\lambda}{\sqrt{n}}}$,   problem \eqref{eq:problem-c-2} becomes \eqref{eq:problem-c-1}, and thus its minimum among $c\ge \frac{1}{\sqrt{K}}\paren{1 - \frac{\wt\lambda}{\sqrt{n}}}$ is $c^\star = \max\paren{\frac{1}{\sqrt{K}(\lambdab + 1)}, \frac{1}{\sqrt{K}}\paren{1 - \frac{\wt\lambda}{\sqrt{n}}}}$.
   On the other hand, when $c\le \frac{1}{\sqrt{K}}\paren{1 - \frac{\wt\lambda}{\sqrt{n}}}$, the problem \eqref{eq:problem-c-2} is also a quadratic function on  $c$ and achieves its minimum among  $c\le \frac{1}{\sqrt{K}}\paren{1 - \frac{\wt\lambda}{\sqrt{n}}}$ is $c^\star = \min\paren{\frac{\wt\lambda}{\lambdab\sqrt{N}}, \frac{1}{\sqrt{K}}\paren{1 - \frac{\wt\lambda}{\sqrt{n}}}}$. 
   
   We now find the minimum value among these two cases. When $\frac{1}{\sqrt{K}(\lambdab + 1)} \ge \frac{1}{\sqrt{K}}\paren{1 - \frac{\wt\lambda}{\sqrt{n}}}$, i.e., $\lambdab \le \frac{\sqrt{KN \lambdaW\lambdaH}}{1 - \sqrt{KN \lambdaW\lambdaH}}$, we have $\frac{\wt\lambda}{\lambdab\sqrt{N}}\ge \frac{1}{\sqrt{K}}\paren{1 - \frac{\wt\lambda}{\sqrt{n}}}$, which together with the form of the two quadratic functions implies that the minimum is achieved when $c^\star = \frac{1}{\sqrt{K}(\lambdab + 1)}$. On the other hand, when $\frac{1}{\sqrt{K}(\lambdab + 1)} < \frac{1}{\sqrt{K}}\paren{1 - \frac{\wt\lambda}{\sqrt{n}}}$, i.e., $\lambdab > \frac{\sqrt{KN \lambdaW\lambdaH}}{1 - \sqrt{KN \lambdaW\lambdaH}}$, we have $\frac{\wt\lambda}{\lambdab\sqrt{N}}< \frac{1}{\sqrt{K}}\paren{1 - \frac{\wt\lambda}{\sqrt{n}}}$, which together with the form of the two quadratic functions implies that the minimum is achieved when $c^\star = \frac{\wt\lambda}{\lambdab\sqrt{N}} = \frac{\sqrt{N\lambdaW\lambdaH}}{\lambdab}$. Thus, we can also conclude that $c^\star \rightarrow 0$ when $\lambdab \rightarrow \infty$ and $c^\star \rightarrow \frac{1}{\sqrt{K}}$ when $\lambdab \rightarrow 0$.
    
\end{enumerate}

The proof is completed by invoking \Cref{lem:global-one-hot-nuclear} to characterize $(\mW^\star,\mH^\star)$.
\end{proof}

\subsection{Supporting Lemmas}

We first characterize the following balance property between $\mW$ and $\mH$ for any critical point  $(\mW,\mH,\vb)$ of our loss function:

\begin{lemma}\label{lem:global-fact-nuclear} 
For any $K,d,N$, and $\wt{\mY}\in\R^{K\times N}$ with SVD given by $\wt \mY = \mU\mSigma\mV^\top = \sum_{i=1}^K \sigma_i \vu_i\vv_i^\top$ where  $\sigma_1 \ge \sigma_2 \ge \cdots \ge \sigma_K \ge 0$ are the singular values, the following problem
\begin{align}\label{eq:global-fact-nuclear}
  \min_{\mW\in\R^{K\times d},\mH \in\R^{d\times N}} \xi(\mW,\mH) \;= \; \frac{1}{2}\norm{\mW\mH - \wt\mY}{F}^2 \;+\; \frac{\lambda_{\mb W} }{2} \norm{\mb W}{F}^2 + \frac{\lambda_{\mb H} }{2} \norm{\mb H}{F}^2
\end{align}
is a strict saddle function with no spurious local minimizer, in the sense that
	\begin{itemize}
	    \item Any local minimizer $(\mW^\star,\mH^\star)$ of \eqref{eq:global-fact-nuclear} is a global minimizer of \eqref{eq:global-fact-nuclear}, with the following form 
	    \[
	    \mW^\star \mH^\star \;=\; \sum_{i=1}^{\min(d,K)} \eta_i \vu_i\vv_i^\top,
	    \]
	    where we let $\eta_i(\lambdaW,\lambdaH) := \brac{ \sigma_i - \sqrt{\lambdaW\lambdaH}}_+$.
	    Correspondingly, the minimal objective value of \eqref{eq:global-fact-nuclear} is
	    \begin{align}\label{eqn:minimum-obj-val}
	    \xi_\star  \;=\; \begin{cases}\sum_{i=1}^K \frac{1}{2}\paren{\sigma_i - \eta_i }^2 + \sqrt{\lambdaW\lambdaH}\; \eta_i , & d\ge K\\
	   \sum_{i=1}^d \frac{1}{2}\paren{\sigma_i - \eta_i }^2 + \sqrt{\lambdaW\lambdaH} \; \eta_i + \sum_{i = d+1}^K \sigma_i^2 , & d<K
	    \end{cases}.
	    \end{align}
	 
	    \item Any critical point $(\mb W,\mb H)$ of \eqref{eq:global-fact-nuclear} that is not a local minimizer is a strict saddle with negative curvature, i.e. the Hessian at this critical point has at least one negative eigenvalue.
	\end{itemize}
\end{lemma}

\begin{proof}[Proof of Lemma \ref{lem:global-fact-nuclear} ]
By definition, any critical point $(\mW,\mH)$ of \eqref{eq:global-fact-nuclear} satisfies the following:
\begin{align*}
    	\nabla_{\mW}\xi(\mW,\mH) \;&=\; (\mW\mH - \wt\mY)\mH^\top + \lambdaW \mW \;=\; \vzero,\\
	\nabla_{\mH}\xi(\mW,\mH) \;&=\;  \mW^\top (\mW\mH - \wt\mY) + \lambdaH \mH \;=\; \vzero. 
\end{align*}
By left multiplying the first equation by $\mb W^\top$ on both sides and then right multiplying second equation by $\mH^\top$ on both sides and combining the equations together, we obtain
\begin{align}
    \lambdaW \mW^\top \mW \;=\; \lambdaH \mH \mH^\top.
\label{eq:critical-balance}\end{align}

This further gives  
\e\begin{split}
    \frac{\lambdaW}{\lambdaH} \mW\mW^\top\mW  + \lambdaW \mW \;&=\; \wt\mY\mH^\top,\\
\frac{\lambdaH}{\lambdaW} \mH^\top \mH\mH^\top + \lambdaH \mH^\top \;&=\; \wt\mY^\top \mW. 
\end{split}
\label{eqn:crtical-fact-nuclear}
\ee
In the following, without loss of generality, we assume that the critical point $(\mb W,\mb H)$ satisfying the above equations has the form 
\begin{align}\label{eqn:WH-orthogonal}
    \mb W \;=\; \begin{bmatrix}
    \wh{\mb W} & \mb 0
    \end{bmatrix},\quad \mb H \;=\; \begin{bmatrix}
    \wh{\mb H} \\ \mb 0
    \end{bmatrix}
\end{align}
where the columns of $\wh{\mb W}$ are orthogonal and the rows of $\wh{\mb H}$ are orthogonal, and the zeros $\mb 0$ in $\mb W$ and $\mb H$ might or might not exist depending on the rank of $\mb W$ and $\mb H$. 
The underlying reasoning is that, for any $\mb W$ satisfying \eqref{eqn:crtical-fact-nuclear}, the Gram-Schmidt process implies that we can always orthogonalize $\mb W$ by an orthonormal matrix $\mb R\in \bb R^{d\times d}$ (i.e., $\mb R^\top \mb R = \mb R\mb R^\top  = \mb I$), such that $\wt{\mb W} = \mb W\mb R = \begin{bmatrix}\wh{\mb W} & \mb 0 \end{bmatrix}$. On the other hand, let $\wt{\mb H}= \mb R^\top \mb H $. Because $\lambdaW \mW^\top \mW \;=\; \lambdaH \mH\mH^\top$, we have $\lambdaW \wt\mW^\top \wt\mW \;=\; \lambdaH \wt\mH \wt\mH^\top$, which implies that the rows of $\wt \mH$ are also orthogonal. Therefore, multiply $\mb R$ on both sides of \eqref{eqn:crtical-fact-nuclear}, we always have
\begin{align*}
        \frac{\lambdaW}{\lambdaH} \wt{\mW}\wt{\mW}^\top\wt{\mW}  + \lambdaW \wt{\mW} \;=\; \wt\mY\wt{\mH}^\top,\quad \frac{\lambdaH}{\lambdaW} \wt{\mH}^\top \wt{\mH}\wt{\mH}^\top + \lambdaH \wt{\mH}^\top \;=\; \wt\mY^\top \wt{\mW}. 
\end{align*}
Thus, we can verify that $(\wt\mW,\wt\mH)$ is also a critical point with $\wt\mW \wt\mH = \mW\mH$ and has the same Hessian information as $(\mW,\mH)$. Thus, without the loss of generality, we can assume orthogonal $(\mb W,\mb H)$ in the form \eqref{eqn:WH-orthogonal}, but with possible zero columns.

%We can apply the Gram-Schmidt process to orthonormalize the columns of $\mW$ such that $\wt \mW = \mW\mR$, where $\wt \mW$ is orthogonal but with possible zero columns and $\mR\in\R^{d\times d}$ is orthonormal, i.e., $\mR\mR^\top = \mId$. Also let $\wt\mH = \mR^\top \mH$. Since $\lambdaW \mW^\top \mW \;=\; \lambdaH \mH\mH^\top$, we have $\lambdaW \wt\mW^\top \wt\mW \;=\; \lambdaH \wt\mH \wt\mH^\top$, which implies that $\wt \mH$ is also orthogonal. Moreover, we can verify that $(\wt\mW,\wt\mH)$ is also a critical point with $\wt\mW \wt\mH = \mW\mH$ and has the same Hessian information as $(\mW,\mH)$. Thus, without loss of generality, we assume the critical point $\mW$ is orthogonal, but with possible zero columns. 

\paragraph{Form of the global solutions.} Based on the orthogonalization, we further decompose \eqref{eqn:crtical-fact-nuclear} for all $i=1,\cdots, d$ columns of $\mb W$ as
\e\begin{split}
   \paren{ \frac{\lambdaW}{\lambdaH} \norm{\vw_i}{}^2  + \lambdaW } \vw_i \;&=\; \wt\mY\vh^i,\\
\paren{\frac{\lambdaH}{\lambdaW} \norm{\vh^i}{}^2 + \lambdaH} \vh^i \;&=\; \wt\mY^\top \vw_i,
\end{split}
\label{eqn:crtical-fact-nuclear-2}
\ee
which implies that either (\emph{i}) $\vw_i = \vzero$ and $\vh^i = \vzero$, or (\emph{ii}) $\vw_i, \vh^i$ are the (scaled) left and right singular vectors of $\wt\mY$. In particular, when $\vw_i \neq \vzero$ and $\vh_i \neq \vzero$, then by \eqref{eq:critical-balance}, it gives
\begin{align}\label{eqn:norm-relationship}
    \norm{\vh^i}{}^2 = \frac{\lambdaW}{\lambdaH} \norm{\vw_i}{}^2.
\end{align}
By further plugging the equation above into \eqref{eqn:crtical-fact-nuclear-2}, it gives
\e\begin{split}
   \paren{ \sqrt{\frac{\lambdaW}{\lambdaH}} \norm{\vw_i}{}^2  + \sqrt{\lambdaW\lambdaH} } \frac{\vw_i}{\norm{\vw_i}{}} \;&=\; \wt\mY \frac{\vh^i}{\norm{\vh^i}{}},\\
\paren{ \sqrt{\frac{\lambdaW}{\lambdaH}} \norm{\vw_i}{}^2  + \sqrt{\lambdaW\lambdaH} } \frac{\vh^i}{\norm{\vh^i}{}} \;&=\; \wt\mY^\top \frac{\vw_i}{\norm{\vw_i}{}}.
\end{split}
\label{eqn:crtical-fact-nuclear-3}
\ee
Thus, when  $\vw_i \neq \vzero$ and $\vh_i \neq \vzero$, we conclude that $\sqrt{\frac{\lambdaW}{\lambdaH}} \norm{\vw_i}{}^2  + \sqrt{\lambdaW\lambdaH}$ is a singular value of $\wt\mY$, say $\sigma_{i_j}$, and $\frac{\vw_i}{\norm{\vw_i}{}}$ and $\frac{\vh^i}{\norm{\vh^i}{}}$ are the corresponding left and right singular vectors, respectively. In other words, when $\vw_i \neq \vzero$ and $\vh_i \neq \vzero$, then
\begin{align}\label{eqn:singular-values-Y}
    \sigma_{i_j} \;=\; \sqrt{\frac{\lambdaW}{\lambdaH}} \norm{\vw_i}{}^2  + \sqrt{\lambdaW\lambdaH}, \quad \mb u_{i_j} \;=\; \frac{\vw_i}{\norm{\vw_i}{}},\quad \mb v_{i_j} \;=\; \frac{\vh^i}{\norm{\vh^i}{}}
\end{align}
for some $i_j$ such that $\sigma_{i_j} > \sqrt{\lambdaW\lambdaH}$. Together with \eqref{eqn:norm-relationship}, it further implies that
\begin{align*}
    \mb w_i \mb h^{i \top}\;=\; \norm{\mb w_i}{2}^2 \frac{ \mb w_i}{ \norm{\mb w_i}{2} } \frac{ \mb h^{i \top} }{ \norm{\mb w_i}{2} } \;=\;\sqrt{\frac{\lambdaW}{\lambdaH}} \norm{\mb w_i}{2}^2  \frac{ \mb w_i}{ \norm{\mb w_i}{2} } \frac{ \mb h^{i \top} }{ \norm{\mb h^i}{2} } \;=\; \paren{\sigma_{i_j} - \sqrt{\lambdaW\lambdaH}} \mb u_{i_j} \mb v_{i_j}^\top .
\end{align*}
%When $\vw_i \neq \vzero$ and $\vh_i \neq \vzero$, 
Next, we discuss global minimizers and global function values in two cases: (\emph{i}) $d \geq K$, and (\emph{ii}) $d<K$. For both cases, based on the above results, we can write
\begin{align*}
    \mb W \mb H^\top \;=\; \sum_{i=1}^d \mb w_i \mb h^{i\top} \;& =\; \sum_{ \vw_i \neq \vzero, \vh^i \neq \vzero } \paren{\sigma_{i_j} - \sqrt{\lambdaW\lambdaH}} \mb u_{i_j} \mb v_{i_j}^\top +  \sum_{ \vw_i = \vzero \text{ and } \vh^i = \vzero }  \mb w_i \mb h^{i \top} \\
    \;&=\;  \sum_{ \vw_i \neq \vzero, \vh^i \neq \vzero } \paren{\sigma_{i_j} - \sqrt{\lambdaW\lambdaH}} \mb u_{i_j} \mb v_{i_j}^\top.
\end{align*}

\noindent \textbf{Case I: $d\geq K$}. In this case, given the rank of $\mb W$ is at most $K$, we know that the minimum is achieved when
\begin{align*} 
    \mW^\star \mH^\star \;=\; \sum_{i=1}^{K} \brac{\sigma_i - \sqrt{\lambdaW\lambdaH}}_+\vu_i\vv_i^\top 
\end{align*}
with $\sigma_i \geq \sqrt{\lambdaW\lambdaH}$ for all $i=1,\cdots, K$. In this case, we have
\begin{align*}
    \xi_\star \;&=\; \frac{1}{2}\sum_{i=1}^{ K } \paren{ \brac{ \sigma_i - \sqrt{\lambdaW\lambdaH} }_+  - \sigma_i}^2 + \frac{\lambdaW}{2} \sum_{i=1}^d \norm{\mb w_i}{2}^2  + \frac{\lambdaH}{2} \sum_{i=1}^d \norm{\mb h^i}{2}^2 \\
    \;&=\; \frac{1}{2}\sum_{i=1}^{ K } \paren{ \brac{ \sigma_i - \sqrt{\lambdaW\lambdaH} }_+  - \sigma_i}^2 + \lambdaW \sum_{i=1}^d \norm{\mb w_i}{2}^2 \\
    \;&=\; \frac{1}{2}\sum_{i=1}^{ K } \paren{ \brac{ \sigma_i - \sqrt{\lambdaW\lambdaH} }_+  - \sigma_i}^2 + \sqrt{ \lambdaH \lambdaW } \sum_{i=1}^{ K } \brac{ \sigma_i - \sqrt{\lambdaW\lambdaH} }_+ ,
\end{align*}
where for the second and third equality, we used \eqref{eqn:norm-relationship} and \eqref{eqn:singular-values-Y}, respectively.

\noindent \textbf{Case II: $d< K$}. In this case, we know that the minimum is achieved when
\begin{align*} 
    \mW^\star \mH^\star \;=\; \sum_{i=1}^{d} \brac{\sigma_i - \sqrt{\lambdaW\lambdaH}}_+\vu_i\vv_i^\top 
\end{align*}
with $\sigma_i \geq \sqrt{\lambdaW\lambdaH}$ for all $i=1,\cdots, d$. Similarly, we have
\begin{align*}
    \xi_\star \;&=\; \frac{1}{2}\sum_{i=1}^{ d } \paren{ \brac{ \sigma_i - \sqrt{\lambdaW\lambdaH} }_+  - \sigma_i}^2 + \sqrt{ \lambdaH \lambdaW } \sum_{i=1}^{ d } \brac{ \sigma_i - \sqrt{\lambdaW\lambdaH} }_+ + \sum_{i=d+1}^K \sigma_i^2 ,
\end{align*}
where the extra term $ \sum_{i=d+1}^K \sigma_i^2$ is coming from the singular values of $\wh{\mb Y}$ and the decomposition of $\frac{1}{2}\norm{\mW\mH - \wt\mY}{F}^2 \;+\; \frac{\lambda_{\mb W} }{2} \norm{\mb W}{F}^2$. 

In summary, the minimum function value is obtained when 
\begin{align}\label{eqn:global-minimizer-condition}
    \mW^\star \mH^\star \;=\; \sum_{i=1}^{\min \{d,K\}} \brac{\sigma_i - \sqrt{\lambdaW\lambdaH}}_+\vu_i\vv_i^\top \;=\; \sum_{i=1}^{\min \{d,K\}}\eta_i \vu_i\vv_i^\top,
\end{align}
with $\eta_i(\lambdaW,\lambdaH) := \brac{ \sigma_i - \sqrt{\lambdaW\lambdaH}}_+$, and the minimum function value is attained as in \eqref{eqn:minimum-obj-val}.

\paragraph{Showing negative curvature for strict saddles.}
In the remaining part, we show those critical point $(\mW,\mH)$ that does not satisfy the condition in \eqref{eqn:global-minimizer-condition} are strict saddle points, by showing that the Hessian of \eqref{eq:global-fact-nuclear} has negative eigenvalues. First, we derive the directional Hessian of \eqref{eq:global-fact-nuclear}, which has the following form
\begin{align} 
2[\nabla^2 \xi(\mW,\mH)](\mDelta,\mDelta) \;=\;& \norm{\mDelta_{\mW}\mH + \mW\mDelta_{\mH}}{F}^2 + 2\innerprod{\mW\mH - \wt\mY}{\mDelta_{\mW}\mDelta_{\mH}} \nonumber  \\
&+ \lambdaW\norm{\mDelta_{\mW}}{F}^2 + \lambdaH\norm{\mDelta_{\mH}}{F}^2. \label{eqn:hessian-form}
\end{align}
Given that a critical point $(\mW,\mH)$ is not a global minimizer, then \eqref{eqn:global-minimizer-condition} is not satisfied. This implies that  there must exist a singular value of $\wh{\mb Y}$ with $\sigma_j > \sqrt{\lambdaW\lambdaH}$, which cannot be not covered by any $(\vw_i,\vh_i)$ in the sense that $\vw_j\vh^{j\top} \neq (\sigma_{j} - \sqrt{\lambdaW\lambdaH})\vu_{j}\vv_{j}^\top$ for some $j$. We now discuss this situation separately in two cases: (\emph{i}) $d \geq K$, and (\emph{ii}) $d<K$.

\paragraph{Case I: $d\ge K$.} In this case, since each column of $\mW$ is either zero or corresponds to the left singular vectors of $\wt\mY$, it implies that the column space of $\mW$ has a non-trivial null space, i.e., there must exist a unit vector $\valpha \in\R^{d}$ such that $\mW\valpha = \vzero$. Since $\lambdaW \mW^\top \mW \;=\; \lambdaH \mH\mH^\top$,  we also have $\valpha^\top \mH = \vzero$. With this property, for the index $j$ with $\vw_j\vh^{j\top} \neq (\sigma_{j} - \sqrt{\lambdaW\lambdaH})\vu_{j}\vv_{j}^\top$, we construct $\mDelta_{\mW} = \paren{\frac{\lambdaH}{\lambdaW}}^{1/4} \vu_j\valpha^\top, \mDelta_{\mH} = \paren{\frac{\lambdaW}{\lambdaH}}^{1/4} \valpha \vv_j^\top$. Given that $\mDelta_{\mW}\mH = \mb 0$ and $\mW\mDelta_{\mH} = \mb 0$
\begin{align*}
\norm{\mDelta_{\mW}\mH + \mW\mDelta_{\mH}}{F}^2 \;&=\; 0,\\
\innerprod{\mW\mH - \wt\mY}{\mDelta_{\mW}\mDelta_{\mH}} \;& =\; - \sigma_j,\\
\lambdaW\norm{\mDelta_{\mW}}{F}^2 + \lambdaH\norm{\mDelta_{\mH}}{F}^2 \;&=\; 2\sqrt{\lambdaW\lambdaH}. 
\end{align*}
Plugging this into the Hessian \eqref{eqn:hessian-form}, it gives
\[
2[\nabla^2 \xi(\mW,\mH)](\mDelta,\mDelta) \;=\; -2\sigma_j + 2\sqrt{\lambdaW\lambdaH} \;=\; -2(\sigma_j - \sqrt{\lambdaW\lambdaH})\;<\;0.
\]
This implies that there exists a negative curvature for the Hessian, and the saddle point must be strict saddle.

\paragraph{Case II: $d< K$.} Recall from \eqref{eqn:WH-orthogonal} and \eqref{eqn:singular-values-Y} that $\sqrt\frac{\lambdaW}{\lambdaH}\mW^\top\mW$ is a diagonal matrix with the values of diagonal entry from $\sets{\brac{\sigma_1 - \sqrt{\lambdaW\lambdaH}}_+,\ldots,\brac{\sigma_K - \sqrt{\lambdaW\lambdaH}}_+,0}$, but here it excludes $\brac{\sigma_j - \sqrt{\lambdaW\lambdaH}}_+$ which equals $\sigma_j - \sqrt{\lambdaW\lambdaH}$ by our assumption. Thus, $\sqrt\frac{\lambdaW}{\lambdaH}\mW^\top\mW$ has at least one diagonal entry which is strictly smaller than $\sigma_j - \sqrt{\lambdaW\lambdaH}$. Now let $\valpha\in\R^d$ be the eigenvector associated with the smallest eigenvalue of $\mW^\top\mW$, so that  %\qq{isn't smallest eigenvalue should be 0?}
\begin{align*}
    \nu:=\sqrt\frac{\lambdaW}{\lambdaH}\valpha^\top \mW^\top\mW\valpha< \sigma_j - \sqrt{\lambdaW\lambdaH}. 
\end{align*}
Since $\lambdaW \mW^\top \mW \;=\; \lambdaH \mH\mH^\top$, we also have $\sqrt\frac{\lambdaH}{\lambdaW}\valpha^\top \mH^\top\mH\valpha = \nu$. With this property, we construct $\mDelta_{\mW} = \paren{\frac{\lambdaH}{\lambdaW}}^{1/4} \vu_j\valpha^\top, \mDelta_{\mH} = \paren{\frac{\lambdaW}{\lambdaH}}^{1/4} \valpha \vv_j^\top$, which satisfies 
\begin{align*}
\norm{\mDelta_{\mW}\mH + \mW\mDelta_{\mH}}{F}^2 &= \sqrt\frac{\lambdaW}{\lambdaH}\valpha^\top \mW^\top\mW\valpha + \sqrt\frac{\lambdaH}{\lambdaW}\valpha^\top \mH^\top\mH\valpha = 2\nu,\\
\innerprod{\mW\mH - \wt\mY}{\mDelta_{\mW}\mDelta_{\mH}} & = - \sigma_j,\\
\lambdaW\norm{\mDelta_{\mW}}{F}^2 + \lambdaH\norm{\mDelta_{\mH}}{F}^2 & = 2\sqrt{\lambdaW\lambdaH}. 
\end{align*}
Plugging this into the Hessian quadratic form gives
\[
2[\nabla^2 \xi(\mW,\mH)](\mDelta,\mDelta) = 2\nu -2\sigma_j + 2\sqrt{\lambdaW\lambdaH} = -2(\sigma_j - \sqrt{\lambdaW\lambdaH} - \nu)<0.
\]
Therefore, we prove $(\mW,\mH)$ is a strict saddle for both cases. This completes the proof.
\end{proof}

\begin{lemma} Assume the number of training samples in each class is balanced, i.e., $n = n_1 = \cdots = n_K$, and let $\mY = \begin{bmatrix}\vy_1 & \cdots & \vy_1 & \vy_2 & \cdots & \vy_K \end{bmatrix} \in \R^{K\times nK}$ be the matrix that contains the one-hot vectors for all the training samples. Then $\wt \mY = \mY - \vb\vone^\top$ has at least $K-2$ singular values being $\sqrt{n}$. The rest of the two singular values, without loss of generality, denoted by $\sigma_1$ and $\sigma_K$, depend on $\vb$. Then, we have the following lower bounds for $\sigma_1$ and $\sigma_K$.
\begin{enumerate}
    \item For any $\vb$, the largest singular value $\sigma_1$ can be lower bounded by 
\begin{align}
    \sigma_1 \ge \sqrt{n} \max\paren{\sqrt{1  + K \paren{ \norm{\vb}{}^2 - \frac{1}{K}\paren{\vone^\top \vb}^2 }},    \abs{1 - \vone^\top \vb }}.
\label{eq:sigma1-lower-bound}\end{align}
\item For any $\vb$ on the sphere $\{\vb\in\R^K:\norm{\vb}{} = c\}$ with $c\le \frac{1}{\sqrt{K}}$, we have
\begin{align}
    \sigma_1 \ge  \sqrt{n}, \quad \sigma_K \ge \sqrt{n}\paren{1 - \sqrt{K}c }
\end{align}
and both inequalities become equalities \emph{if and only if} $\vb = \frac{c}{\sqrt{K}}\vone$. 
\end{enumerate}
\label{lem:spectral-Ytilde}\end{lemma}

\begin{proof}[Proof of \Cref{lem:spectral-Ytilde}] To study the singular values of $\mY$, it is equivalent to look at the eigenvalues of the Gram matrix of $\wt \mY^\top$:
\[
\mG = \wt \mY \wt \mY^\top = \paren{\mY - \vb\vone^\top}\paren{\mY - \vb\vone^\top}^\top = n \paren{\mId - \vb\vone^\top - \vone \vb^\top + K \vb\vb^\top}. 
\]
If $\vb$ is aligned with $\vone$, i.e., they live in the same line, then $- \vb\vone^\top - \vone \vb^\top + K \vb\vb^\top$ is a rank-1 matrix and $\mG$ has $K-1$ eignevalues being $n$ and the rest eigenvalue being $n \paren{1 - \vone^\top \vb }^2$. On the other hand, if $\vb$ is not aligned with $\vone$, then $- \vb\vone^\top - \vone \vb^\top + K \vb\vb^\top$ is a rank-2 matrix and $\mG$ has $K-2$ eignevalues being $n$. In this case, the rest of the two eigenvalues, denoted by $\pi_1$ and $\pi_K$, correspond to the eigenvectors within the subspace spanned by $\vone$ and $\vb$. 

To estimate the largest eigenvalues $\pi_1$, we construct two orthonormal vectors within this subspace spanned by $\vone$ and $\vb$ and compute the corresponding Rayleigh quotient. Specifically, we first compute the Rayleigh quotient along the direction $\vone$ as
\begin{align*}
    \frac{\vone^\top\mG \vone}{\vone^\top \vone} = \frac{n}{K}\paren{K - 2K\vone^\top \vb + K \paren{\vone^\top \vb}^2  } = n \paren{1 - 2\vone^\top \vb +  \paren{\vone^\top \vb}^2  } = n \paren{1 - \vone^\top \vb }^2.
\end{align*}

Use Gram-Schmidt orthonormalization to obtain the other direction as $\va = \vb - \frac{1}{K} \vone^\top \vb \vone$, which gives the following Rayleigh quotient:
\begin{align*}
    \frac{\va^\top\mG \va}{\va^\top \va} = \frac{n}{\norm{\va}{}^2}\paren{\norm{\va}{}^2  + K \paren{ \norm{\vb}{}^2 - \frac{1}{K}\paren{\vone^\top \vb}^2  }^2  } = n  + nK \paren{ \norm{\vb}{}^2 - \frac{1}{K}\paren{\vone^\top \vb}^2 }, 
\end{align*}
where the last equality follows because $\norm{\va}{}^2 = \norm{\vb}{}^2 - \frac{1}{K}\paren{\vone^\top \vb}^2$. Thus, by the min-max theorem (i.e., Courant–Fischer–Weyl min-max principle), we have
\[
\pi_1 \ge  \max\paren{\frac{\vone^\top\mG \vone}{\vone^\top \vone},\frac{\va^\top\mG \va}{\va^\top \va} } \ge \max\paren{n \paren{1 - \vone^\top \vb }^2, n  + nK \paren{ \norm{\vb}{}^2 - \frac{1}{K}\paren{\vone^\top \vb}^2}} \ge n,
\]
where the last inequality becomes an inequality \emph{if and only if} $\vb$ is a scaled version of the vector $\vone$, i.e., $\vb = \frac{\norm{\vb}{}}{\sqrt{K}}\vone$.

To obtain a lower bound for $\pi_K$ whenever $\norm{\vb}{}\le \frac{1}{\sqrt{K}}$, we again use the the min-max theorem as
\begin{align*}
\frac{1}{n}\pi_K &\ge \min_{\norm{\vu}{}=1} \frac{1}{n} \vu^\top \mG \vu = \min_{\norm{\vu}{}=1} 1 - 2\vu^\top \vb \vone^\top \vu + K (\vu^\top \vb)^2\\
& \ge \min_{\norm{\vu}{}=1} 1 - 2\sqrt{K}\abs{\vu^\top \vb}  + K (\vu^\top \vb)^2\\
& \ge 1 - 2\sqrt{K}\norm{\vb}{} + K \norm{\vb}{}^2 = \paren{1 - \sqrt{K}\norm{\vb}{}}^2, 
\end{align*}
where the first inequality achieves equality when $\vu$ is restricted to the subspace spanned by $\vone$ and $\vb$, the second inequality becomes an equality only when $\vu = \vone/\sqrt{K}$ and $\vu^\top \vb \ge 0$ or $\vu = -\vone/\sqrt{K}$ and $\vu^\top \vb \le 0$, and the last inequality achieves equality if and only if $\vu$ is aligned with $\vb$, i.e., $\abs{\vu^\top \vb} = \norm{\vb}{}$. Thus, for any $\vb$ on the sphere $\{\vb\in\R^K:\norm{\vb}{} = c\}$ with $c\le \frac{1}{\sqrt{K}}$, $\pi_K$ achieves its minimum possible value $n \paren{1 - \sqrt{K}\norm{\vb}{}}^2$ if and only if $\vb = \pm\frac{c}{\sqrt{K}}\vone$. This completes the proof.

\end{proof}

\begin{lemma}\label{lem:global-one-hot-nuclear} Assume the number of training samples in each class is balanced, i.e., $n = n_1 = \cdots = n_K$, and let $\mY = \begin{bmatrix}\vy_1 & \cdots & \vy_1 & \vy_2 & \cdots & \vy_K \end{bmatrix} \in \R^{K\times nK}$ be the matrix that contains the one-hot vectors for all the training samples. Suppose $b^\star\le \frac{1}{K}$. Then any global minimizer  $(\mW^\star,\mH^\star)$  of 
	\begin{align}\label{eq:obj-app}
     \min_{\mb W , \mb H}  ~ \frac{1}{2N}\norm{\mW\mH + b^\star\vone\vone^\top - \mY}{F}^2 \;+\; \frac{\lambda_{\mb W} }{2} \norm{\mb W}{F}^2 + \frac{\lambda_{\mb H} }{2} \norm{\mb H}{F}^2. 
\end{align}
satisfies the self-duality
\begin{gather*}
      \vh_{k,i}^\star \;=\;  \sqrt{ \frac{ \lambda_{\mb W}  }{ \lambda_{\mb H} n } } \vw^{\star k} ,\quad \forall \; k\in[K],\; i\in[n].
\end{gather*}
Moreover, if $\lambdaW\lambdaH \ge \frac{1}{NK}$, then $\mW^\star = \vzero$ and $\mH^\star = \vzero$. On the other hand, if $\lambdaW\lambdaH < \frac{1}{NK}$, $(\mW^\star,\mH^\star)$ further obeys the following properties for different $d$:
\begin{enumerate}
\item $d< K-1$: $\mW^\star\mW^{\star\top}\sim \calP_d(\mId - \frac{1}{K} \mb 1_K \mb 1_K^\top)$ where $\calP_d$ denotes the best rank-$d$ approximating and $\mA\sim\mB$ means that there is a constant $c$ such that $\mA = c\mB$;
\item  $d = K-1$: In this case, $\mW^{\star} \mW^{\star\top} \sim\mb I_K - \frac{1}{K} \mb 1_K \mb 1_K^\top$;
 \item  $d\ge K$ and $b^\star \ge \frac{1}{K} - \sqrt{n\lambdaW\lambdaH}$: $\mW^{\star} \mW^{\star\top} \sim\mb I_K - \frac{1}{K} \mb 1_K \mb 1_K^\top$;
\item  $d\ge K$ and $b^\star < \frac{1}{K} - \sqrt{n\lambdaW\lambdaH}$: $\mW^{\star} \mW^{\star\top} \sim \mId - \frac{b^\star}{1 - K\sqrt{n\lambdaW\lambdaH}} \vone_K \vone_K^\top$;

\end{enumerate}

\end{lemma}
\begin{proof}[Proof of \Cref{lem:global-one-hot-nuclear}] 
For convenience, let $\vone_{K\times L}$ represents an all-ones matrix of size $K\times L$.
Since $\mY - b^\star \vone_{K\times nK}$ contains many repeated columns, we first consider $\ol \mY = \mId_K - b^\star \vone_{K\times K}$ that contains the non-repeated columns of $\mY - b^\star \vone_{K\times nK}$. Let $\ol \mY = \mU\overline\mSigma \mU^\top$ be the eigenvalue decomposition, where $\mU\in\R^{K\times K}$ is an orthonormal matrix and  $\ol\mSigma\in\R^{K\times K}$ is a diagonal matrix with eigenvalues $\ol\sigma_1 \ge \cdots \ol\sigma_K $ along the diagonals. Since $b^\star\le \frac{1}{K}$, the eigenvalues are given by $\ol\sigma_1 = \cdots = \ol\sigma_{K-1} = 1 \ge \ol\sigma_{K} = 1 - b^\star K$, and the eigenvector corresponding to $\ol \sigma_K$ is $\vu_K = \frac{1}{\sqrt{K}}\vone$, which implies that $[\mU]_{K-1}[\mU]_{K-1}^\top = \mId - \frac{1}{K} \mb 1_K \mb 1_K^\top$, where $[\mU]_{r}$ means a $K\times r$ submtraix of $\mU$ by taking the first $r$ columns.

Let $\mSigma = \sqrt{n}\ol\mSigma$ and $\mV^\top = \frac{1}{\sqrt{n}}\begin{bmatrix}\vu^1 & \cdots & \vu^1 & \vu^2 & \cdots & \vu^K \end{bmatrix} \in \R^{K\times nK}$ that repeats the rescaled version of the column of $\mU$ $n$ times so that $\mV^\top\mV = \mU^\top\mU = \mId$. By noting the relation between $\mY - b^\star \vone_{K\times nK}$ and $\ol \mY$, we know $\mU\mSigma\mV^\top$ is the SVD of $\mY - b^\star \vone_{K\times nK}$. When $\lambdaW\lambdaH \ge \frac{1}{NK}$, by applying \Cref{lem:global-fact-nuclear} and \Cref{lem:spectral-Ytilde}, we conclude that $\mW^\star = \vzero$ and $\mH^\star = \vzero$ since $\sqrt{n} - N\sqrt{\lambdaW\lambdaH}\le 0$. We now assume $\lambdaW\lambdaH < \frac{1}{NK}$ and utilize \Cref{lem:global-fact-nuclear} and \Cref{lem:spectral-Ytilde} again for the following cases:
\begin{enumerate}
\item $d < K-1$: In this case, we have
 \begin{align*}
&\mW^\star = \sqrt\frac{\lambda_{\mH}}{\lambda_{\mW}}\paren{\sqrt{n} - N\sqrt{\lambdaW\lambdaH}}^{1/2}\mU(:,1:d)\mR, \\
&\mH^\star = \sqrt\frac{\lambda_{\mW}}{\lambda_{\mH}}\paren{\sqrt{n} - N\sqrt{\lambdaW\lambdaH}}^{1/2}  \mR^\top \mV(:,1:d)^\top, \forall \mR\in\R^{d\times d}, \mR^\top \mR = \mId. 
\end{align*}
Thus, $\vh_{k,i}^\star \;=\;  \sqrt{ \frac{ \lambda_{\mb W}  }{ \lambda_{\mb H} n } } \vw^{\star k}$ and
$\mW^\star\mW^{\star\top}\sim \mU(:,1:d) \mU(:,1:d)^\top = \calP_d(\mId - \frac{1}{K} \mb 1_K \mb 1_K^\top)$.
    \item $d = K-1$: In this case, we have
    \begin{align*}
&\mW^\star = \sqrt\frac{\lambda_{\mH}}{\lambda_{\mW}}\paren{\sqrt{n} - N\sqrt{\lambdaW\lambdaH}}^{1/2}\mU(:,1:K-1)\mR, \\
&\mH^\star = \sqrt\frac{\lambda_{\mW}}{\lambda_{\mH}}\paren{\sqrt{n} - N\sqrt{\lambdaW\lambdaH}}^{1/2}  \mR^\top \mV(:,1:K-1)^\top, \forall \mR\in\R^{(K-1)\times (K-1)}, \mR^\top \mR = \mId. 
\end{align*}
Thus, $\vh_{k,i}^\star \;=\;  \sqrt{ \frac{ \lambda_{\mb W}  }{ \lambda_{\mb H} n } } \vw^{\star k}$ and $
\mW^{\star} \mW^{\star\top} \sim [\mU]_{K-1}[\mU]_{K-1}^\top = \mb I_K - \frac{1}{K} \mb 1_K \mb 1_K^\top$. 
\item $d = K$: In this case, we have
\begin{align*}
&\mW^\star = \sqrt\frac{\lambda_{\mH}}{\lambda_{\mW}}\mU\brac{ \mSigma - N\sqrt{\lambdaW\lambdaH}  }_{+}^{1/2} \mR, \\
&\mH^\star = \sqrt\frac{\lambda_{\mW}}{\lambda_{\mH}}\mR^\top \brac{ \mSigma - N\sqrt{\lambdaW\lambdaH}  }_{+}^{1/2}  \mV^\top, \ \forall \mR\in\R^{K\times K}, \mR^\top \mR = \mId  
\end{align*}
Thus, $\vh_{k,i}^\star \;=\;  \sqrt{ \frac{ \lambda_{\mb W}  }{ \lambda_{\mb H} n } } \vw^{\star k}$. Moreover, if $N\sqrt{\lambdaW\lambdaH}\ge \sqrt{n}(1-b^\star K)$, i.e., $b^\star \ge \frac{1}{K} - \sqrt{n\lambdaW\lambdaH}$, then  $
\mW^{\star} \mW^{\star\top} \sim [\mU]_{K-1}[\mU]_{K-1}^\top =  \mb I_K - \frac{1}{K} \mb 1_K \mb 1_K^\top$. On the other hand, if $b^\star < \frac{1}{K} - \sqrt{n\lambdaW\lambdaH}$, then
\begin{align*}
\mW^{\star} \mW^{\star\top} &\sim \mU\mSigma\mU^\top - K\sqrt{n\lambdaW\lambdaH}\mU\mU^\top =  \ol\sigma_{K} = \mId - b^\star\mb 1_K \mb 1_K^\top - N\sqrt{\lambdaW\lambdaH}\mId \\
& = (1 - K\sqrt{n\lambdaW\lambdaH})\mId - b^\star\mb 1_K \mb 1_K^\top
\sim \mId - \frac{b^\star}{1 - K\sqrt{n\lambdaW\lambdaH}} \vone_K \vone_K^\top.
\end{align*}

\item $d > K$: In this case, we have
\begin{align*}
&\mW^\star = \sqrt\frac{\lambda_{\mH}}{\lambda_{\mW}} \begin{bmatrix}\mU\brac{ \mSigma - N\sqrt{\lambdaW\lambdaH}  }_{+}^{1/2}  & \vzero \end{bmatrix}\mR, \\ 
&\mH^\star = \sqrt\frac{\lambda_{\mW}}{\lambda_{\mH}}\mR^\top \begin{bmatrix}\brac{ \mSigma - N\sqrt{\lambdaW\lambdaH}  }_{+}^{1/2}  \mV^\top \\
\vzero 
\end{bmatrix},  \ \forall \mR\in\R^{d\times d}, \mR^\top \mR = \mId. 
\end{align*}
One can verify that $(\mW^\star,\mH^\star)$ satisfies the same properties as in the case of $d = K$.
\end{enumerate}
\end{proof}

\section{Visualizations of Optimization Landscapes in \Cref{subsec:landscape-rescaling}}
\label{sec:appendix-visualization}

\subsection{Details of the Visualization Technique}

We provide the technical details on how the visualization in \Cref{subsec:landscape-rescaling} is obtained. 

The following result expresses the output of the classifier layer for a feature vector $\bm h$ as a function of the norm of $\bm h$ and its angle to a classifier weight vector $\bm w^k$.

% The visualization of the optimization landscape w.r.t. feature vectors, provided in \Cref{subsec:landscape-rescaling}, is obtained with the following result. 

\begin{proposition}\label{thm:visualization}

Given any $d \ge K-1 > 1$, take the classifier weights $\bm W, \bm b$ to be such that $\bm W$ is an arbitrary $K$-Simplex ETF (see Definition~\ref{def:simplex-ETF}) and $\bm b = \bm 0$. Take any $k, k' \in \{1, \ldots, K\}$, and consider a vector $\mb h$ on the two-dimensional plane $\text{span}\{\bm w^k, \bm w^{k'}\}$ parameterized in the polar coordinate system with polar axis being $\mb w^k$.
Denote $s$ and $\theta$ the radial and angular (in radians) coordinates of $\bm h$, respectively 
(positive angular direction of the polar coordinate system is taken so that $\mb w^{k'}$'s angular coordinate is in $(0, \pi)$). 
We have
\begin{itemize}[leftmargin=*]
    \item The feature $\bm h$ can be expressed as a linear combination of $\bm w^{k}$ and $\bm w^{k'}$:
    \begin{equation}
        \bm h = s \Big( \frac{\sin\theta}{\sqrt{K^2 - 2K}} + \cos\theta \Big) \bm w^{k} + s (K-1) \frac{\sin \theta}{\sqrt{K^2 - 2K}} \bm w^{k'};
    \end{equation}
    \item The output of the classifier layer $(\bm W, \bm b)$ is given by
    \begin{equation}
    \label{eq:visualization-classifier-output}
        \langle \bm w^{k''}, \bm h \rangle +  b_{k''} = 
        \begin{cases} 
             s \cos \theta, & \text{if}~k'' = k;\\
             s \frac{\sqrt{K^2 - 2K}}{K-1} \sin \theta - \frac{s}{K-1} \cos \theta, & \text{if}~k'' = k';\\
             -s \sqrt{\frac{K}{K-2}} \frac{1}{K-1} \sin \theta - \frac{s}{K-1} \cos \theta, & \text{otherwise}.
        \end{cases}
    \end{equation}
    Note that \eqref{eq:visualization-classifier-output} is invariant to the arbitrary rotation in K-Simplex ETF. 
\end{itemize}

\end{proposition}

We omit the proof to Proposition~\ref{thm:visualization} as it can be obtained via simple algebra.

Based on Proposition~\ref{thm:visualization}, we can obtain the (rescaled) MSE and CE losses as a function of $(s, \theta)$. 
Assuming that $\bm h$ belongs to class $k$, the rescaled MSE loss defined in \eqref{eq:obj-rescaled} w.r.t. $\bm h$ is given by
\begin{equation}
    \text{Loss}_\text{MSE}(\bm h; \alpha, M) = \frac{\alpha}{2} \Big( \langle \bm w^{k}, \bm h \rangle +  b_{k} - M \Big) ^2 + \frac{1}{2}\sum_{k''\ne k}\Big( \langle \bm w^{k''}, \bm h \rangle +  b_{k''} - 1 \Big) ^2,
\end{equation}
where $\alpha, M$ are rescaling parameters.
Plugging in the results in \eqref{eq:visualization-classifier-output}, we obtain
\begin{multline}\label{eq:landscape-rMSE}
    \text{Loss}_\text{MSE}(s, \theta; \alpha, M) = \frac{\alpha}{2} \cdot\left(s\cos\theta-M\right)^{2\ } + \frac{s^{2}}{2} \cdot\left(\frac{\sqrt{K^{2}-2K}\sin\theta - \cos\theta}{K-1}\right)^{2} 
    \\+ \frac{s^{2}}{2}\cdot\left(K-2\right)\cdot\left(\frac{\sqrt{\frac{K}{K-2}}\sin\theta + \cos \theta}{K-1}\right)^{2}.
\end{multline}
Similarly, we may obtain the CE loss as
\begin{equation}\label{eq:landscape-CE}
    \text{Loss}_\text{CE}(s, \theta) = -\log\left(\frac{e^{s\cos\theta}}{e^{s\cos\theta}+e^{s\frac{\sqrt{K^{2}-2K}\sin\theta - \cos\theta}{K-1}}+\left(K-2\right)e^{-s\frac{\sqrt{\frac{K}{K-2}}\sin\theta + \cos \theta}{K-1}}}\right).
\end{equation}
\Cref{fig:visualization} is obtained by plotting the loss functions in \eqref{eq:landscape-rMSE} and \eqref{eq:landscape-CE}. 

\subsection{Visualization of the Gradient Vector Field}

% Complementary to the graphical illustration of the optimization landscape in \Cref{subsec:landscape-rescaling}, we provide an analytical understanding on how the choice of loss affects the optimization.

We consider the regime of $K \to \infty$ in which the rescaled MSE loss \eqref{eq:landscape-rMSE} becomes
\begin{equation}
    \lim_{K \to \infty} \text{Loss}_\text{MSE}(s, \theta; \alpha, M) = \frac{\alpha}{2}(s \cos \theta - M)^2 + \frac{1}{2}s^2 \sin^2 \theta.
\end{equation}
Taking the derivative w.r.t. $s$ and $\theta$, we obtain
\begin{equation}\label{eq:gradient-MSE}
    \begin{split}
        \frac{\partial}{\partial s}\lim_{K \to \infty} \text{Loss}_\text{MSE}(s, \theta; \alpha, M) &= s + (\alpha - 1) s \cos^2\theta - \alpha M \cos \theta,\\
        \frac{\partial}{\partial \theta}\lim_{K \to \infty} \text{Loss}_\text{MSE}(s, \theta; \alpha, M) &= \alpha M s \sin\theta - (\alpha-1) s^2 \sin\theta \cos\theta. 
    \end{split}
\end{equation}
Similarly, we may obtain the gradient for CE as
\begin{equation}\label{eq:gradient-CE}
    \begin{split}
        \frac{\partial}{\partial s}\lim_{K \to \infty} \text{Loss}_\text{CE}(s, \theta; \alpha, M) &= \frac{e^{\sin\theta}(\sin\theta - \cos\theta)}{e^{\sin\theta} + e^{\cos\theta}},\\
        \frac{\partial}{\partial \theta}\lim_{K \to \infty} \text{Loss}_\text{CE}(s, \theta; \alpha, M) &= \frac{s e^{\sin\theta}(\sin\theta + \cos\theta)}{e^{\sin\theta} + e^{\cos\theta}}. 
    \end{split}
\end{equation}
In \Cref{fig:gradient-visualization}, we visualize the gradient of MSE (in \eqref{eq:gradient-MSE}) and CE (in \eqref{eq:gradient-CE}) losses by plotting their gradient vector fields. 
It shows that rescaling of the MSE loss by either increasing $M$ or increasing $\alpha$ helps to align the gradient along the direction of minimizing $\theta$.
Recall that $\theta$ determines the classifier's prediction of the class membership for $\bm h$ while $s$ is irrelevant.

When restricting our attention to a feature $\bm h$ with $\theta = \frac{\pi}{2}$, the gradient w.r.t. $s$ and $\theta$ becomes $s$ and $\alpha M s$, respectively. 
Here, increasing the rescaling parameters $\alpha$ or $M$ in the range of $(1, \infty)$ has the effect of increasing the component of the gradient along the $\theta$ direction while keeping the component along the $s$ direction fixed.

 \begin{figure}[t]
     \centering
     \subfloat[Vanilla MSE ($\alpha=1, M = 1$)\label{fig:gradient-mse}]{
         \begin{tikzpicture}
             \begin{axis}[
                 xmin = 0, xmax = 5,
                 ymin = 0, ymax = 3.14,
                 zmin = 0, zmax = 1,
                 axis equal image,
                 view = {0}{90},
                 xlabel={$s$},
                 ylabel={$\theta$},
                 scale=0.8,
             ]
             \def\M{1}
             \def\a{1}
             \addplot3[
                 quiver = {
                     u = {-x - (\a-1) * x * cos(deg(y)) * cos(deg(y)) + \M * cos(deg(y))},
                     v = {-x * \a * \M * sin(deg(y)) + (\a-1) * x^2 * sin(deg(y)) * cos(deg(y))},
                     scale arrows = 0.04,
                 },
                 -stealth,
                 domain = 0:5,
                 domain y = 0:3.14,
             ] {0};
             \end{axis}
         \end{tikzpicture}   
     } 
     ~
     \subfloat[Cross Entropy \label{fig:gradient-ce}]{
         \begin{tikzpicture}
             \begin{axis}[
                 xmin = 0, xmax = 5,
                 ymin = 0, ymax = 3.14,
                 zmin = 0, zmax = 1,
                 axis equal image,
                 view = {0}{90},
                 xlabel={$s$},
                 ylabel={$\theta$},
                 scale=0.8,
             ]
             \def\M{1}
             \def\a{5}
             \addplot3[
                 quiver = {
                     v = {-1 * x * 2.7^sin(deg(y)) * (sin(deg(y)) + cos(deg(y))) / (2.7^cos(deg(y)) + 2.7^sin(deg(y))) },
                     u = {-1 * 2.7^sin(deg(y)) * (sin(deg(y)) - cos(deg(y))) / (2.7^cos(deg(y)) + 2.7^sin(deg(y)))},
                     scale arrows = 0.08,
                 },
                 -stealth,
                 domain = 0:5,
                 domain y = 0:3.14,
             ] {0};
             \end{axis}
         \end{tikzpicture}
     } 
     \\\vspace{1em}
     \subfloat[Rescaled MSE ($\alpha=5, M = 1$) \label{fig:gradient-rmse-alpha}]{
         \begin{tikzpicture}
             \begin{axis}[
                 xmin = 0, xmax = 5,
                 ymin = 0, ymax = 3.14,
                 zmin = 0, zmax = 1,
                 axis equal image,
                 view = {0}{90},
                 xlabel={$s$},
                 ylabel={$\theta$},
                 scale=0.8,
             ]
             \def\M{1}
             \def\a{5}
             \addplot3[
                 quiver = {
                     u = {-x - (\a-1) * x * cos(deg(y)) * cos(deg(y)) + \M * cos(deg(y))},
                     v = {-x * \a * \M * sin(deg(y)) + (\a-1) * x^2 * sin(deg(y)) * cos(deg(y))},
                     scale arrows = 0.01,
                 },
                 -stealth,
                 domain = 0:5,
                 domain y = 0:3.14,
             ] {0};
             \end{axis}
         \end{tikzpicture}
     } 
     ~
     \subfloat[Rescaled MSE ($\alpha=1, M = 5$)  \label{fig:gradient-rmse-M}]{
         \begin{tikzpicture}
             \begin{axis}[
                 xmin = 0, xmax = 5,
                 ymin = 0, ymax = 3.14,
                 zmin = 0, zmax = 1,
                 axis equal image,
                 view = {0}{90},
                 xlabel={$s$},
                 ylabel={$\theta$},
                 scale=0.8,
             ]
             \def\M{5}
             \def\a{1}
             \addplot3[
                 quiver = {
                     u = {-x - (\a-1) * x * cos(deg(y)) * cos(deg(y)) + \M * cos(deg(y))},
                     v = {-x * \a * \M * sin(deg(y)) + (\a-1) * x^2 * sin(deg(y)) * cos(deg(y))},
                     scale arrows = 0.01,
                 },
                 -stealth,
                 domain = 0:5,
                 domain y = 0:3.14,
            ] {0};
             \end{axis}
         \end{tikzpicture}
     } 
     \caption{\textbf{Visualization of the gradient vector fields with different losses.} We fix $\mb W$ as a simplex ETF and illustrate the landscape only w.r.t. a feature $\mb h_{k, i}$. For each plot, the $s$-axis denotes $\norm{\mb h_{k, i}}{2}$, and the $\theta$-axis denotes the angle $\arccos\paren{ \innerprod{\mb h_{k, i}}{\mb w^k} }$. The arrows point to gradient descent directions with length proportional to the gradient norm. 
     }
     \label{fig:gradient-visualization}
\end{figure}
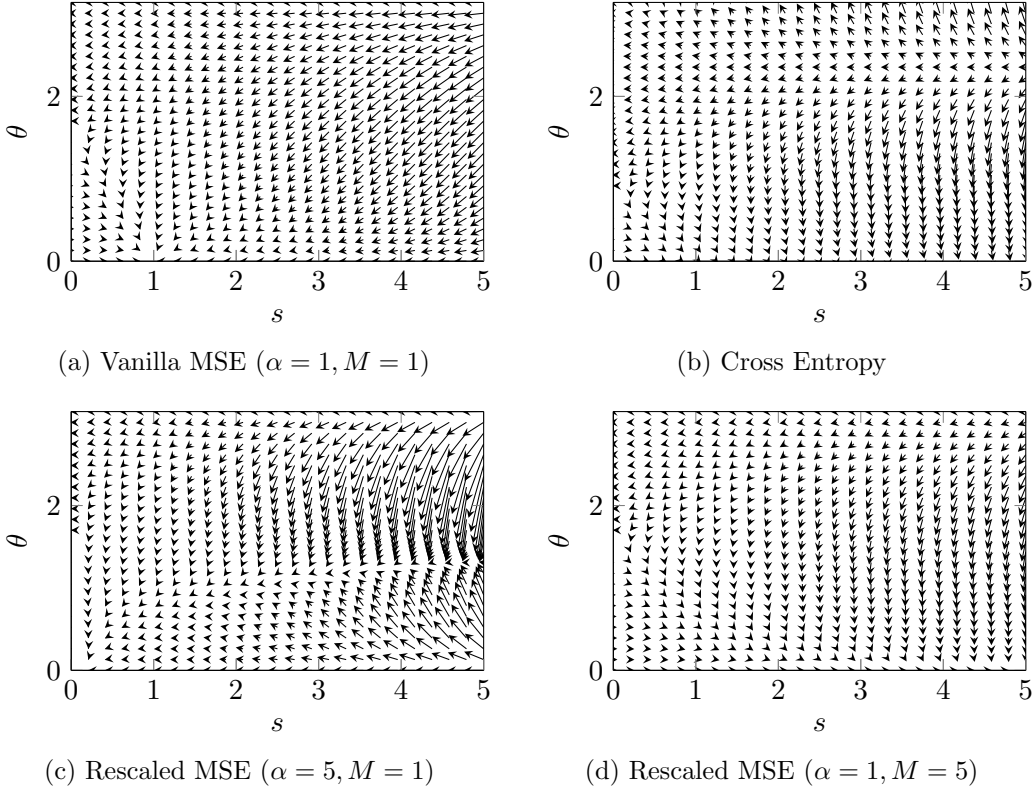

%%%%%%%%%%%%%%%%%%%%%%%%%%%%%%%%%%%%%%%%%%%%%%%%%%%%%%%%%%%%%%%%%%%%%%%%%%%%%%%
%%%%%%%%%%%%%%%%%%%%%%%%%%%%%%%%%%%%%%%%%%%%%%%%%%%%%%%%%%%%%%%%%%%%%%%%%%%%%%%

\end{document}